\documentclass[12pt]{alt2021} 

\title[Average-case Complexity of Teaching Convex Polytopes via Halfspace Queries]{Average-case Complexity of Teaching Convex Polytopes\\via Halfspace Queries}
\usepackage{times}

\usepackage{savesym,multirow,url,xspace,graphicx,hyperref,enumitem,color}

\usepackage[utf8]{inputenc} 
\usepackage[T1]{fontenc}    
\usepackage{hyperref}       
\usepackage{url}            
\usepackage{booktabs}       
\usepackage{amsfonts}       
\usepackage{nicefrac}       
\usepackage{microtype}      
\usepackage{tikz}
\usepackage{pgfplots}
\usepackage{verbatim}
\usepackage{amsmath}
\usepackage[noend]{algorithmic}
\usepackage{algorithm}
\usepackage{multirow}
\usepackage{mathtools}
\usepackage{wrapfig}
\usepackage{multicol}
\usepackage{caption}
\usepackage{subfloat}

\SetCommentSty{mycommfont}

\newtheorem{assumption}{Assumption}

\newcommand{\argmin}{\operatornamewithlimits{arg\,min}}

\PassOptionsToPackage{dvipsnames}{xcolor}
\PassOptionsToPackage{dvipsnames}{xcolor}
\usepackage{color,soul}
\usepackage{backref}
\usepackage{hyperref}
\hypersetup{
  colorlinks = true,
  citecolor = black,
  linkcolor = black,
}

\makeatletter

\newcommand{\Rmnum}[1]{\expandafter\@slowromancap\romannumeral #1@}
\makeatother
\usepackage{caption}

\newcommand{\Rd}{\mathbb{R}^d}

\newcommand{\defref}[1]{Definition~\ref{#1}}
\newcommand{\tabref}[1]{Table~\ref{#1}}
\newcommand{\figref}[1]{Fig.~\ref{#1}}
\newcommand{\eqnref}[1]{\text{Eq.}~(\ref{#1})}
\newcommand{\secref}[1]{\S\ref{#1}}
\newcommand{\appref}[1]{Appendix \ref{#1}}
\newcommand{\thmref}[1]{Theorem~\ref{#1}}
\newcommand{\corref}[1]{Corollary~\ref{#1}}
\newcommand{\propref}[1]{Proposition~\ref{#1}}
\newcommand{\lemref}[1]{Lemma~\ref{#1}}

\newcommand{\assref}[1]{Assumption~\ref{#1}}

\newcommand{\algoref}[1]{Algorithm~\ref{#1}}

\newcommand{\paren} [1] {\ensuremath{ \left( {#1} \right) }}
\newcommand{\parenb} [1] {\ensuremath{ \big( {#1} \big) }}
\newcommand{\bigparen} [1] {\ensuremath{ \Big( {#1} \Big) }}
\newcommand{\biggparen} [1] {\ensuremath{ \bigg( {#1} \bigg) }}
\newcommand{\Biggparen} [1] {\ensuremath{ \Bigg( {#1} \Bigg) }}
\newcommand{\bracket}[1]{\left[#1\right]}
\newcommand{\tuple}[1]{\ensuremath{\left\langle #1 \right\rangle}}
\newcommand{\curlybracket}[1]{\ensuremath{\left\{#1\right\}}}

\newcommand{\condcurlybracket}[2]{\ensuremath{\left\{#1\left\lvert\:#2\right.\right\}}}

\newcommand{\expctover}[2]{\mathbb{E}_{#1}\!\left[#2\right]}

\def \argmin {\mathop{\rm arg\,min}}

\newcommand{\bigO}[1]{\ensuremath{\mathcal{O}\paren{#1}}}
\newcommand{\bigTheta}[1]{\ensuremath{\Theta\paren{#1}}}
\newcommand{\bigOmega}[1]{\ensuremath{\Omega\paren{#1}}}

\newcommand{\reals}{\ensuremath{\mathbb{R}}}

\newcommand{\cA}{{\mathcal{A}}}

\newcommand{\cU}{{\mathcal{U}}}

\newcommand{\cC}{{\mathcal{C}}}
\newcommand{\cO}{{\mathcal{O}}}

\newcommand{\ff}{{\boldsymbol{\mathfrak{F}}}}
\newcommand{\fr}{{\boldsymbol{\mathfrak{r}}}}

\newcommand{\Hnd}{\boldsymbol{\mathcal{H}}}
\newcommand{\Hndp}{\boldsymbol{\mathcal{H}}_{(n),d}}
\newcommand{\Hndo}{\boldsymbol{\mathcal{H}}_{n-1,d-1}}
\newcommand{\xnd}{\boldsymbol{\mathcal{X}}}
\newcommand{\xndp}{\boldsymbol{\mathcal{X}}^{+}}
\newcommand{\xndn}{\boldsymbol{\mathcal{X}}^{-}}
\newcommand{\Hndd}{\boldsymbol{\hat{\mathcal{H}}}_{n,d'}}
\newcommand{\rnd}{\boldsymbol{\fr}\parenb{\cA(\Hnd)}}
\newcommand{\regionset}{\boldsymbol{\mathfrak{R}}\parenb{\mathcal{A}(\Hnd)}}
\newcommand{\regionpair}{\boldsymbol{\mathfrak{R}}\parenb{\mathcal{A}(\boldsymbol{\mathcal{H}}_{(n),d})}}
\renewcommand{\tt}[1]{\textit{#1}}
\newcommand{\rank}[1]{\textbf{rank}\parenb{#1}}

\newcommand{\boldlam}{\mathbf{\boldsymbol{\Lambda}}}
\def\mathbi#1{\textbf{\em #1}}

\def\BState{\State\hskip-\ALG@thistlm}

\newcommand{\data}{\textit{data space}}
\newcommand{\hypo}{\textit{hypothesis space}}

\newcommand{\gen}{{$d'$-\textit{relaxed general position}}}

\newcommand{\h}{z}
\newcommand{\dt}{x}
\newcommand{\rgn}{r}

\newtoggle{longversion}
\settoggle{longversion}{true}

\altauthor{%
  \Name{Akash Kumar}
  \Email{akumar@mpi-sws.org} \\
  \addr MPI-SWS
   \AND
   \Name{Adish Singla} \Email{adishs@mpi-sws.org} \\
   \addr MPI-SWS 
   \AND
   \Name{Yisong Yue} \Email{yyue@caltech.edu} \\
   \addr Caltech
   \AND
   \Name{Yuxin Chen} \Email{chenyuxin@uchicago.edu} \\
   \addr University of Chicago \\
}

\begin{document}

\maketitle

\begin{abstract}%
  We examine the task of locating a target region among those induced by intersections of $n$ halfspaces in $\mathbb{R}^d$. This generic task connects to  fundamental machine learning problems, such as training a perceptron and learning a $\phi$-separable dichotomy. We investigate the \emph{average teaching} complexity of the task, i.e., the minimal number of samples (halfspace queries) required by a \emph{teacher} to help a version-space learner in locating a \emph{randomly} selected target. As our main result, we show that the average-case teaching complexity is  $\Theta(d)$, which is in sharp contrast to the worst-case teaching complexity of $\Theta(n)$. If instead, we consider the average-case learning complexity, the bounds have a dependency on $n$ as $\Theta(n)$ for \tt{i.i.d.} queries and $\Theta(d \log(n))$ for actively chosen queries by the learner. Our proof techniques are based on novel insights from computational geometry, which allow us to count the number of convex polytopes and faces in a Euclidean space depending on the arrangement of halfspaces. Our insights allow us to establish a tight bound on the average-case complexity for $\phi$-separable dichotomies, which generalizes the known $\mathcal{O}(d)$ 
bound on the average number of ``extreme patterns'' in the classical computational geometry literature \citep{cover1965geometrical}.
\end{abstract}
\smallskip

\section{Introduction}
We consider the problem of locating a target region among those induced by intersections of $n$ halfspaces in $d$-dimension (\figref{fig:example.polytope}). 
In the basic setting, the learner receives a sequence of instructions, which we refer to as \emph{halfspace queries} \citep[same as membership queries in][]{ANGLUIN198787,ANGLUINb}, each specifying a halfspace the target region is in. Based on the evidence it receives, the learner then determines the location of the target region. 
This generic task connects to several fundamental problems in machine learning. 
Consider learning a linear prediction function in $\reals^d$ (aka perceptron, see \figref{fig:example.perceptron}) over $n$ linearly separable data points. Here, every data point specifies a halfspace, and the target hypothesis corresponds to a region in the hypothesis space. The learning task reduces to identifying the convex polytope induced by the $n$ halfspace constraints in the hypothesis spaces \citet{bishop2006pattern}.
Similarly, when the set of data points are not linearly separable, but are separable by a $\phi$-surface (aka $\phi$-separable dichotomy, see \figref{fig:example.phi}), the problem of finding the $\phi$-separable dichotomy could be viewed as training a perceptron in the $\phi$-induced space \citep{cover1965geometrical}.

While these fundamental problems have been extensively studied in the \emph{passive learning} setting \citep{vapnik1971uniform,natarajan1987learning,blumer1989learnability,goldman1993learning}, the underlying \emph{i.i.d.} sampling strategy often requires more data than necessary to learn the target concept (when one is able to control the sampling strategy).
Moreover, the majority of existing work focuses on the worst-case complexity measures, which are often too pessimistic and do not reflect the learning complexity in the real-world scenarios \citep{haussler1994bounds,wan2010learning,nachum2019average}. As shown in \tabref{tab:sample-complexity}, the label complexity of passive learning for the above generic task is $\bigTheta{n}$.
Recently, there has been increasing interest in understanding the complexity of \emph{interactive learning}, which aims to learn under more optimistic, realistic scenarios, in which ``representative'' examples are selected, and the number of examples needed for successful learning may shrink significantly. For example, under the \emph{active learning} setting, the learner only query data points that are helpful for the learning task, which could lead to exponential savings in the sample complexity as compared with the passive learning setting \citep{guillory2009average,jamieson2011active,hanneke2015minimax,kane2017active}. \\
\vspace{-2mm}
\begin{table}[h!]
  \centering
  \vspace{-1mm}
  \begin{tabular}{llll}
    \toprule
    Type     & Average-case     & Worst-case   & Condition on hyperplane arrangement\\
    \midrule
    Passive learning & $\stackrel{{\bf}}{}$\: $\Theta(n)$ & $\stackrel{{\bf}}{}$\: $\Theta(n)$   & - \\
    Active learning
    & $\stackrel{{\bf}}{}$\:  $\bigTheta{d' \log n}$ 
    & $\stackrel{{\bf}}{}$\: $\Theta(n)$
    & $d'$-relaxed general position 
    \\
    Teaching 
    & $\stackrel{{\bf}}{}$\:  $\bigTheta{d'}$ 
    & $\stackrel{{\bf}}{}$\:  $\Theta(n)$ 
    & $d'$-relaxed general position \\
    \bottomrule
  \end{tabular}
    \caption{Sample complexity for various types of data selection algorithms for learning intersection of halfspaces halfspaces. We assume $d'\leq d$ for the $d'$-relaxed general position arrangement.}
  \label{tab:sample-complexity}
\end{table}

\vspace{-2.7mm}An alternative interactive learning scenario is the setting where the learning happens in the presence of a helpful teacher, which identifies useful examples for the learning task. This setting is known as \emph{machine teaching} \citep{DBLP:journals/corr/ZhuSingla18}. Importantly, the label complexity of teaching provides a lower bound on the number of samples needed by active learning \citep{zilles2011models}, and therefore can provide useful insights for designing interactive learning algorithms \citep{brown2019machine}. Machine teaching has been extensively studied in terms of the worst-case label complexity \citep{goldman1995complexity,article:anthony95,zilles2008teaching,doliwa2014recursive,chen2018understanding,mansouri2019preference}. However, to the best of our knowledge, the \emph{average complexity} of machine teaching, even for the fundamental tasks described above, remains significantly underexplored.

In this paper, we investigate the \emph{average teaching} complexity, i.e., the minimal number of examples 
required by a \emph{teacher} to help a learner in locating a \emph{randomly} selected target. We highlight our key results below.
\begin{itemize}
\looseness-1
\vspace{-.5mm}
    \item We show that under the common assumption that the $n$ hyperplanes are in general position in $\reals^d$, the average-case complexity for teaching such a target is $\bigTheta{d}$. This is in sharp contrast to the worst-case teaching complexity of $\bigTheta{n}$ (cf \secref{sec:avgtd}).
\vspace{-.5mm}
    \item We provide a natural extension of the general-position hyperplane arrangement condition, and show that if the $n$ hyperplanes in $\reals^d$ are in ``\emph{$d'$-relaxed general position arrangement}'' where $d'\leq d$, then one can further obtain improved complexity results of $\bigTheta{d'}$ for average-case teaching. Our proof techniques are based on novel insights from computational geometry, which allow us to count the number of convex polytopes and faces in a Euclidean space depending on the hyperplane arrangement. Our result improves upon the existing 
     $\bigO{d}$ result for arbitrary hyperplane arrangement \citep{Fukuda1991BoundingTN}  (cf \secref{sec:avgtd}). 
\vspace{-.5mm}    
    \item To draw a connection with the learning complexity, we show that without the presence of a teacher, a learning algorithm requires $\Theta(n)$ for \emph{i.i.d.} queries and $\Theta(d \log(n))$ for actively chosen queries.
    Table~\ref{tab:sample-complexity} summarizes our main complexity results (cf \secref{sec.learning}).
\vspace{-.5mm}
    \item Based on our proof framework in \secref{sec:avgtd}, 
    we  provide complexity results for teaching $\phi$-separable dichotomies, which recovers and extends the known $\mathcal{O}(d)$ bound on the average number of ``extreme patterns'' in the classical computational geometry literature \citep{cover1965geometrical} (cf \secref{sec.applied}). 
\looseness-2
\end{itemize}

\begin{figure}
\centering
	\subfigure[The generic task]{
		\includegraphics[width=.25\textwidth]{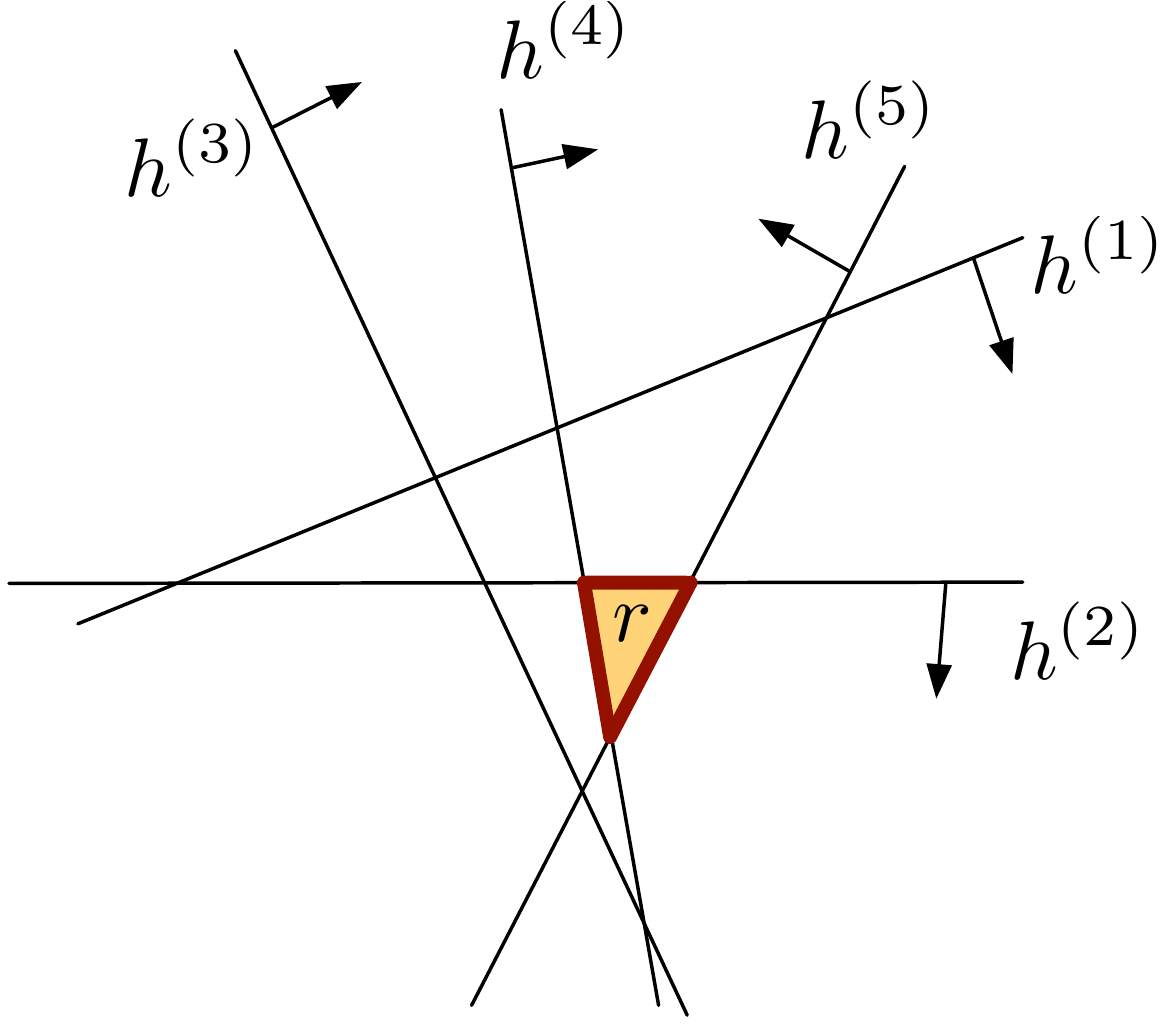}
		\label{fig:example.polytope}
		}\qquad
	\subfigure[Perceptron]{
		\includegraphics[width=.25\linewidth]{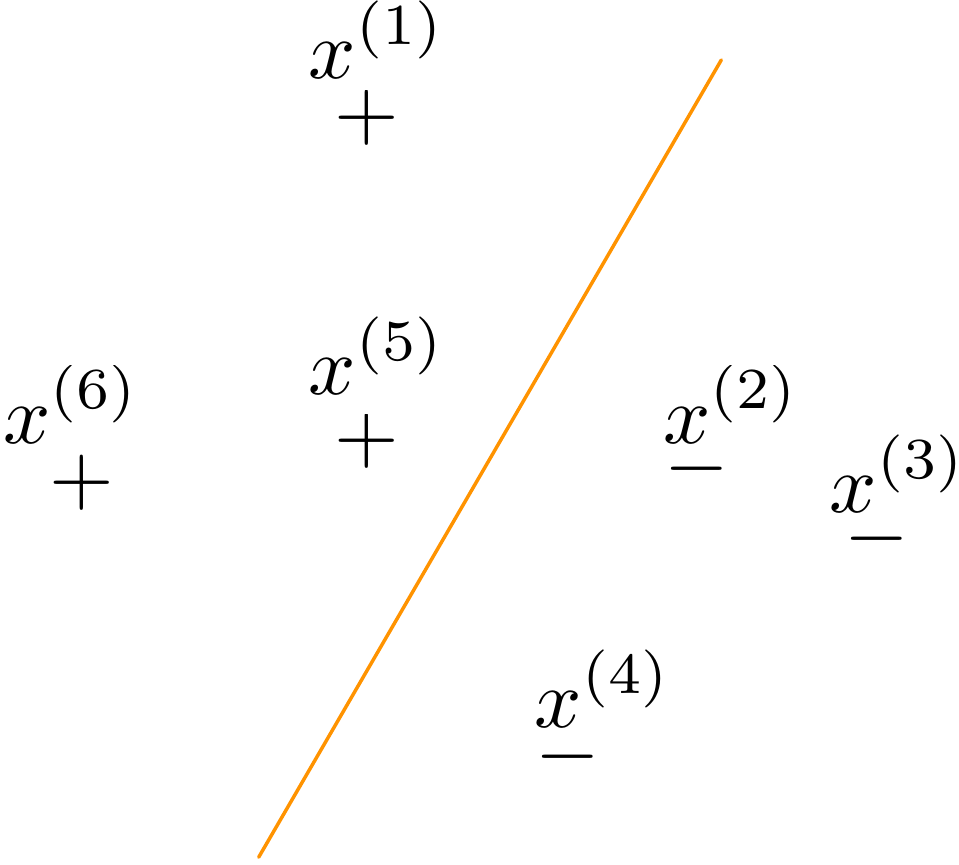}
		\label{fig:example.perceptron}
		}\qquad
	\subfigure[$\phi$-separable dichotomy]{
		\includegraphics[width=.25\linewidth]{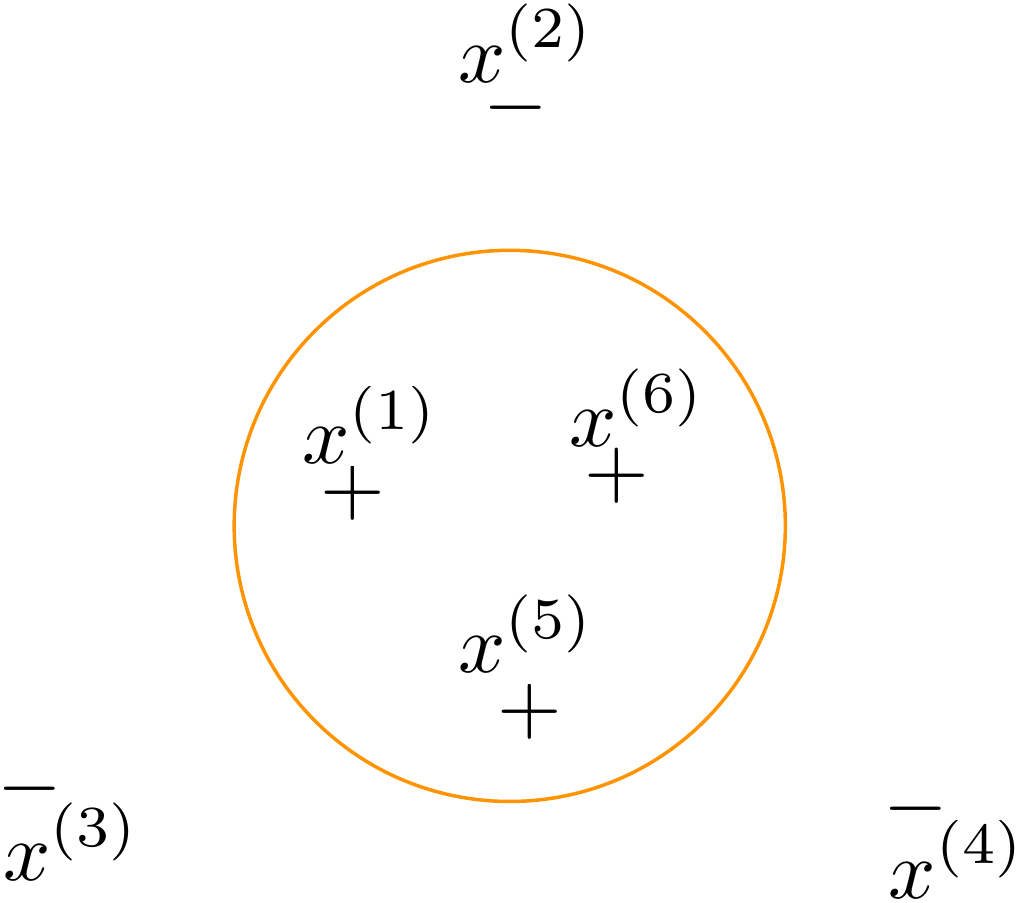}
		\label{fig:example.phi}
		}
  \caption{Different tasks as teaching convex polytopes via halfspace queries.}
	\label{fig:model}
\vspace{-2mm}
\end{figure}


\vspace{-1mm}
\section{Related Work}
\paragraph{Average-case complexity of learning}
While the majority of complexity measures for concept classes and data selection algorithms focus on the worst-case scenarios, there have been a few work concerning the average-case complexity for various types of learning algorithms. Here we provide a survey on related work concerning average-case complexity under the learning setting. \cite{haussler1994bounds} studied how the sample complexity depends on properties of a prior distribution on the concept class and over the sequence of examples the algorithm receives. Specifically, they studied the probability of an incorrect prediction for an optimal learning algorithm using the Shannon information gain. \cite{wan2010learning} considered the problem of learning DNF-formulas sampled from the uniform distribution. 
\cite{nachum2019average} considered the average information complexity of learning (defined as the average mutual information between the input and the output of the
learning algorithm). They show that for a concept class of VC dimension $d$, there exists a proper learning algorithm that reveals $O(d)$ bits of information for most concepts. 
Intuitively, this result aligns with our observation that average complexities of various data selection algorithms are significantly lower than that in the worst-case scenario.
\cite{smoothea,smootheb} introduce the paradigm of smoothed analysis which differs from our average-case analysis as we don't allow perturbations to input spaces. Perhaps most similar to our approach, in terms of technical insights, is the work of \cite{jamieson2011active}, who studied the problem of active ranking via pairwise comparisons, and have used the geometrical properties of hyperplanes in $\reals^d$ to achieve an average complexity of $\bigTheta{d\log n}$ for active ranking over $n$ points.
In our work, we extend their results to the general problem of active learning of halfspaces, and also consider the teaching variant of the ranking via pairwise comparison problem.

\paragraph{Connection with the PAC learning framework} Intersection of halfspaces have been studied in PAC learning framework~\citep{pacpitt,BLUM1997371,KLIVANS2004808,klivanscrypto,vempala,KHOT2011129,learnconvexpolytope}. Although we focus on exact teaching of intersections of halfspaces induced by $n$ hyperplanes,
our results could be readily extended to analyze the average sample complexity for teaching a PAC learner 
under the realizable case. It is well known that a single halfspace can be PAC-learnt efficiently by sampling a polynomial number of data points and finding a separating hyperplane via linear programming \citep{blumer1989learnability}. Relating this to the worst-case sample complexity results in \tableref{tab:sample-complexity}, we know that the worst-case sample complexity for teaching a halfspace to a PAC learner is also polynomial in the VC dimension, i.e., $n=\mathcal{{O}}(\text{poly}(d))$ for halfspaces. One can then extend the average-case complexity results in \tableref{tab:sample-complexity}, based on an argument similar with \emph{pool-based active learning} \citep{mccallumzy1998employing}. The idea is for the teacher to draw $n$ unlabeled examples \emph{i.i.d.} from the underlying data distribution in $\reals^d$. Instead of providing all labels, the teacher provides labels to an optimal teaching set such that all unlabeled examples are implied by the given labels. Thus the learner has obtained $n$ labeled examples drawn \emph{i.i.d.}, and classical PAC bounds still apply.


\looseness-1
\paragraph{Relevant work in algorithmic machine teaching}
As discussed above, teaching problem of various concept classes has been explored before. The  classic definition of average teaching dimension~\citep{goldman1995complexity} which is same as our definition in the uniform setting has been studied in various settings:  \citet{article:anthony95} showed the bound of $\bigO{n^2}$
for the class of linearly separable Boolean functions; \citet{Kushilevitz1996WitnessSF} showed an improved upper bound of $\bigO{|\cC|^{\frac{1}{2}}}$ for
any concept class $\cC$; \citet{kuhlman}  proved that all classes of VC dimension 1 have an average teaching dimension of less than 2; \citet{DNFteach} have shown an $\bigO{ns}$ bound on the class of DNFs with at most $s \le 2^{\bigTheta{n}}$ terms.
In contrast, our work bypasses any dependence on the size of the concept class, and achieves an average teaching complexity of $\bigTheta{d'}$ (where $d' \le d$). 
Some more powerful notions of teaching dimension in sequential setting: recursive and preference-based, have been studied in \citet{doliwa2014recursive,recursiveteach}, which differ from our batched setting. There is increasing interest in connecting the VC dimension to the teaching problem of concept classes \citep[stated in][]{Simon2015OpenPR,Hu2017QuadraticUB}, we notice the VC dimension of $n$ hyperplanes in general position is $\min\{n,d\}$~\citep{edelsbrunner} which is closely related to our average-case $\bigTheta{d'}$ result but away from the worst-case $\bigTheta{n}$ result.


\section{Teaching Convex Polytopes via Halfspace Queries: A General Model}\label{sec:formulation}

\paragraph{Convex polytopes induced by hyperplanes}
Let $h = \condcurlybracket{\h{}}{\eta \cdot \h{}= b,\: \h{}\in \mathbb{R}^d}$ be a hyperplane in $\Rd$, where $\eta\in\reals^d$ 
and $b\in\reals$. We say a point $\h{}\in \Rd$ \tt{satisfies} or \tt{lies} in $h$ if $\h{}\in h$. We define a \tt{halfspace} induced by a hyperplane $h$ to be one of the two connected components of $\big(\Rd - h\big)$ \tt{i.e.} sets corresponding to $\mathrm{sgn}\parenb{\eta \cdot \h{}- b}$. We define  $\Hnd \triangleq \curlybracket{h^{(1)}, h^{(2)}, \dots, h^{(n)}}$ as a set of $n$ hyperplanes in $\mathbb{R}^d$.
The arrangement of the hyperplanes in $\Rd$, denoted as $\cA\parenb{\Hnd}$, induces \tt{intersections of halfspaces} which create connected components. 
Any connected component of $\Rd-\cup_{h \in \Hnd} h$ is defined as a \tt{region} or \tt{convex polytope} in $\Rd$. 
Equivalently, any region $r$ can be exactly specified by the intersections of halfspaces induced by hyperplanes in $\Hnd$. We call the smallest subset $B_r \subseteq \Hnd$ that exactly specifies $r$ the \tt{bounding set of hyperplanes} for $r$. We define connected components induced on hyperplanes (e.g. $h^{(i)}-\cup_{h \in \Hnd\setminus h^{(i)}} h$ for any $h^{(i)} \in \Hnd$) by $\cA\parenb{\Hnd}$ as \tt{faces}. Thus, bounding set $B_r$ forms the faces to the polytope $r$. 
\begin{example}[Convex polytopes induced by hyperplanes]
\figref{fig:example.polytope} provides an example of the arrangement of 5 hyperplanes in $\reals^2$, where arrows on the hyperplanes specify halfspaces. The bounding set for the highlighted region $r$, namely $\{h^{(2)}, h^{(4)}, h^{(5)}\}$, forms 3 faces to $r$.
\end{example}
\indent
We use $\regionset$ to denote the 
regions \tt{induced} by the arrangement $\cA\parenb{\Hnd}$ 
and the number of regions 
$\rnd$ $\triangleq$  $|\regionset|$. We define a \tt{labeling function} $\ell_r:\Hnd \rightarrow \{-1,+1\}$ for an arbitrary region $r \in \regionset$. Note that 
$r$ uniquely identifies its labeling function $\ell_r$.

\paragraph{The teaching framework
}\label{section: teaching framework}
We study the problem of teaching \tt{target regions} (convex polytopes) induced by hyperplane arrangment $\cA\parenb{\Hnd}$ in $\Rd$. Our teaching model is formally stated below.
Consider the set of instances $\Hnd$, with label set $ \mathcal{Y} = \{1,-1\}$ corresponding to two halfspaces induced by a hyperplane. Our hypothesis class, denoted as $\regionset$, is the set of regions induced by $\cA\parenb{\Hnd}$. Consider a target region $r^* \in \regionset$. 
Let $\mathcal{Q} \subseteq \Hnd \times \{1,-1\}$ be the ground set of examples (i.e. labeled instances). 
We define a labeled subset $Q \subseteq \mathcal{Q}$ as \tt{halfspace queries}.
We assume that for any halfspace queries $Q$ $\mathrm{wrt}$ $r^*$, the labels are consistent, \tt{i.e.}, $\forall (h, l) \in {Q}$, $ \ell_{r^*}(h) = l$.
The \tt{version space} induced by ${Q}$ is the subset of regions $\mathbf{\mathrm{VS}}({Q}) \subseteq \regionset$ that are consistent
with the labels of all the halfspace queries \tt{i.e.},
\begin{align}
\mathrm{VS}({Q}) = \condcurlybracket{r \in \regionset}{\forall (h,l) \in {Q},  \ell_r(h) = l}, \nonumber
\end{align}
\tt{or} equivalently, set of convex polytopes which satisfy the halfspace queries ${Q}$. We define our version space learner as one which upon seeing a set of halfspace queries, maintains a version space containing all the regions that are consistent with all the observed queries. Corresponding to a version space learner and a target region $r^*$, we define a \tt{teaching set} $\mathcal{TS}(\Hnd,r^*)$ as a minimal set of halfspace queries such that the resulting version space exactly contains $\{r^*\}$. Formally,
\begin{align}
    \mathcal{TS}(\Hnd,r^*) \in \argmin_{Q \subseteq \mathcal{Q}} |Q|,  \text{~s.t.~}  \mathrm{VS}(Q) = \{r^*\}.\nonumber 
\end{align}
Consequently, 
we want to teach a target hypothesis (regions), say $r^*$  via
specifying halfspace queries in the teaching set $\mathcal{TS}(\Hnd,r^*)$ to a learner. Given a target region $r^*$, the teaching \tt{complexity} \citep{goldman1995complexity} is defined as the sample size of the teaching set  \tt{i.e.} $|\mathcal{TS}(\Hnd,r^*)|$.\vspace{6pt}\\
\noindent
In section \secref{sec:avgtd}, we analyze the teaching complexity of convex polytopes both in the framework of \tt{average-case} and \tt{worst-case}.
We define \emph{average teaching complexity} of convex polytopes via halfspace queries as the expected size of the teaching set \tt{i.e.} $\expctover{r\sim\cU}{|\mathcal{TS}(\Hnd,r)|}$, when the target region $r$ is sampled uniformly at random. 
We define \tt{worst-case teaching complexity} as the worst-case sample size of a teaching set corresponding to target regions from the set of hypotheses.


\paragraph{Hyperplanes in general position}
We adopt a common assumption in computational geometry \cite{feldman2013neural,miller2007geometric}
that the 
hyperplane arrangement is in \tt{general position}, and further provide a relaxed notion of general position hyperplane arrangement, as defined below.
\begin{definition}[General position of hyperplanes \cite{miller2007geometric}]\label{defn general position}
For a set of $n$ hyperplanes $\Hnd$ in $\mathbb{R}^d$, the arrangement $\mathcal{A}(\Hnd)$ is in general position if any subset $\mathcal{S}$ $\subseteq$ $\Hnd$ of $k$ hyperplanes where $1 \le k \le d$, intersects in a $(d-k)$-dimensional plane, otherwise has null intersection.
\end{definition}
\begin{wrapfigure}{h!}{.48\textwidth}
\vspace{-4mm}
\centering
    \subfigure[general]{
		\includegraphics[width=.25\linewidth]{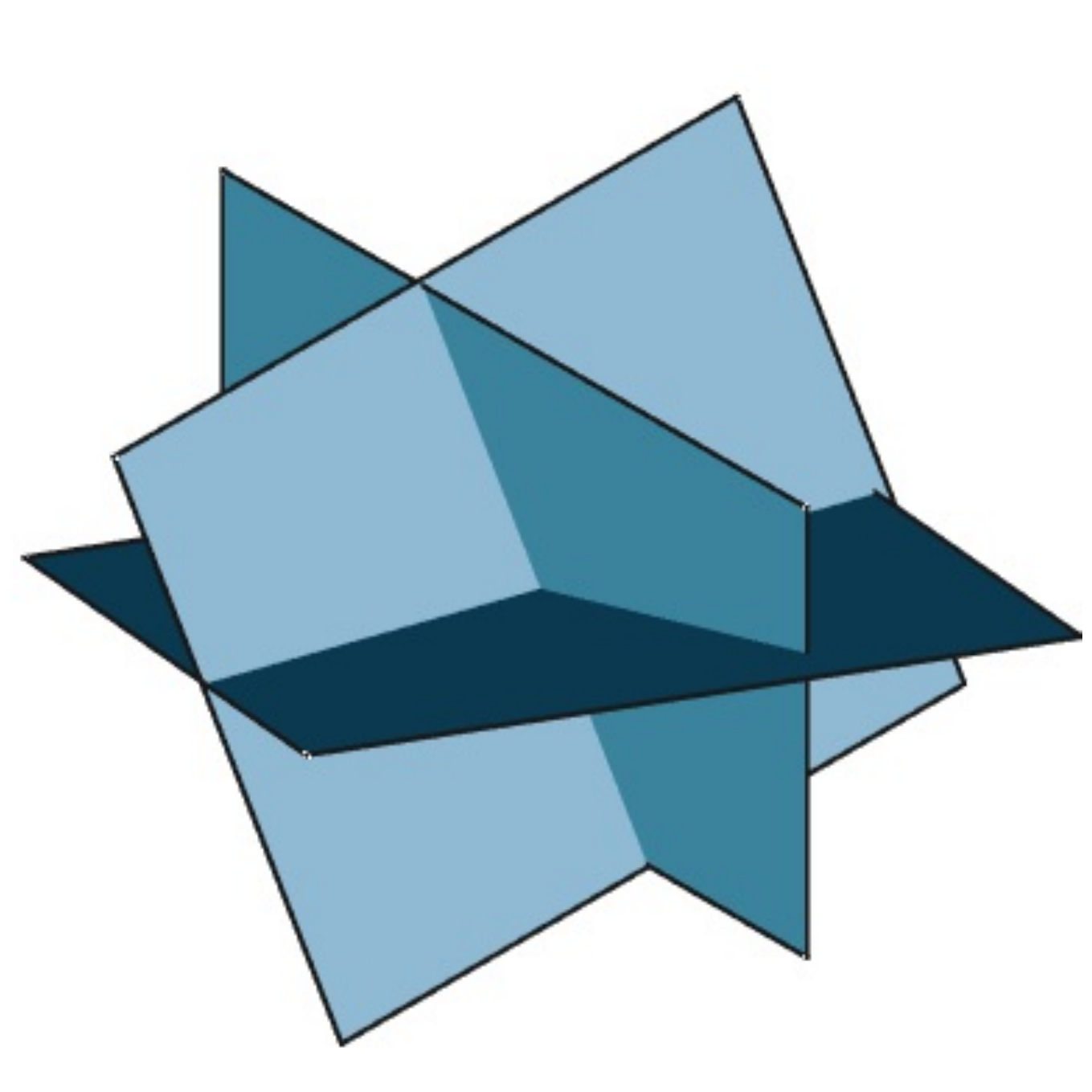}
		\label{fig:example.3general}
		}
	\subfigure[2-relaxed]{
		\includegraphics[width=.25\linewidth]{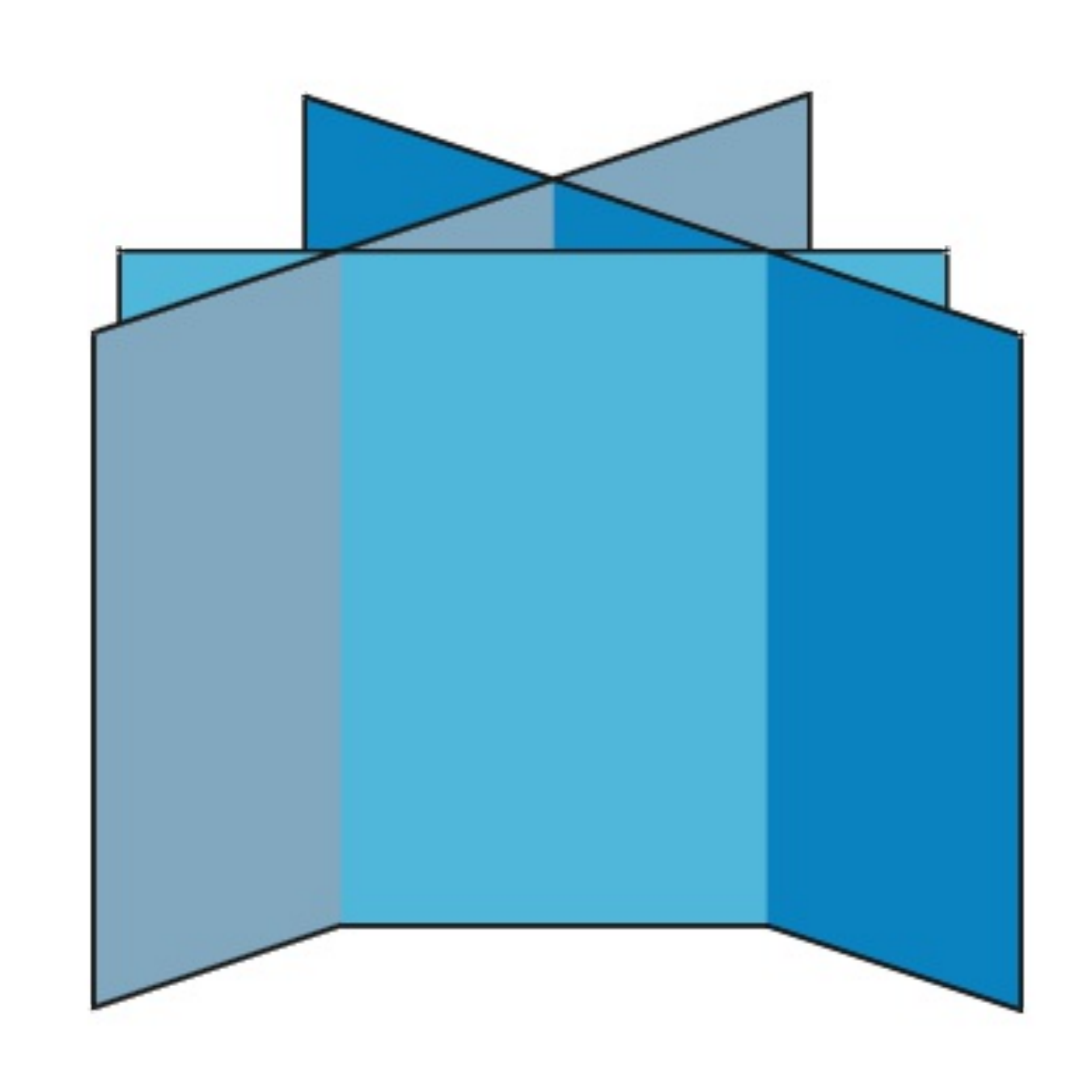}
		\label{fig:example.2general}
		}
	\subfigure[1-relaxed]{
		\includegraphics[width=.25\linewidth]{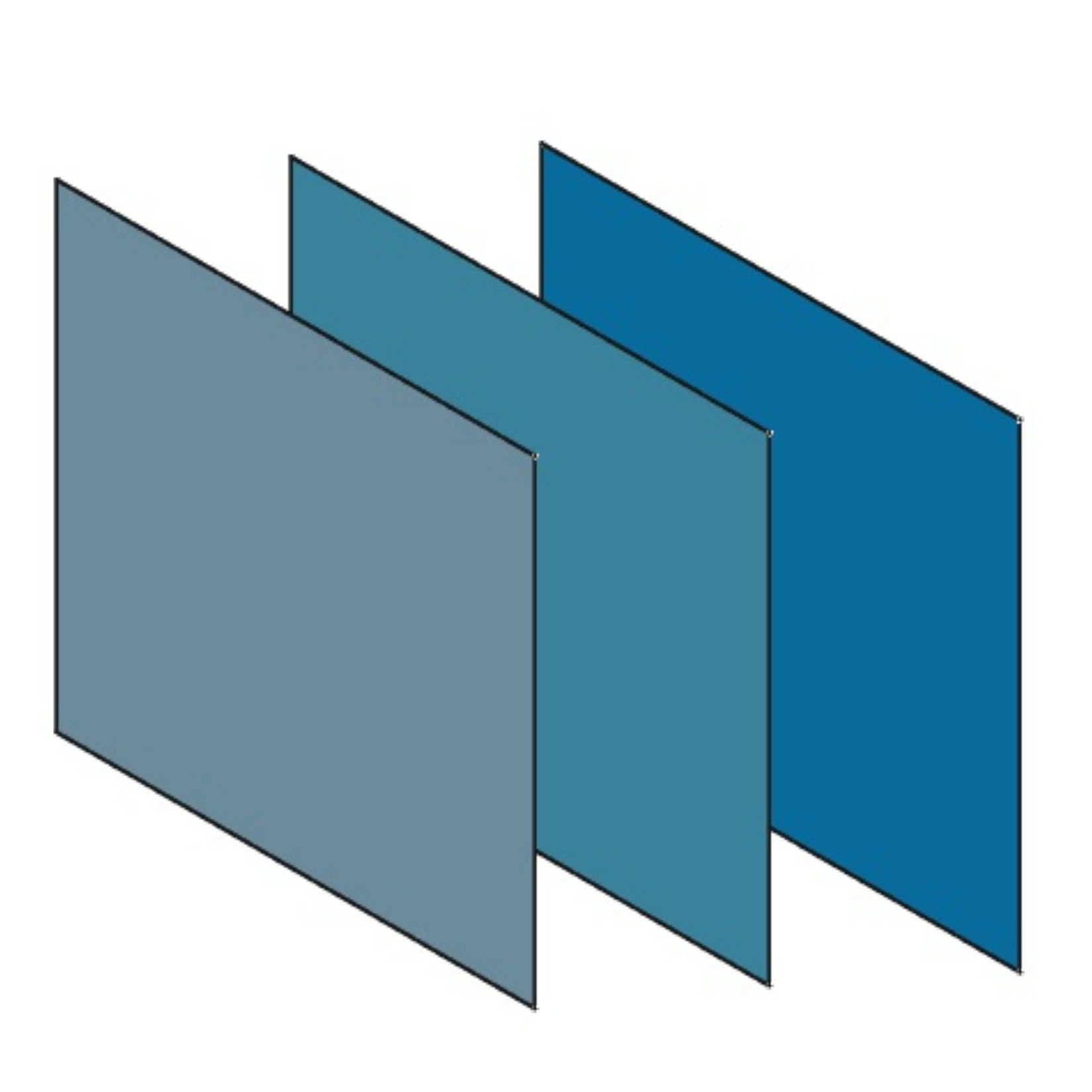}
		\label{fig:example.1general}
		}
		\vspace{-2mm}
	\caption{relaxed general position}\label{fig:illu-relaxed-general}
	\vspace{-6mm}
\end{wrapfigure}
\begin{definition}[Relaxed general position of hyperplanes]\label{defn relaxed general position} 
For a set of $n$ hyperplanes $\Hnd$ in $\mathbb{R}^d$ and $d' \in [d]$, the arrangement $\mathcal{A}(\Hnd)$ is in $d'$-\textit{relaxed general position} if any subset $\mathcal{S}$ $\subseteq$ $\Hnd$ of $k$ hyperplanes where $1 \le k \le d'$, intersects in a $(d-k)$-dimensional plane, otherwise has null intersection.
\end{definition}

As illustrated in \figref{fig:illu-relaxed-general},  
\defref{defn relaxed general position} accounts for arrangements beyond general position (\figref{fig:example.3general}) \tt{e.g} parallel hyperplanes in \figref{fig:example.1general}. \defref{defn general position} is a \tt{special case} of \defref{defn relaxed general position} which we discuss in details in \iftoggle{longversion}{\appref{appendix: active learning}}{the supplemental materials}. \vspace{-3mm}

\section{Average-case Teaching Complexity} \label{sec:avgtd}
In this section, we study the generic problem of teaching convex polytopes via halfspace queries as illustrated in \figref{fig:example.polytope}. Before establishing our main result, we first introduce two important results inherently connected to the average teaching complexity: the number of regions (which corresponds to the target hypotheses) induced by the intersections of $n$ halfspaces, 
and the number of faces 
(which corresponds to the teaching sets) induced by the hyperplane arrangement. 
Our proofs are inspired by ideas from combinatorial geometry and affine geometry, as detailed below. 

\subsection{Regions and Faces Induced by Intersections of Halfspaces}\label{subsection: count regions}
Consider a set of $n$ hyperplanes $\Hnd$ in $\Rd$. 
Generally, it is non-trivial to count the number of regions induced by an arbitrary hyperplane arrangement $\cA\parenb{\Hnd}$. When the hyperplane arrangement is in general position  (\defref{defn general position}, \figref{fig:example.3general}), \cite{miller2007geometric} established an exact result for counting the induced regions. However, it remains a challenging problem to identify the number of regions for more general hyperplane arrangements. However, we show that under the relaxed condition of \defref{defn relaxed general position}, which accounts for various non-trivial arrangements as shown in \figref{fig:example.3general}-\ref{fig:example.1general}, one can exactly count the number of regions.
\begin{theorem}[\bf Regions induced by $d'$-relaxed general position arrangement]\label{thm:d'-general}
Consider a set $\Hnd$ of $n$ hyperplanes in $\mathbb{R}^d$. If the hyperplane arrangement $\mathcal{A}(\Hnd)$ is in $d'$-relaxed general position for some $d' \in \left[d\right]$, then the following holds: $\fr\parenb{\mathcal{A}(\Hnd)} = \sum_{i=0}^{d'} \binom{n}{i}.$
\end{theorem}
\looseness -1 In the following we sketch the proof of \thmref{thm:d'-general}.
The key insight for the proof is
in reducing it to the special case of general position in some $d'$ subspace where $d'\leq d$. We show the reduction by constructing a subspace $\mathbb{N}$ defined as:
$$\mathbb{N} \triangleq \textbf{span}\tuple{\curlybracket{\eta_h \bigm\vert h \in \Hnd,\: h \coloneqq \eta_h\cdot \h{}+ b_h = 0,\: \h{}\in \mathbb{R}^d}}.$$
As a key observation, note that $\mathbb{N}$ is $d'$-dimensional. Let $\Hndd$ be the induced set of hyperplanes in the subspace $\mathbb{N}$ formed by the intersections of $\Hnd$ with $\mathbb{N}$. Therefore, the number of regions induced\footnote{ This idea is more formally studied in the hyperplane arrangement literature as \tt{essentialization} \citep[see][chap: An introduction to hyperplane arrangement]{miller2007geometric}. See \appref{appendixsub: bijection} for further discussion.} by the arrangement of $\Hndd$, denoted as $\fr\parenb{\cA(\Hndd)}$, is exactly $\fr\parenb{\cA(\Hnd)}$. Thus, informatively, it is sufficient to rely on $\cA(\Hndd)$ in $\mathbb{N}$ to understand the intersection of halfspaces induced by $\cA(\Hnd)$ in $\Rd$. We observe that every region $\hat{r} \in \mathcal{A}(\Hndd)$ is contained in exactly one region in $\mathcal{A}(\Hnd)$.
With this observation, we construct the following map $\mathcal{B}$ from the regions induced by the hyperplane arrangement $\mathcal{A}(\Hndd)$, to those induced by $\mathcal{A}(\Hnd)$:
\begin{align*}
  \mathcal{B}:\:\boldsymbol{\mathfrak{R}}\parenb{\mathcal{A}(\Hndd)} &
  \longrightarrow \boldsymbol{\mathfrak{R}}\parenb{\mathcal{A}( \Hnd)} \quad :\quad
  \hat{r} 
  \longmapsto \mathbf{region}_{\mathcal{A}(\Hnd)}(\hat{r}),
\end{align*}
where $\mathbf{region}_{\mathcal{A}(\Hnd)}(\hat{r}) \coloneqq r$ for some $r \in \regionset$ such that $\hat{r} \subseteq r$. 
The following proposition shows that $\mathcal{B}$ is bijective, thereby providing an alternate way to count $\regionset$.
\begin{proposition}\label{bijection}
The map $\mathcal{B}$ (as defined above) is a bijection. Thus,  $\fr\parenb{\mathcal{A}(\Hnd)} = \fr\parenb{\mathcal{A}(\Hndd)}$.
\end{proposition}
Note that, if we can resolve $\fr\parenb{\mathcal{A}(\Hndd)}$ induced by the hyperplane arrangement $\cA\parenb{\Hndd}$, then $\fr\parenb{\mathcal{A}(\Hnd)}$ can be ascertained too. 
The following key lemma, proved in \iftoggle{longversion}{\appref{appendixsub: key lemma}}{the supplemental materials}, shows that $\mathcal{A}(\Hndd)$ is in $d'$-relaxed general position. 
\begin{lemma}
\label{key lemma}
The induced hyperplane arrangement $\mathcal{A}(\Hndd)$ is in $d'$-relaxed general position. 
\end{lemma}
This implies that $\cA\parenb{\Hndd}$ is  structurally the same as $d'$-general position arrangement of $n$ hyperplanes (\tt{i.e.} \defref{defn general position}) in $\mathbb{R}^{d'}$ because any $d'$-dimensional subspace of $\mathbb{R}^{d}$ is \tt{isomorphic} to $\mathbb{R}^{d'}$. Thus, from the relaxed definition of general position, we reduce the problem of counting $\rnd$ to counting 
$\fr\parenb{\cA\parenb{\Hndd}}$ which has the special arrangement of general position. By 
\cite{miller2007geometric} we therefore conclude that $\fr\parenb{\cA\parenb{\Hnd}}$ can be ascertained in an exact form as in \thmref{thm:d'-general}. We defer the full proof of \thmref{thm:d'-general} to \iftoggle{longversion}{\appref{appendix: regions relaxed general position}}{the supplemental materials}.
\paragraph{Faces Induced by $\cA(\Hnd)$} 
We denote by $\ff\parenb{\cA(\Hnd)}$ the number of faces (i.e. regions induced on the hyperplanes) induced by $\cA(\Hnd)$ in $\Rd$.
Consider an arbitrary $h^* \in \regionset$. Note if $\cA(\Hnd)$ is in $d'$-relaxed general position for 
$d' > 1$ then $\forall h \in \Hnd\setminus \{h^*\}$, intersection of $h$ and $h^*$ forms a $(d-2)$-dimensional \tt{flat} on $h^*$ by definition (see \iftoggle{longversion}{\appref{appendixsub: affine geometry}}{supplemental materials} for formal definitions of the relevant affine geometry concepts).
To count the regions induced on $h^*$ is to analyze, $\mathrm{wrt}$  $\Hnd\setminus \{h^*\}$, the $n-1$ \tt{flats} of dimension $(d-2)$; thereby reducing the problem to the case of $n-1$ hyperplanes in $\mathbb{R}^{d-1}$. We would show that these \tt{newly} induced hyperplanes (\tt{i.e. flats}) are in relaxed general position, and thus one can invoke \thmref{thm:d'-general} to count the faces. 
\propref{faces_teach}, as proved in \iftoggle{longversion}{\appref{appendix: faces teaching}}{the supplemental materials}, provides the exact count of faces induced by $\cA(\Hnd)$.
\begin{proposition}[\bf Faces induced by hyperplane arrangement]\label{faces_teach}
Consider a set $\Hnd$ of $n$ hyperplanes in $\mathbb{R}^d$. If the hyperplane arrangement $\mathcal{A}(\Hnd)$ is in $d'$-relaxed general position for some $d' \in \left[d\right]$, the number of faces induced by the arrangement satisfies the recursion:
     $\ff\parenb{\cA(\Hnd)} =  n\cdot \sum_{i=0}^{d'-1}\binom{n-1}{i}.$
\end{proposition}
\subsection{Bound for Average Teaching Complexity: $\bigTheta{d'}$}\label{subsection: main theorem}
\looseness -1
We are now ready to provide our main result on the average-case teaching complexity, when considering teaching convex polytopes induced by hyperplanes in $d'$-\tt{relaxed general position}.
We show that using results in \secref{subsection: count regions}, 
we achieve an \tt{average-case} teaching complexity of $\Theta(d')$ by \algoref{algo: teaching halfspace}. 
\paragraph{Teaching algorithm}
\begin{wrapfigure}{R}{0.4\textwidth}
\scalebox{0.8}{
    \begin{minipage}{0.5\textwidth}
    \vspace{-5mm}
\begin{algorithm}[H]
\nl {\bf Input}: $\Hnd$; random target region $\rgn \in \regionset$  {\label{lst:line:dual01}}\\
\Begin{
\tcp{indentifies  $\mathcal{TS}(\Hnd,\rgn)$ via linear programming} 
\nl $\mathcal{TS}(\Hnd,\rgn) \leftarrow$ {\bf FindTS}$\parenb{\rgn}${\label{lst:line:dual02}}\\
\nl\For {$(h,l)$ $\in$ $\mathcal{TS}(\Hnd,\rgn)$ {\label{lst:line:dual03}}}{
teacher provides halfspace queries $(h,l)$
}
}
\caption{Teaching algorithm}\label{algo: teaching halfspace}
\end{algorithm}
\vspace{-6mm}
\end{minipage}
}
\end{wrapfigure}

Let $\rgn \sim \mathcal{U}$ be a region sampled uniformly at random from $\regionset$. 
To teach $\rgn$, a teacher 
has to provide the halfspace queries in $\mathcal{TS}(\Hnd,\rgn)$. Note that these labels is sufficient to teach $\rgn$ since the version space $\mathrm{VS}(\mathcal{TS}(\Hnd, \rgn)) = \{\rgn\}$. In \algoref{algo: teaching halfspace}, the teacher first collects $\mathcal{TS}(\Hnd,\rgn)$ via subroutine \textbf{FindTS}($\cdot$), and then provides labels to the learner. In particular, the subroutine \textbf{FindTS}($\cdot$) identifies $\mathcal{TS}(\Hnd,\rgn)$ via linear programming:
It checks if each hyperplane intersects the convex body defined by all the $n-1$ constraints (one linear constraint for each hyperplane); each iteration takes polynomial time as it requires solving a linear equation system. In total, it takes $n$ iterations to decide whether any hyperplane is in the teaching set. Thus, the overall computational complexity of this algorithm is $\cO(\textbf{poly}(d)\cdot\textbf{poly}(n))$ (assuming $d$ is smaller than $n$).

\paragraph{Average-case analysis} 
Recall that in section \secref{section: teaching framework}, we defined $B_{r^*}$ to be the bounding set of hyperplanes for the polytope that contains $r^*$.  
To teach $r^*$,  
the teacher has to \tt{identify} the \tt{exact} subset of hyperplanes in $B_{r^*}$ 
(\tt{i.e.} the \tt{faces of the polytope}), 
and provides the halfspace labels corresponding to the hyperplanes in $B_{r^*}$. 
Thus, teaching a target region corresponds to providing labels for the faces of the bounding set. One can ask if there are \tt{pathological} arrangements, where teacher has to provide all the $n$ labels? It turns out that, one can construct arrangements of the hyperplane set $\Hnd$ in $\Rd$ where the \tt{worst-case} teaching complexity is $\bigOmega{n}$ as shown in \thmref{theorem: Main Theorem}. This calls for analyzing the teaching problem under the \tt{average-case}.\vspace{6pt}\\
\noindent
Intuitively, the average teaching complexity of convex polytopes reduces to the average number of faces per region, 
\tt{i.e.} the ratio of number of faces induced on $\Hnd$ to number of regions induced in $\Rd$ by $\cA\parenb{\Hnd}$. In \tt{arbitrary} arrangement of hyperplanes, it is challenging to bound the ratio \ref{eqn:tag:a.1}, as one needs to provide upper bound and lower bound for both terms, and it is unclear how $\ff\parenb{\cA(\Hnd)}$ and $\fr\parenb{\cA(\Hnd)}$ are correlated. However, by imposing the $d'$-\tt{relaxed general position} condition (for any $d' \in \bracket{d}$ ) on the hyperplane arrangement, we can leverage our exact results on counting the regions and faces using \thmref{thm:d'-general} and \propref{faces_teach}:
\begin{align}
    \expctover{\rgn\sim\cU}{|\mathcal{TS}(\Hnd,\rgn)|}
    = \underbrace{\frac{2\cdot\ff\parenb{\cA(\Hnd)}}{\fr\parenb{\cA(\Hnd)}} = 
    \frac{ \text{\propref{faces_teach}}}{\text{\thmref{thm:d'-general}}} }_{\textbf{$d'$-relaxed general position}}.
    \tag{A.1}\label{eqn:tag:a.1}
    \nonumber
\vspace{-5mm}
\end{align}
Ideally, to bound \ref{eqn:tag:a.1}, $\ff(\cdot)$ and $\fr(\cdot)$ need to be appropriately bounded. 
We further show (\iftoggle{longversion}{in the \appref{appendix: main theorem}}{in supplemental}) that for a relaxed general position of hyperplane arrangement, $\ff\parenb{\cA(\Hnd)}$ can be rewritten in terms of $\fr(\cdot)$ 
in lower dimensional space.
Thus, to bound the ratio in \ref{eqn:tag:a.1}, it suffices to bound $\fr(\cdot)$. \corref{cor: bounds regions}, as proved in \iftoggle{longversion}{\appref{appendix: main theorem}}{the supplemental materials}, provides 
tight bounds on $\fr(\cdot)$.
\begin{corollary}\label{cor: bounds regions}
If $\cA(\Hnd)$ is in $d'$-relaxed general position, then $\rnd$ satisfies the following for $n > 2d'$:
$ \binom{n-1}{d'} \le\:\:\rnd\:\: \le \binom{n}{d'}\cdot\frac{n-d'+1}{n-2d'+1}$
\end{corollary}
Let $M_n$ denote the sample size of $\mathcal{TS}(\Hnd, \rgn)$ from \algoref{algo: teaching halfspace} to teach $\rgn\sim \mathcal{U}$, then $\expctover{\cU}{M_n} = \expctover{r\sim\cU}{|\mathcal{TS}(\Hnd,r)|}$. Combining \eqref{eqn:tag:a.1} and \corref{cor: bounds regions}, we obtain our main result below.
\begin{theorem}[\bf Main theorem]\label{theorem: Main Theorem}
Assume $\Hnd$ is in $d'$-\tt{relaxed general position}. Assume $\rgn$ $\sim$ $\mathcal{U}$. Let the random variable $M_n$ denote the number of halfspace queries that are requested in the teaching \algoref{algo: teaching halfspace}, then, 
$\expctover{\cU}{M_n} = \bigTheta{d'},$
\tt{i.e.} the average teaching complexity of convex polytopes is $\bigTheta{d'}$. Furthermore, the worst-case teaching complexity of convex polytopes is $\bigTheta{n}$.
\end{theorem}

\paragraph{Arbitrary position arrangements of hyperplanes} For general position arrangement, exact forms have been established \citep{miller2007geometric,trove.nla.gov.au/work/10888269,10.2307/2303424} for $\rnd$. But it is mentioned in \cite{Fukuda1991BoundingTN} that for \tt{any} arbitrary arrangement one cannot explicitly give a simple formula for $\rnd$ since 
\citet{Vergnas1980ConvexityIO} and \citet{trove.nla.gov.au/work/10888269}
showed that $\rnd$ depends on the underlying matroid structure. Interestingly, via \thmref{thm:d'-general} we establish an exact form for a non-trivial ($d'$-relaxed general position) setting. Apparently, $\mathrm{Theorem}$ 1.2 of \:\citet{Fukuda1991BoundingTN} establishes that for \tt{any} hyperplane arrangement, average teaching complexity of convex polytopes is $\bigO{d}$. In contrast, \thmref{theorem: Main Theorem} provides a stronger bound of $\bigTheta{d'}$ in the $d'$-relaxed general position setting since $d'\le d$. In addition, as further discussed \iftoggle{longversion}{in the \appref{appsubsec: discussion on extension}}{in the supplemental materials} the geometrical insights in the proof of \thmref{thm:d'-general} can be leveraged for extending to more general teaching complexity results. 
\looseness -3



\vspace{-1mm}


\section{Connections to Learning Complexity}\label{sec.learning}
In this section, we consider the problem of learning a convex polytope via halfspace queries, \emph{without} the presence of a helpful teacher. We consider both the passive learning setting where learner makes \emph{i.i.d.} queries and the active learning setting with actively chosen queries, and provide sample complexity results accordingly.
\paragraph{Learning convex polytopes via halfspace queries} Consider the hyperplane set $\Hnd$ in $\Rd$ and a target region $r^* \in \regionset$. For any hyperplane $h \in \Hnd$
where $h = \condcurlybracket{\h{}}{\eta_h \cdot \h{}= b_h,\: \h{}\in \mathbb{R}^d}$, the labeling function $\ell_{r^*}$, as defined in \secref{sec:formulation}, specifies its label (halfspace) as $\ell_{r^*}(h) = $ $\mathrm{sgn}\parenb{\eta_h \cdot r^*- b_h}$.
The problem of learning a region $r^*$ therefore reduces to identifying the corresponding labeling function $\ell_{r^*}$.
The objective here is to learn the region by querying the reference
of the form
    $q_{h}:= \mathbf{1}\curlybracket{\ell_{r^*}(h) = 1}$,
where $h \in \Hnd$ and $\mathbf{1}\curlybracket{\cdot}$ is the indicator function. 
Similar to the teaching setting, we assume that the target $\rgn^*$ is sampled uniformly at random. 
In the following, we establish sample complexity results, i.e., on the minimal number of halfspace queries required to determine a target region, under the settings of \tt{active} and \tt{passive} learning.\vspace{-2mm}
\subsection{Active Learning of Convex Polytopes}\label{sec:active learning}
In \secref{subsection: main theorem}, we showed that worst-case teaching complexity for convex polytopes is $\bigOmega{n}$,  
this directly implies the lower bound of $\bigOmega{n}$ on the worst-case for active learning.
We now show that when the underlying hyperplane arrangement is in \gen{}, the average-case complexity of active learning has only a $\log n$ dependency on the number of hyperplanes. 
\looseness -1 We achieve this by actively selecting 
informative 
queries---a similar characterization of the ambiguous queries as considered by \citet{jamieson2011active} for the pairwise ranking problem.
Concretely, we consider the following
querying strategy: For an (unknown) target region $\rgn$ $\in \regionset$  and a uniformly random ordering of hyperplanes $\Hnd$, the learner checks in each iteration if a query $q_{h}$  is ambiguous for randomly selected $h \in \Hnd$ (i.e. intersects the convex body defined by hyperplanes sampled previously); then asks or imputes the labels depending on their ambiguity.\vspace{6pt}\\
\noindent
In any iteration $k$ of the above query selection procedure\footnote{Full algorithm is detailed in \appref{appendixsub: characterization}.}, denote the event of requesting the query for a sampled hyperplane $h^{(k)}$ by $B_k$. That is, $B_k = {\bf 1}\curlybracket{q_{h^{(k)}} \text{is requested}}$. Note that each $B_k$ is a Bernoulli distribution with \tt{unknown} parameter $(\star)$ to be ascertained. If we can bound $(\star)$ then we bound the expected number of queries as well. We define by  $\mathcal{S} \subseteq \Hnd$ of size $k$ as the set of hyperplanes sampled by the procedure. We notice that the $(k+1)^{\mathrm{th}}$ sampled hyperplane is ambiguous if it intersects the convex body defined by hyperplanes in $\mathcal{S}$. Thus we want to bound the probability of the event that the query $q_{h^{(k+1)}}$ is ambiguous. Denote the probability of such an event as $P_A(k,d,\mathcal{U})$. Notice $P_A(k,d,\mathcal{U})$ is our $(\star)$ here. In \lemref{probability ambiguity hyperplane}, we show that $P_A(k,d,\mathcal{U})$ is upper bounded by a factor of $1/k$ for relatively small sample size $k$. 

\begin{lemma}[Probability of ambiguity]\label{probability ambiguity hyperplane}
 Assume  $\rgn \sim \mathcal{U}$. Let $P_A(k, d, \mathcal{U})$ denote the probability of the event that the query $q_{h^{(k+1)}}$ is ambiguous where $h^{(k+1)}$ is the $(k+1)^{\mathrm{th}}$ sampled hyperplane. If $\Hnd$ is in $d'$-relaxed general position, then there exists a positive, real
number constant $a$ independent of $k$ such that for $k > 2d'$, $P_A(k, d, \mathcal{U}) \le a\cdot\frac{d'}{k}$.
\end{lemma}
\lemref{probability ambiguity hyperplane} allows us to bound the expected value of $\sum_{j=1}^n {\bf 1}\{q_{h^{(j)}}\text{is requested}\}$. As detailed in \iftoggle{longversion}{\appref{appendix: active learning}}
{the supplemental materials}, we show that for hyperplane arrangement in $d'$-relaxed general configuration, the expected value for $\sum_{j=1}^n {\bf 1}\{q_{h^{(j)}}\text{is requested}\}$ is $\bigTheta{d'\log n}$. 
This leads to the following complexity results for active learning. 
\begin{theorem}\label{pairwise teaching bound}
Assume $\rgn$ $\sim$ $\mathcal{U}$ and that the underlying hyperplane arrangement of $\Hnd$ is in $d'$-relaxed general position.
Let
$M_n$ denote a random variable for the number of queries that are requested in the query selection procedure in \secref{sec:active learning}, then $\mathbb{E}_{\mathcal{U}}[M_n] = \bigTheta{d'\log n}$,
\tt{i.e.} the average-case query complexity of convex polytopes is $\bigTheta{d'\log n}$.
Moreover, the worst-case query complexity is $\bigTheta{n}$.
\vspace{-2mm}
\end{theorem}
\subsection{Passive learning of convex polytopes} 
\vspace{-1mm}
In the case of passive learning of a target region, the \tt{average} sample complexity is trivially lower bounded by $\bigOmega{n}$ since the learner gets a label uniform at random. 
Since there are $n$ hyperplanes $n$ samplings are sufficient to get all the labels which trivially give a $\cO(n)$ solution. Thus, it is not very difficult to see that in the case of passive learning the average sample complexity is $\bigTheta{n}$.

\vspace{-2mm}
\section{Teaching $\phi$-separable Dichotomy as Teaching Convex Polytopes}\label{sec.applied}
\vspace{-1mm}
In \secref{sec:avgtd}, we discussed the generic problem of teaching convex polytopes induced by intersections of halfspaces via halfspace queries.
We now consider the problem of $\phi$-separability of points (also see \figref{fig:example.perceptron}-\ref{fig:example.phi}) which could be viewed as a variant of teaching convex polytopes. We achieve similar average-case teaching complexity results for the problem. In the seminal work \cite{cover1965geometrical}, Cover studied the problem of $\phi$-separability of points in which the task is to classify points using various types of classifiers (linear or non-linear).\vspace{6pt}\\
\noindent
We first provide useful definitions for the domain of discussion. We define a set of $n$ points in $\mathbb{R}^d$ as 
$\xnd \triangleq \curlybracket{{\dt{}}^{(1)}, {\dt{}}^{(2)}, \dots, {\dt{}}^{(n)}}$ (referred to as \tt{data space}),  
and use ${\dt{}}_{\bracket{d-1}}$ to represent the first $d-1$ coordinates of a point ${\dt{}} \in \Rd$.
A map $\phi: \xnd \rightarrow \mathbb{R}^{d_{\phi}}$, is called $\phi$-map, and the subset $\phi(\xnd) \subset\mathbb{R}^{d_{\phi}}$ is called $\phi$-induced space. A dichotomy (i.e., a disjoint partition of a set) $\{\xndp,\xndn\}$ of $\xnd$ is $\phi$-\textit{separable} if there exists a vector (aka \emph{separator} of the dichotomy) 
$w \in \mathbb{R}^{d_{\phi}}$ such that: \text{if}\: ${\dt{}} \in \xndp$ then $ w\cdot\phi({\dt{}}) > 0$ and if ${\dt{}} \in \xndn$ then $ w\cdot\phi({\dt{}}) < 0$.
\begin{definition}[Relaxed general position of points]\label{general position vectors}
For a set of $n$ data points in $\mathbb{R}^d$, say $\xnd$, is in $d'$-\tt{general position}\footnote{See \cite{cover1965geometrical} for the definition of general position of points.} for a fixed $d' \in [d]$ if every $d'$ subset of $\xnd$ is linearly independent.
\end{definition}
\vspace{-4mm}
\begin{definition}[Relaxed $\phi$-general position]\label{phi general}
Consider a set of $n$ data points $\xnd$ in $\Rd$. For a $\phi$-map in $\mathbb{R}^{d_{\phi}}$, $\xnd$ is said to be in $d'_\phi$-relaxed $\phi$-general position for a fixed $d'_\phi \in [d_\phi]$ if every $d'_{\phi}$ subset of $\phi$-induced points $\phi(\xnd)$ is linearly independent.
\end{definition}
\vspace{-2mm}
We consider the problem of teaching $\phi$-separable dichotomy as providing labels to subset $E \subset \xnd$ such that a separator $w_{\phi}$ can be taught which separates the entire dichotomy. 
In the remaining of this section, we show that the teaching problem of $\phi$-separability of dichotomies (\figref{fig:example.perceptron}-\ref{fig:example.phi}) can be studied as a special case of teaching convex polytopes. We connect the two problems via duality. 
Notice that showing the duality for homogeneous linear separability of dichotomies \tt{i.e} $\phi = \boldsymbol{\mathrm{Id}}$ (identity function) suffices for general $\phi$-separability since it reduces to the homogeneous case. 

Naturally, we define \tt{teaching set} for a $\phi$-separable dichotomy as the teaching set for the dual convex polytopes of the $\phi$-induced space. Following the standard practice, we call the hypothesis space (where each hypothesis/region corresponds to a $w$)
as the \tt{dual space}, and data space as the \tt{primal space}.
We discuss the construction and relevant properties of duality below.\vspace{6pt}\\
\noindent
$\mathrm{WLOG}$ 
we assume that ${\dt{}}^{(n)}$ = $\mathbi{e}_d$ (standard basis vector in $\Rd$ 
with coordinate $d$ being 1 and others being 0). Denote the set of all homogeneously linear separable dichotomies of $\xnd$ by $\boldsymbol{\mathfrak{D}}_{\xnd}$. We observe that if $w$ is a linear separator of $\{\xndp,\xndn\}$, then $-w$ forms a linear separator for $\{\xndn,\xndp\}$.
Based on this observation, we define a \tt{relation} $\backsim$ on elements of $\boldsymbol{\mathfrak{D}}_{\xnd}$ as follows: $\mathbi{u}, \mathbi{v} \in \boldsymbol{\mathfrak{D}}_{\xnd}\:\: \text{then}\:\: \mathbi{u} \backsim \mathbi{v} \Longleftrightarrow  
\textit{if}\:\: w\:\: \text{separates}\:\: \mathbi{u},\:\: \textit{then}\:\: w\:\:\text{or}\:\: -w \:\: \text{separates}\:\: \mathbi{v}$. Notice that $\backsim$ is \tt{reflexive, symmetric,} and \tt{transitive}. Thus, $\backsim$ is an \tt{equivalence relation}. Denote by $\boldsymbol{\mathfrak{E}}\parenb{\xnd}$ the set of equivalence classes \tt{i.e.} the quotient set \citep[see][]{rossen2003discrete} $\boldsymbol{\mathfrak{D}}_{\xnd}/\backsim$. It is easy to see that $\#\bracket{\mathbi{v}}= 2$, where $[\mathbi{v}]$ denotes an equivalence class for any $\mathbi{v} \in \boldsymbol{\mathfrak{D}}_{\xnd}$. Before we construct the dual map, $\mathrm{wlog}$, we state a key assumption used in construction as follows:
\begin{assumption}\label{assumption:positivelabel}
We represent each equivalence class by the dichotomy which labels $\dt{}^{(n)}$ as positive.
\vspace{-2mm}
\end{assumption}
\vspace{-2.5mm}
This implies that if $w = (w_1,\cdots,w_d) \in \Rd$ is a homogeneous linear separator of the \tt{representative} dichotomy of a class then $w_d > 0$ as $w\cdot \dt{}^{(n)} > 0$.
Thus, dual map exploits this property of each equivalence class i.e.
\vspace{-1mm}
\begin{align}
    w\cdot \dt{} = \dt{}\cdot w &= \parenb{\dt{}_{[d-1]}, \dt{}_{d}}\cdot \parenb{w_{\bracket{d-1}}/{w_d}, 1}  &\lessgtr 0 \nonumber\\
    &\Rightarrow \dt{}_{[d-1]}\cdot\parenb{w_{\bracket{d-1}}/{w_d}} + \dt{}_d \triangleq h_{[d-1]}\cdot \h{}_{w} + \dt{}_{d}&\lessgtr 0 \label{eqn:construction}
\end{align}
Hence, points $\dt{} \in\Rd$ maps to hyperplane $ h_{\dt{}} \triangleq h_{[d-1]}\cdot \h{}+ {\dt{}}_{d} = 0$, $\h{} \in \mathbb{R}^{d-1}$ in $\mathbb{R}^{d-1}$ in the dual space
and homogeneous linear hyperplane $w\cdot \dt{} = 0$ maps to point $\h{}_w = w_{\bracket{d-1}}/{w_d}$ in $\mathbb{R}^{d-1}$. Notice that, $\dt{}^{(n)}$ maps to a hyperplane which exists in infinity i.e $h^{(n)}_{[d-1]} = \mathbf{0}^{d-1}$.
Denote the set of dual hyperplanes by $\boldsymbol{\mathcal{H}}_{n-1,d-1}$ ($=: \boldsymbol{\bar{\mathcal{H}}}$)\footnote{We use this notation to signify that $x^{(n)}$ exists in infinity.}. Formally, we define our dual map $\bracket{\Upsilon_{\mathrm{dual}} ,\varphi_{\mathrm{dual}}}$ as follows:
\begin{align}
   \Upsilon_{\mathrm{dual}}: \xnd &\rightarrow \boldsymbol{\bar{\mathcal{H}}}
   &\varphi_{\mathrm{dual}} : \boldsymbol{\mathfrak{E}}_{\xnd} &\rightarrow \boldsymbol{\mathfrak{R}}\parenb{\mathcal{A}(\boldsymbol{\bar{\mathcal{H}})}}\nonumber\\ 
   x &\mapsto h_\dt{} &
    \bracket{\mathbi{v}}:w_{\bracket{\mathbi{v}}}  &\mapsto r_{\h{}_{\bracket{\mathbi{v}}}} \iftoggle{longversion}{\tag{D.M}\label{tag: d.m}}{\nonumber}  
\end{align}
where $\h{}_{\bracket{\mathbi{v}}} \in r_{\h{}_{\bracket{\mathbi{v}}}} \in \boldsymbol{\mathfrak{R}}\parenb{\mathcal{A}(\boldsymbol{\bar{\mathcal{H}})}}$ and $\h{}_{\bracket{\mathbi{v}}}$ 
is dual point of the separator  $w_{\bracket{\mathbi{v}}}$ to $\bracket{\mathbi{v}}$.
We state the main result on dual map in \thmref{dual map} below with detailed proofs in \appref{appendix: Dual map}.
\begin{theorem}[Dual map]\label{dual map} Consider a set of $n$ points $\xnd$ in $\Rd$ in $d'$-relaxed general position. The hyperplane arrangement induced by  $\boldsymbol{\mathcal{H}}_{n-1,d-1} = \Upsilon_{\mathrm{dual}}\parenb{\xnd}$ is in ($d'-1$)-relaxed general position. Moreover, $\varphi_{\mathrm{dual}}$ is a bijection.
\end{theorem}
\thmref{dual map} claims that $\Upsilon_{\mathrm{dual}}\parenb{\xnd}$ is in $(d'-1)$-\tt{relaxed general position}. Combining the above result with \thmref{theorem: Main Theorem}, and the observation that any $\phi$-separability reduces to the homogeneous case, we obtain the average teaching complexity of $\cO(d'_{\phi})$ for $\phi$-separable dichotomy. 
\begin{corollary}[Teaching $\phi$-separable dichotomies]\label{cor: teach phi surfaces} 
Consider a $\phi$-map in $\mathbb{R}^{d_{\phi}}$. Assume that $\xnd$ are in $d'_\phi$-relaxed $\phi$-general position for a fixed $d'_\phi \in [d_\phi]$. If $\boldsymbol{\mathfrak{E}}_{\xnd}^{\phi}$ denotes the set of $\phi$-separable dichotomies of $\xnd$, then the average teaching complexity of dichotomies from $\boldsymbol{\mathfrak{E}}_{\xnd}^{\phi}$ is $\mathcal{O}(d'_\phi)$ \tt{i.e.} $\mathop{\mathbb{E}}_{\rgn_{[\mathbi{u}]} \sim \mathcal{U}}[M_n] = \mathcal{O}(d'_\phi),$ where $M_n$ denotes the number of teaching labels for a class $\rgn_{[\mathbi{u}]} \in \boldsymbol{\mathfrak{E}}_{\xnd}^{\phi}$. 
\end{corollary}
\textbf{Remark}: In \secref{subsection: main theorem}, we discussed that for $\tt{any}$ hyperplane arrangement, \citet{Fukuda1991BoundingTN} established $\cO(d)$ result for 
the average teaching complexity of convex polytopes. We can obtain similar result \noindent 
for \tt{any}  arrangement of points for separable dichotomies via duality. The average teaching complexity of linear-separable dichotomies using duality can be established to $\cO(d)$ (similarly $\cO(d_{\phi})$ for $\phi$-separable dichotomies).

\looseness -1
\paragraph{Connection to the notion of extreme points of \cite{cover1965geometrical}}\label{subsection:connection} %
We now establish the connection between teaching set in the dual space and 
the \emph{extreme points} in the primal space.
This implies that our result on the average teaching complexity in \corref{cor: teach phi surfaces} recovers the $\bigO{d_{\phi}}$ result on the average number of extreme points, which was proved via a different framework in \cite{cover1965geometrical}. 
\begin{definition}[Extreme points]\label{def:extreme-points}
Consider an arbitrary $\phi$-separable dichotomy $\{\xndp,\xndn\}$ of a set of points $\xnd$ in $\Rd$. We say a subset $E \subset \xnd$ to be extremal points $\mathrm{wrt}$ $\{\xndp,\xndn\}$ if it is minimal and $\{\xndp,\xndn\}$ is $\phi$-separable by $w_{\phi}$ iff 
$\{\xndp \cap E,\xndn \cap E\}$ is $\phi$-separable by $w_{\phi}$.
\end{definition}
According to Lemma 1 \citep{cover1965geometrical}, a point $y$ is in the minimal set $E$ of extreme points for a dichotomy $\{X^+,X^-\}$ if it is \textit{ambiguous} $\mathrm{wrt}$ the dichotomy \tt{i.e.} both $\{X^+\cup \{y\}, X^- \}$ and $\{X^+, X^-\cup \{y\} \}$ are homogeneously linearly separable. We show that this characterization of ambiguous points is \tt{equivalent} to a characterization of hyerplanes in the dual space: 
\begin{definition}[Ambiguous hyperplanes in the dual space]\label{ambiguous hyperplane} Let $\mathcal{H}$ be a set of hyperplanes in $\Rd$, and let $r^{*}$ be a region induced by the hyperplane arrangement $\cA\parenb{\mathcal{H}}$. Then, an arbitrary hyperplane $h'$ is informative or ambiguous with respect to $r^{*}$ iff $\exists$ a point $\h{}$ in $h'$ such that a normed ball $\mathbb{B}_2\parenb{\h{}, \epsilon} \subset r^{*}$ for some $\epsilon > 0$.
\end{definition}
Note that only an ambiguous hyperplane can be contained in the teaching set for $r^{*}$. 
To achieve the equivalence of the two characterizations provided in \defref{def:extreme-points} and \defref{ambiguous hyperplane}, 
our key insight is in noting that \eqnref{eqn:construction} preserves signs of dot products in both the primal and dual spaces. Using this, we realize that (i) every ambiguous data point to dichotomy $\{X^+,X^-\}$ passes through the dual region corresponding to it, and (ii) similarly, every ambiguous hyperplane can be shown to form a data point which intersects a separator of $\{X^+,X^-\}$.
Formally, we establish the connection via the following theorem below with detailed discussions and proofs deferred to \iftoggle{longversion}{\appref{appendix: extreme points}}
{the supplemental materials}.
\begin{theorem}\label{main extreme points}
Consider a set of $n$ points $\xnd$ in $\Rd$ and a $\phi$-map where $\phi: \xnd \rightarrow \mathbb{R}^{d_{\phi}}$. Assume that $\xnd$ are in $d_\phi$-relaxed $\phi$-general position (\defref{phi general}). Let $\{\xndp,\xndn\}$ be a $\phi$-separable dichotomy. Now, for a subset $E \subseteq \xnd$, $E$ is a set of extremal points iff $\Upsilon_{\mathrm{dual}}(E)$ with the appropriate labels forms a teaching set for $\varphi_{\mathrm{dual}}\parenb{\big[{\{\xndp,\xndn\}}\big]}$.
\end{theorem}

\section{Discussion and Conclusion}\label{sec.conclusions}

We have studied the average-case complexity of teaching convex polytopes with halfspace queries, and showed that if the hyperplane arrangement is in $d'$-relaxed general position, then the average teaching complexity is $\bigTheta{d'}$. In contrast, the average-case sample complexity is $\bigTheta{d'\log n}$ for active learning and $\bigTheta{n}$ for passive learning. We showed that our insights could be applied to teaching $\phi$-separable dichotomies. Moreover, as discussed in details in the \appref{appendix: teaching ranking}, we further show that our insights in \secref{sec:avgtd} could be further generalized to the problem of teaching rankings over $n$ points $\{x_1, \dots, x_n\} \subseteq \mathbb{R}^d$ (encoded by their distances to an \emph{unknown} reference point $r \in \mathbb{R}^d$) via pairwise comparisons (e.g., ``is $x_i$ closer to $r$ than $x_j$''?). 
One interesting line of future work is to understand whether our result could be extended to more general hyperplane arrangement settings. We believe our results provide useful geometrical insights for analyzing the average-case complexity for more complex hypothesis classes.


\section*{Acknowledgements}
We thank Ali Sayyadi for the helpful discussions. This work was supported in part by fundings from PIMCO and Bloomberg.
\bibliography{reference}

\begin{thebibliography}{52}
\providecommand{\natexlab}[1]{#1}
\providecommand{\url}[1]{\texttt{#1}}
\expandafter\ifx\csname urlstyle\endcsname\relax
  \providecommand{\doi}[1]{doi: #1}\else
  \providecommand{\doi}{doi: \begingroup \urlstyle{rm}\Url}\fi

\bibitem[Angluin(1987)]{ANGLUIN198787}
Dana Angluin.
\newblock Learning regular sets from queries and counterexamples.
\newblock \emph{Information and Computation}, 75\penalty0 (2):\penalty0 87 --
  106, 1987.
\newblock ISSN 0890-5401.
\newblock \doi{https://doi.org/10.1016/0890-5401(87)90052-6}.

\bibitem[Angluin(1988)]{ANGLUINb}
Dana Angluin.
\newblock Queries and concept learning.
\newblock \emph{Mach. Learn.}, 2\penalty0 (4):\penalty0 319–342, April 1988.
\newblock ISSN 0885-6125.
\newblock \doi{10.1023/A:1022821128753}.

\bibitem[Anthony et~al.(1995)Anthony, Brightwell, and
  Shawe-Taylor]{article:anthony95}
Martin Anthony, Graham Brightwell, and John Shawe-Taylor.
\newblock On specifying boolean functions by labelled examples.
\newblock \emph{Discrete Applied Mathematics}, 61:\penalty0 1--25, 07 1995.
\newblock \doi{10.1016/0166-218X(94)00007-Z}.

\bibitem[Bishop(2006)]{bishop2006pattern}
Christopher~M Bishop.
\newblock \emph{Pattern recognition and machine learning}.
\newblock springer, 2006.

\bibitem[Blum and Kannan(1997)]{BLUM1997371}
Avrim~L. Blum and Ravindran Kannan.
\newblock Learning an intersection of a constant number of halfspaces over a
  uniform distribution.
\newblock \emph{Journal of Computer and System Sciences}, 54\penalty0
  (2):\penalty0 371 -- 380, 1997.
\newblock ISSN 0022-0000.
\newblock \doi{https://doi.org/10.1006/jcss.1997.1475}.

\bibitem[Blumer et~al.(1989)Blumer, Ehrenfeucht, Haussler, and
  Warmuth]{blumer1989learnability}
Anselm Blumer, Andrzej Ehrenfeucht, David Haussler, and Manfred~K Warmuth.
\newblock Learnability and the vapnik-chervonenkis dimension.
\newblock \emph{Journal of the ACM (JACM)}, 36\penalty0 (4):\penalty0 929--965,
  1989.

\bibitem[Brown and Niekum(2019)]{brown2019machine}
Daniel~S Brown and Scott Niekum.
\newblock Machine teaching for inverse reinforcement learning: Algorithms and
  applications.
\newblock In \emph{Proceedings of the AAAI Conference on Artificial
  Intelligence}, volume~33, pages 7749--7758, 2019.

\bibitem[Buck(1943)]{10.2307/2303424}
R.~C. Buck.
\newblock Partition of space.
\newblock \emph{The American Mathematical Monthly}, 50\penalty0 (9):\penalty0
  541--544, 1943.
\newblock ISSN 00029890, 19300972.

\bibitem[Chen et~al.(2018)Chen, Singla, Mac~Aodha, Perona, and
  Yue]{chen2018understanding}
Yuxin Chen, Adish Singla, Oisin Mac~Aodha, Pietro Perona, and Yisong Yue.
\newblock Understanding the role of adaptivity in machine teaching: The case of
  version space learners.
\newblock In \emph{Advances in Neural Information Processing Systems}, pages
  1476--1486, 2018.

\bibitem[Cohen et~al.(2013)Cohen, Denham, Falk, Schenck, Suciu, Terao, and
  Yuzvinsky]{book}
Daniel Cohen, Graham Denham, Michael Falk, Hal Schenck, Alex Suciu, Hiroaki
  Terao, and Sergey Yuzvinsky.
\newblock \emph{Complex Arrangements: Algebra, Geometry, Topology}.
\newblock 08 2013.

\bibitem[Cover(1965)]{cover1965geometrical}
Thomas~M Cover.
\newblock Geometrical and statistical properties of systems of linear
  inequalities with applications in pattern recognition.
\newblock \emph{IEEE transactions on electronic computers}, \penalty0
  (3):\penalty0 326--334, 1965.

\bibitem[Doliwa et~al.(2014)Doliwa, Fan, Simon, and
  Zilles]{doliwa2014recursive}
Thorsten Doliwa, Gaojian Fan, Hans~Ulrich Simon, and Sandra Zilles.
\newblock Recursive teaching dimension, vc-dimension and sample compression.
\newblock \emph{JMLR}, 15\penalty0 (1):\penalty0 3107--3131, 2014.

\bibitem[Edelsbrunner(1987)]{edelsbrunner}
Herbert Edelsbrunner.
\newblock \emph{Algorithms in Combinatorial Geometry}.
\newblock Springer-Verlag, Berlin, Heidelberg, 1987.
\newblock ISBN 038713722X.

\bibitem[Feldman and Rojas(2013)]{feldman2013neural}
J.~Feldman and R.~Rojas.
\newblock \emph{Neural Networks: A Systematic Introduction}.
\newblock Springer Berlin Heidelberg, 2013.
\newblock ISBN 9783642610684.

\bibitem[Fukuda et~al.(1991)Fukuda, Saito, Tamura, and
  Tokuyama]{Fukuda1991BoundingTN}
Komei Fukuda, Shigemasa Saito, Akihisa Tamura, and Takeshi Tokuyama.
\newblock Bounding the number of k-faces in arrangements of hyperplanes.
\newblock \emph{Discret. Appl. Math.}, 31:\penalty0 151--165, 1991.

\bibitem[Gao et~al.(2017)Gao, Ries, Simon, and Zilles]{recursiveteach}
Ziyuan Gao, Christoph Ries, Hans~U. Simon, and Sandra Zilles.
\newblock Preference-based teaching.
\newblock \emph{J. Mach. Learn. Res.}, 18\penalty0 (1):\penalty0 1012–1043,
  January 2017.
\newblock ISSN 1532-4435.

\bibitem[Goldman and Kearns(1995)]{goldman1995complexity}
Sally~A Goldman and Michael~J Kearns.
\newblock On the complexity of teaching.
\newblock \emph{Journal of Computer and System Sciences}, 50\penalty0
  (1):\penalty0 20--31, 1995.

\bibitem[Goldman et~al.(1993)Goldman, Rivest, and
  Schapire]{goldman1993learning}
Sally~A Goldman, Ronald~L Rivest, and Robert~E Schapire.
\newblock Learning binary relations and total orders.
\newblock \emph{SIAM Journal on Computing}, 22\penalty0 (5):\penalty0
  1006--1034, 1993.

\bibitem[Gottlieb et~al.(2018)Gottlieb, Kaufman, Kontorovich, and
  Nivasch]{learnconvexpolytope}
Lee-Ad Gottlieb, Eran Kaufman, Aryeh Kontorovich, and Gabriel Nivasch.
\newblock Learning convex polytopes with margin.
\newblock In S.~Bengio, H.~Wallach, H.~Larochelle, K.~Grauman, N.~Cesa-Bianchi,
  and R.~Garnett, editors, \emph{Advances in Neural Information Processing
  Systems 31}, pages 5706--5716. Curran Associates, Inc., 2018.

\bibitem[Guillory and Bilmes(2009)]{guillory2009average}
Andrew Guillory and Jeff Bilmes.
\newblock Average-case active learning with costs.
\newblock In \emph{International conference on algorithmic learning theory},
  pages 141--155. Springer, 2009.

\bibitem[Hanneke and Yang(2015)]{hanneke2015minimax}
Steve Hanneke and Liu Yang.
\newblock Minimax analysis of active learning.
\newblock \emph{The Journal of Machine Learning Research}, 16\penalty0
  (1):\penalty0 3487--3602, 2015.

\bibitem[Haussler et~al.(1994)Haussler, Kearns, and
  Schapire]{haussler1994bounds}
David Haussler, Michael Kearns, and Robert~E Schapire.
\newblock Bounds on the sample complexity of bayesian learning using
  information theory and the vc dimension.
\newblock \emph{Machine learning}, 14\penalty0 (1):\penalty0 83--113, 1994.

\bibitem[Hu et~al.(2017)Hu, Wu, Li, and Wang]{Hu2017QuadraticUB}
Lunjia Hu, Ruihan Wu, T.~Li, and L.~Wang.
\newblock Quadratic upper bound for recursive teaching dimension of finite vc
  classes.
\newblock In \emph{COLT}, 2017.

\bibitem[Jamieson and Nowak(2011)]{jamieson2011active}
Kevin~G Jamieson and Robert Nowak.
\newblock Active ranking using pairwise comparisons.
\newblock In \emph{Advances in Neural Information Processing Systems}, pages
  2240--2248, 2011.

\bibitem[Kane et~al.(2017)Kane, Lovett, Moran, and Zhang]{kane2017active}
Daniel~M Kane, Shachar Lovett, Shay Moran, and Jiapeng Zhang.
\newblock Active classification with comparison queries.
\newblock In \emph{2017 IEEE 58th Annual Symposium on Foundations of Computer
  Science (FOCS)}, pages 355--366. IEEE, 2017.

\bibitem[Khot and Saket(2011)]{KHOT2011129}
Subhash Khot and Rishi Saket.
\newblock On the hardness of learning intersections of two halfspaces.
\newblock \emph{Journal of Computer and System Sciences}, 77\penalty0
  (1):\penalty0 129 -- 141, 2011.
\newblock ISSN 0022-0000.
\newblock \doi{https://doi.org/10.1016/j.jcss.2010.06.010}.
\newblock Celebrating Karp's Kyoto Prize.

\bibitem[Klivans and Servedio(2006)]{attribute}
Adam~R. Klivans and Rocco~A. Servedio.
\newblock Toward attribute efficient learning of decision lists and parities.
\newblock \emph{Journal of Machine Learning Research}, 7, 2006.

\bibitem[Klivans and Sherstov(2006)]{klivanscrypto}
Adam~R. Klivans and Alexander~A. Sherstov.
\newblock Cryptographic hardness for learning intersections of halfspaces.
\newblock In \emph{Proceedings of the 47th Annual IEEE Symposium on Foundations
  of Computer Science}, FOCS '06, page 553–562, USA, 2006. IEEE Computer
  Society.
\newblock ISBN 0769527205.
\newblock \doi{10.1109/FOCS.2006.24}.

\bibitem[Klivans et~al.(2004)Klivans, O'Donnell, and Servedio]{KLIVANS2004808}
Adam~R. Klivans, Ryan O'Donnell, and Rocco~A. Servedio.
\newblock Learning intersections and thresholds of halfspaces.
\newblock \emph{Journal of Computer and System Sciences}, 68\penalty0
  (4):\penalty0 808 -- 840, 2004.
\newblock ISSN 0022-0000.
\newblock \doi{https://doi.org/10.1016/j.jcss.2003.11.002}.
\newblock Special Issue on FOCS 2002.

\bibitem[Kuhlmann(1999)]{kuhlman}
Christian Kuhlmann.
\newblock On teaching and learning intersection-closed concept classes.
\newblock In \emph{Proceedings of the 4th European Conference on Computational
  Learning Theory}, EuroCOLT '99, page 168–182, Berlin, Heidelberg, 1999.
  Springer-Verlag.
\newblock ISBN 3540657010.

\bibitem[Kushilevitz et~al.(1996)Kushilevitz, Linial, Rabinovich, and
  Saks]{Kushilevitz1996WitnessSF}
Eyal Kushilevitz, Nathan Linial, Yuri Rabinovich, and Michael~E. Saks.
\newblock Witness sets for families of binary vectors.
\newblock \emph{J. Comb. Theory, Ser. A}, 73:\penalty0 376--380, 1996.

\bibitem[Kwek and Pitt(1996)]{pacpitt}
Stephen Kwek and Leonard Pitt.
\newblock Pac learning intersections of halfspaces with membership queries
  (extended abstract).
\newblock In \emph{Proceedings of the Ninth Annual Conference on Computational
  Learning Theory}, COLT '96, page 244–254, New York, NY, USA, 1996.
  Association for Computing Machinery.
\newblock ISBN 0897918118.
\newblock \doi{10.1145/238061.238109}.

\bibitem[Lee et~al.(2006)Lee, Servedio, and Wan]{DNFteach}
Homin Lee, Rocco Servedio, and Andrew Wan.
\newblock Dnf are teachable in the average case.
\newblock volume~69, pages 214--228, 09 2006.
\newblock ISBN 978-3-540-35294-5.
\newblock \doi{10.1007/11776420_18}.

\bibitem[Mansouri et~al.(2019)Mansouri, Chen, Vartanian, Zhu, and
  Singla]{mansouri2019preference}
Farnam Mansouri, Yuxin Chen, Ara Vartanian, Jerry Zhu, and Adish Singla.
\newblock Preference-based batch and sequential teaching: Towards a unified
  view of models.
\newblock In \emph{Advances in Neural Information Processing Systems}, pages
  9195--9205, 2019.

\bibitem[McCallumzy and Nigamy(1998)]{mccallumzy1998employing}
Andrew~Kachites McCallumzy and Kamal Nigamy.
\newblock Employing em and pool-based active learning for text classification.
\newblock In \emph{Proc. International Conference on Machine Learning (ICML)},
  pages 359--367. Citeseer, 1998.

\bibitem[Miller et~al.(2007)Miller, Reiner, and Sturmfels]{miller2007geometric}
E.~Miller, V.~Reiner, and B.~Sturmfels.
\newblock \emph{Geometric Combinatorics}.
\newblock IAS/Park City mathematics series. American Mathematical Society,
  2007.
\newblock ISBN 9780821837368.

\bibitem[Nachum and Yehudayoff(2019)]{nachum2019average}
Ido Nachum and Amir Yehudayoff.
\newblock Average-case information complexity of learning.
\newblock In \emph{Algorithmic Learning Theory}, pages 633--646, 2019.

\bibitem[Natarajan(1987)]{natarajan1987learning}
Balaubramaniam~Kausik Natarajan.
\newblock On learning boolean functions.
\newblock In \emph{Proceedings of the nineteenth annual ACM symposium on Theory
  of computing}, pages 296--304, 1987.

\bibitem[Roman(2007)]{roman2007advanced}
S.~Roman.
\newblock \emph{Advanced Linear Algebra}.
\newblock Graduate Texts in Mathematics. Springer New York, 2007.
\newblock ISBN 9780387728315.

\bibitem[Rossen(2003)]{rossen2003discrete}
Kenneth Rossen.
\newblock Discrete mathematics and its applications.
\newblock \emph{McGraw Hill}, 2003.

\bibitem[Simon and Zilles(2015)]{Simon2015OpenPR}
H.~U. Simon and Sandra Zilles.
\newblock Open problem: Recursive teaching dimension versus vc dimension.
\newblock In \emph{COLT}, 2015.

\bibitem[Spielman and Teng(2009)]{smootheb}
Daniel Spielman and Shang-Hua Teng.
\newblock Smoothed analysis: An attempt to explain the behavior of algorithms
  in practice.
\newblock \emph{Commun. ACM}, 52:\penalty0 76--84, 10 2009.
\newblock \doi{10.1145/1562764.1562785}.

\bibitem[Spielman and Teng(2004)]{smoothea}
Daniel~A. Spielman and Shang-Hua Teng.
\newblock Smoothed analysis of algorithms: Why the simplex algorithm usually
  takes polynomial time.
\newblock \emph{J. ACM}, 51\penalty0 (3):\penalty0 385–463, May 2004.
\newblock ISSN 0004-5411.
\newblock \doi{10.1145/990308.990310}.

\bibitem[Traub(2003)]{traub2003information}
Joseph~F Traub.
\newblock Information-based complexity.
\newblock In \emph{Encyclopedia of Computer Science}, pages 850--854. 2003.

\bibitem[Vapnik and Chervonenkis(1971)]{vapnik1971uniform}
VN~Vapnik and A~Ya Chervonenkis.
\newblock On the uniform convergence of relative frequencies of events to their
  probabilities.
\newblock \emph{Theory of Probability \& Its Applications}, 16\penalty0
  (2):\penalty0 264--280, 1971.

\bibitem[Vempala(2010)]{vempala}
Santosh~S. Vempala.
\newblock A random-sampling-based algorithm for learning intersections of
  halfspaces.
\newblock \emph{J. ACM}, 57\penalty0 (6), November 2010.
\newblock ISSN 0004-5411.
\newblock \doi{10.1145/1857914.1857916}.

\bibitem[Vergnas(1980)]{Vergnas1980ConvexityIO}
Michel~Las Vergnas.
\newblock Convexity in oriented matroids.
\newblock \emph{J. Comb. Theory, Ser. B}, 29:\penalty0 231--243, 1980.

\bibitem[Wan(2010)]{wan2010learning}
Andrew Wan.
\newblock \emph{Learning, cryptography, and the average case}.
\newblock Citeseer, 2010.

\bibitem[Zaslavsky(1975)]{trove.nla.gov.au/work/10888269}
Thomas Zaslavsky.
\newblock \emph{Facing up to arrangements : face-count formulas for partitions
  of space by hyperplanes}.
\newblock Providence : American Mathematical Society, 1975.
\newblock ISBN 0821818546.
\newblock "Volume 1, issue 1.".

\bibitem[Zhu et~al.(2018)Zhu, Singla, Zilles, and
  Rafferty]{DBLP:journals/corr/ZhuSingla18}
Xiaojin Zhu, Adish Singla, Sandra Zilles, and Anna~N. Rafferty.
\newblock An overview of machine teaching.
\newblock \emph{CoRR}, abs/1801.05927, 2018.

\bibitem[Zilles et~al.(2008)Zilles, Lange, Holte, and
  Zinkevich]{zilles2008teaching}
Sandra Zilles, Steffen Lange, Robert Holte, and Martin Zinkevich.
\newblock Teaching dimensions based on cooperative learning.
\newblock In \emph{COLT}, pages 135--146, 2008.

\bibitem[Zilles et~al.(2011)Zilles, Lange, Holte, and
  Zinkevich]{zilles2011models}
Sandra Zilles, Steffen Lange, Robert Holte, and Martin Zinkevich.
\newblock Models of cooperative teaching and learning.
\newblock \emph{Journal of Machine Learning Research}, 12\penalty0
  (Feb):\penalty0 349--384, 2011.

\end{thebibliography}

 \iftoggle{longversion}{
  \newpage
 \onecolumn
 \appendix
 {\allowdisplaybreaks
 \section{List of Appendices}\label{appendix:table-of-contents}
In the appendices, we first provide a table summarizing the notations defined in the main paper. We then provide the proofs of our theoretical results in full detail in the subsequent sections. 

The remainder of the appendices are summarized as follows: 
\begin{itemize}
\item \appref{appendix: table of notations} provides a list of notations defined in the main paper
\item \appref{appendix: regions relaxed general position} provides the proof of \thmref{thm:d'-general} (Number of  Regions Induced by Intersections of  Halfspaces)
\item \appref{appendix: faces teaching} provides the proof of \propref{faces_teach} (Number of Faces Induced by Intersections of Halfspaces)
\item \appref{appendix: main theorem} provides the proof of \thmref{theorem: Main Theorem} (Teaching Complexity of Convex Polytopes)
\item \appref{appendix: active learning} provides the proof of \thmref{pairwise teaching bound} (Learning Complexity of Convex Polytopes)
\item \appref{appendix: Dual map} provides the proof of \thmref{dual map} and \corref{cor: teach phi surfaces} (Teaching Complexity of $\phi$-Separable Dichotomy)
\item \appref{appendix: extreme points}  provides the proof of \thmref{main extreme points} (Equivalence of Teaching Set and Extreme Points)
\item \appref{appendix: teaching ranking} provides an additional use-case of the problem of teaching convex polytopes via halfspace queries. In particular, we introduce the problem of teaching linear rankings via halfspaces queries, and establish a $\bigTheta{d}$ bound on the average teaching complexity.
\end{itemize}
\clearpage
\section{Table of Notations Defined in the Main Paper}\label{appendix: table of notations}
For readers' convenience, we summarize the notations used in the main paper in \tabref{tab:notation-table}.
\begin{table}[h!]
  \caption{Table of Notations}
  \label{tab:notation-table}
  \centering
  \begin{tabular}{lll}
    \toprule
    Notations     & Use      \\
    \midrule
    $h,h^{(i)}$ & a hyperplane    \\
    $x, x^{(i)}$ & a point        \\
    $r, r^*$ & target/sampled  region/hypothesis/concept   \\
    $\bracket{\mathbi{u}}, \bracket{\mathbi{v}}$ & dichotomies equivalence classes  \\
    $\eta, \eta_h$ & normal vectors of a hyperplane\\
    $b, b_h$ & bias of a hyperpane\\
    $\xnd$ & data points in $\Rd$ or \data\\
    $\Hnd$ & $n$ hyperplanes set in $\Rd$ or \hypo\\
    $\mathcal{A}(\Hnd)$ & hyperplanes arrangement of set $\Hnd$\\
    $\regionset$ & set of regions induced by hyperplane arrangement $\Hnd$\\
    $\rnd$ & $\#$regions induced by hyperplane arrangement $\mathcal{A}(\Hnd)$\\
    $\boldsymbol{\mathfrak{D}}_{\xnd}$ & set of dichotomies of $\xnd$\\
    $\boldsymbol{\mathfrak{E}}\parenb{\xnd}$ & the set of equivalence classes of homogeneously linear separable dichotomies\\
    $\boldsymbol{\mathfrak{E}}_{\xnd}^{\phi}$ & the set of equivalence classes of $\phi$-separable dichotomies\\
    $r_{\bracket{\mathbi{u}}}$ & random dichotomy (equivalence) class in $\boldsymbol{\mathfrak{E}}_{\xnd}^{\phi}$\\
    $\mathcal{B},\phi,\Upsilon_{\mathrm{dual}} ,\varphi_{\mathrm{dual}}$ &  maps\\ 
    $\ff\parenb{\cA(\Hnd)}$  & number of faces\\
    $\mathcal{U}$ & uniform distribution\\
    $\Theta$& set of embedded points\\
    $\boldlam$& a matrix\\
    $\mathbb{I}_{[k]}$& set of $k$ indices of naturals\\
    \bottomrule
  \end{tabular}
\end{table}
\clearpage

\section{Regions Induced by Intersections of  Halfspaces: Proof of \thmref{thm:d'-general}}\label{appendix: regions relaxed general position}

In this section, we would provide the relevant results, with proofs to complete the claim of \thmref{thm:d'-general}. The struture of the appendix is: we first introduce basic affine geometry, then construct a subspace in which the underlying hyperplane arrangement is structurally similar to the hyperplane arrangement of discussion \tt{i.e.} $\cA\parenb{\Hnd}$, and establish useful properties in relevant lemmas and proposition to complete the proof of \thmref{thm:d'-general}.\vspace{6pt}\\
\noindent
Before we proceed to the technical part of the appendix, we provide elementary discussion on affine geometry \citep{roman2007advanced} below. 
\subsection{Elementary Affine Geometry}\label{appendixsub: affine geometry}
\begin{definition}[Flats \cite{roman2007advanced}]
Let $S$ be a subspace of a vector space $V$. The coset 
\begin{equation*}
    v + S = \condcurlybracket{v + s}{ s \in S}
\end{equation*}
is called a flat in $V$ with base $S$ and flat representative $v$. We also refer to $v+S$ as a translate of $S$. The set $\mathfrak{A}(V)$ of all flats in $V$ is called the affine geometry of $V$. The dimension $\dim(\mathfrak{A}(V))$ of $\mathfrak{A}(V)$ is defined to be $\dim(V)$.
\end{definition}
While a flat may have many flat representatives, it only has one base since $x + S = y + T$ implies that $x \in y +T$ and so $x + S = y + T = x + T$ whence $S = T$.
\begin{definition}[Dimension of flats]
The dimension of a flat $v + S$ is $\dim(S)$. A flat of dimension $k$ is called a $k$-flat. A 0-\tt{flat} is a \tt{point}, a 1-\tt{flat} is a line, and a 2-\tt{flat} is a plane. A flat of dimention $\dim(\mathfrak{A}(V))-1$ is called a hyperplane.  
\end{definition}
In the discussion ahead, we would interchangeably use the notation $\dim$ for a \tt{flat} and a \tt{subspace}. 
With the discussion above, we realize every hyperplane in $\mathbb{R}^k$ has a dual representation as a flat, and a set defined by a \tt{normal} vector and a \tt{bias} (see \secref{sec:formulation}). We would use these representations to our advantage in defining and constructing mathematical objects in the coming discussion.
\subsection{Construction of $\mathbb{N}$ and Relevant Lemmas}\label{appendixsub: construction of N}
For any hyperplane $h \in \Hnd$ in $\Rd$, it can be written as $h \triangleq \eta_h\cdot \h{}+ b_h = 0$ where $\eta_h$ and $b_h$ are a fixed non-zero normal vector and a scalar bias respectively. Consider the subspace $\mathbb{N}$ spanned by the normal vectors of hyperplanes in $\Hnd$.
$$\mathbb{N} = \textbf{span}\tuple{\condcurlybracket{\eta_h}{h \in \Hnd,\: h \coloneqq \eta_h\cdot \h{}+ b_h = 0,\: \h{}\in \mathbb{R}^d}}$$
This construction is interesting pertaining to the arrangement of the hyperplanes which is $d'$-\tt{relaxed general position}. First, we would show some useful properties of the subspace $\mathbb{N}$ and the manner in which $\Hnd$ intersects $\mathbb{N}$ in \lemref{lemma: dim d} and \lemref{lemma: intersect in d-1}.
\begin{lemma}\label{lemma: dim d}
Consider a set $\Hnd$ of $n$ hyperplanes in $\mathbb{R}^d$. If the hyperplane arrangement $\mathcal{A}\parenb{\Hnd}$ is in $d'$-relaxed general position, then $\dim\parenb{\mathbb{N}} = d'$.
\end{lemma}
\begin{proof} Let us define an \tt{ordered} subset $\mathbb{N}_{[d']}$ $\triangleq$ $\curlybracket{\eta_{i_1},\eta_{i_2},\dots,\eta_{i_{d'}}}$  
of normal vectors of any $d'$ hyperplanes in $\Hnd$. Consider the subset $\mathcal{H}_{[d']} \subset \Hnd$ of hyperplanes corresponding to the normal vectors in $\mathbb{N}_{[d']}$. Ideally, if we can show that $\mathbb{N}_{[d']}$ is linearly independent then we have a lower bound on the dimension of $\mathbb{N}$ \tt{i.e.} $\dim\parenb{\mathbb{N}} \ge d'$.\\
\noindent
We construct the matrix $\boldlam_{\mathbb{N}_{[d']}}$ such that $\boldlam_{\mathbb{N}_{[d']}}[k:] = \eta_{i_k}$. Define $\mathbi{b} \triangleq \paren{b_{i_1},b_{i_1},\dots,b_{i_{d'}}}$. Consider the matrix equation for variable $\h{}\in \Rd$:
\begin{align}
    \boldlam_{\mathbb{N}_{[d']}}\h{}= -\mathbi{b}^{\top} \label{eqn:eq301}
\end{align}
But we note that if $\h{}$ is a solution of \eqnref{eqn:eq301} \tt{iff} $\h{}$ exists in $\bigparen{\bigcap_{h \in \mathcal{H}_{[d']}} h}$.
Notice that by the definition of $d'$-\textit{relaxed general position}, $\bigparen{\bigcap_{h \in \mathcal{H}_{[d']}} h}$ is a $(d-d')$-dimensional flat which also forms a solution for \eqnref{eqn:eq301}. Consider a solution ${\h{}}_0 \in \paren{\bigcap_{h \in \mathcal{H}_{[d']}} h}$ such that $\boldlam_{\mathbb{N}_{[d']}}{\h{}}_0 = -\mathbi{b}^{\top}$. Thus,
\begin{align}
    \boldlam_{\mathbb{N}_{[d']}}\h{}= \boldlam_{\mathbb{N}_{[d']}}{\h{}}_0 
    &\implies \boldlam_{\mathbb{N}_{[d']}}\paren{\h{}-{\h{}}_0} = 0 \nonumber \\
    &\implies \dim\paren{\textbf{Ker}\paren{\boldlam_{\mathbb{N}_{[d']}}}} = d-d' \label{eqn:eq302}
\end{align}
But using \thmref{theorem:rank nullity} (rank-nullity, \appref{appendix: Dual map}), \textbf{rank}$\paren{\boldlam_{\mathbb{N}_{[d']}}}$ = $d'$. It implies $\mathbb{N}_{[d']}$ is a set of $d'$ linearly independent vectors. Thus, $\dim(\mathbb{N}) \ge d'$. \\
Note, that $\dim\paren{\mathbb{N}} \ngtr d'$ otherwise $\exists$ an \tt{ordered} subset $\mathbb{N}_{[d'+1]}$ $\triangleq$ $\curlybracket{\eta_{i_1},\eta_{i_2},\dots,\eta_{i_{d'+1}}}$ of $d'+1$ normal vectors corresponding to a subset $\mathcal{H}_{[d'+1]} \subset \Hnd$, which are linearly independent.
Then, the equation $\boldlam_{\mathbb{N}_{[d'+1]}}\h{}= -\parenb{b_{i_1},b_{i_1},\dots,b_{i_{d'+1}}}^{\top}$ has a solution because \textbf{rank}$\paren{\boldlam_{\mathbb{N}_{[d'+1]}}}$ = $d'+1$. This implies that  $\paren{\bigcap_{h \in \mathcal{H}_{[d'+1]}} h} \neq \emptyset$, which contradicts the $d'$-relaxed general position arrangement of $\Hnd$. Thus, $\dim\paren{\mathbb{N}} = d'$.
\end{proof}
Any hyperplane $h \in \Hnd$ is a ($d-1$)-dimensional \textit{flat} which can be written equivalently as $h \equiv h_{flat}\triangleq v_h + S_h$ for some vector $v_h \in \mathbb{R}^d$ and $(d-1)$-dimensional subspace $S_h$. Notice that $\mathbb{N}$ is a $d'$-dimensional \textit{flat} which can be written as $\paren{\textbf{0}+\mathbb{N}}$. Using Theorem 16.5 \citep[see][page 451]{roman2007advanced}, the intersection \textit{flat} $X_h$ = $\paren{h_{flat} \cap \parenb{\textbf{0}+ \mathbb{N}}}$ can be written as $X_h \triangleq y_h + (S_h \cap \mathbb{N})$ for some $y_h \in (h_{flat} \cap \parenb{\textbf{0}+ \mathbb{N})}$. Now, we show a straightforward result that $X_h$ has dimension $d'-1$ which would be useful when we consider the regions induced by the arrangement of \tt{intersection flats} in $\mathbb{N}$.
\begin{lemma}\label{lemma: intersect in d-1}
For the flat $X_h$ constructed as above, $\dim\paren{X_h}$ = $d'-1$.  
\end{lemma}
\begin{proof}
By Theorem 16.6 of \cite{roman2007advanced}, we know that the dimension of the intersection of two subspaces is
$$\dim\paren{S_h \cap \mathbb{N}} = \dim(S_h) + \dim(\mathbb{N}) - \dim(S_h + \mathbb{N})$$
Since $S_h$ is ($d-1$)-dimensional and the orthogonal vector (\tt{i.e.} the normal vector) of $h$ (or $h_{flat}$) exists in $\mathbb{N}$ by definition, the dimension of ($S_h + \mathbb{N}$) = $d$. This implies that
$$\dim(S_h \cap \mathbb{N}) = (d-1) + d' - d = d'-1$$
Since $\dim(X_h) = \dim(S_h \cap \mathbb{N})$, thus the lemma follows.
\end{proof}
\subsection{Construction of Map $\mathcal{B}$ and Proof of \propref{bijection}}\label{appendixsub: bijection}
Now, consider the induced set of hyperplanes in the $d'$-dimensional subspace $\mathbb{N}$:
$$\Hndd =
\condcurlybracket{X_h}{h \in \Hnd}$$
With the construction of the induced set of hyperplanes, we can talk about the regions $\boldsymbol{\mathfrak{R}}(\mathcal{A}(\Hndd))$ induced by the arrangement of $\Hndd$ in the $d'$ dimensional subspace $\mathbb{N}$. We would show that every region induced by the arrangement $\mathcal{A}(\Hnd)$ in $\mathbb{R}^d$ contains a point (vector) from a region induced by  $\mathcal{A}(\Hndd)$ in the subspace $\mathbb{N}$. Before we develop ideas, to show that, we provide the following definition which characterizes points contained in different regions:
\begin{definition}[Path-connectivity of points]
Consider a set of hyperplanes $\mathcal{H}$ in $\Rd$. For any two points $u,v \in \Rd$, we say $u$ and $v$ are path-connected $\mathrm{wrt}$ the regions induced by $\cA\parenb{\mathcal{H}}$ if the following equivalent conditions hold:
\begin{itemize}
    \item if the line segment $\lambda u + (1-\lambda) v$ where $\lambda \in \boldsymbol{(}0,1\boldsymbol{)}$ is not intersected by any hyperplane in $\mathcal{H}$
    \item u and v belong to the same region induced by $\cA\parenb{\mathcal{H}}$
\end{itemize}
\end{definition}
\paragraph{Notations} Denote the \tt{orthogonal projection} of a point $u \in \Rd$ onto $\mathbb{N}$ by $\mathbf{proj}_{\mathbb{N}}\paren{u}$.
Denote a region (polytope) in $\regionset$ by $r$.
Consider a point ${\h{}}_r\in r\setminus\mathbb{N}$. Since $r$ contains an \tt{open convex polyhedron}, for some $\epsilon > 0$ $\exists$ a normed ball $\mathbb{B}_2({\h{}}_r,\epsilon)$  not intersected by any hyperplane.\vspace{6pt}\\
\noindent
To prove our intuition developed earlier, we would show that ${\h{}}_r$ (\tt{if} it exists) and $\mathbf{proj}_{\mathbb{N}}({\h{}}_r)$ are \textit{path-connected}.
\begin{lemma}\label{path connectivity}
Following the notations as above, ${\h{}}_r$ and $\mathbf{proj}_{\mathbb{N}}({\h{}}_r)$ are \textit{path-connected} and, thus every region $r\in \regionset$ has points contained in $\mathbb{N}$.
\end{lemma}
\begin{proof}
For the sake of contraposition, assume that ${\h{}}_r$ and $\mathbf{proj}_{\mathbb{N}}({\h{}}_r)$ are \tt{not} path-connected.
Let $h \triangleq \eta_h\cdot \h{}+ b_h = 0 \in \Hnd$ be the intersecting hyperplane. Assume that $h$ intersects the line segment $\lambda{\h{}}_r+ (1-\lambda)\cdot \mathbf{proj}_{\mathbb{N}}({\h{}}_r)$ at the point ${\h{}}_{h,\cap}$ \tt{i.e.} ${\h{}}_{h,\cap} = \lambda'{\h{}}_r+ (1-\lambda')\cdot\mathbf{proj}_{\mathbb{N}}({\h{}}_r)$ for some $\lambda' \in (0,1)$. By the property of ${\h{}}_r$, we realize ${\h{}}_{h,\cap} \notin \mathbb{B}_2({\h{}}_r,\epsilon)$. 
Since $\mathbf{proj}_{\mathbb{N}}(\cdot)$ is an orthogonal projection, we have
\begin{align}
    \eta_h \perp \parenb{\mathbf{proj}_{\mathbb{N}}({\h{}}_r)-{\h{}}_r} \implies \eta_h\cdot \mathbf{proj}_{\mathbb{N}}({\h{}}_r) = \eta_h\cdot{\h{}}_r \label{eqn:eqproj01}
\end{align}
Using \eqnref{eqn:eqproj01} and noting that ${\h{}}_{h,\cap}$ lies on $h$, we have:
$$\eta_h\cdot{\h{}}_{h,\cap} + b_h = 0 \implies \eta_h\cdot\parenb{\lambda'{\h{}}_r+ (1-\lambda')\cdot \mathbf{proj}_{\mathbb{N}}({\h{}}_r)} + b_h = 0 \implies \eta_h\cdot{\h{}}_r+ b_h = 0$$
But this is a contradiction because $\mathbb{B}_2({\h{}}_r,\epsilon)$, by definition, is \tt{not} intersected by any hyperplane in $\Hnd$. Thus, the lemma follows and this asserts that the subspace $\mathbb{N}$ has at least one point contained in any region induced by $\mathcal{A}\paren{\Hnd}$.
\end{proof}
This gives us the insight that information theoretically, the regions induced on $\mathbb{N}$ by $\mathcal{A}(\Hndd)$ has similar structure to the regions induced on $\mathbb{R}^d$ by $\mathcal{A}(\Hnd)$. We would ascertain this promisingly by showing a \tt{bijective} map from $\boldsymbol{\mathfrak{R}}(\mathcal{A}(\Hndd))$ to $\regionset$. Before we construct the map, we have certain inferences to make based on the previous discussion.\vspace{6pt}\\
\noindent
We observe that every region $\hat{r} \in \boldsymbol{\mathfrak{R}}\parenb{\mathcal{A}(\Hndd)}$ is contained in \tt{exactly one} region in $\mathcal{A}(\Hnd)$ \tt{i.e.} $\hat{r} \subseteq r$ for some $r \in \regionset$. If it is not so then we have two points $a_r^{ind}$, $b_r^{ind}$ $\in \hat{r}$ which are \tt{not} path-connected (in $\mathbb{R}^d$). Thus, there is some hyperplane $h \in \Hnd$ which cuts the line segment at some point $z$. But then $z \in \mathbb{N}$ because $\forall \lambda \in (0,1)$ the combination $\lambda a_r^{ind} + (1-\lambda) b_r^{ind} \in \mathbb{N}$, implying $z \in X_h$. Contradiction because $a_r^{ind}$ and $b_r^{ind}$ are path-connected in $\mathbb{N}$.\vspace{6pt}\\
\noindent
Let us define the map $\mathcal{B}$ as follows:
\begin{align*}
  \mathcal{B}:\:\boldsymbol{\mathfrak{R}}\parenb{\mathcal{A}(\Hndd)} &\longrightarrow \regionset\\
  \hat{r} &\longmapsto \mathbf{region}_{\mathcal{A}(\Hnd)}\paren{\hat{r}}
\end{align*}
where $\mathbf{region}_{\mathcal{A}(\Hnd)}\paren{\hat{r}}$ is the region (polytope) of $\mathcal{A}(\Hnd)$ in which the polytope $\hat{r}$ is contained. Using the observation above, the map is well-defined. Using the observation and \lemref{path connectivity}, we claim in \propref{bijection} that $\mathcal{B}$ is a \tt{bijection}, and thus $\rnd = \fr\parenb{\mathcal{A}(\Hndd)}$.
\\
\begin{proof}[Proof of \propref{bijection}]
Denote by $r^{ind}_i$ and $r^{ind}_j$ two regions in $\boldsymbol{\mathfrak{R}}(\mathcal{A}(\Hndd))$.
First, we show that the map $\mathcal{B}$ is an \tt{injection}. For the sake of contraposition, assume it is not injective. 
Assume that $\mathcal{B}(r^{ind}_i)$ =  $\mathcal{B}(r^{ind}_j)$ = $r$ (a region in  $\mathbb{R}^d$). Note that $r^{ind}_i$ and $r^{ind}_j$ are \textit{not} path connected\footnote{Notion of path-connectivity can be extended for two regions (subsets of points) where no two points in the open convex polyhedrons of the regions are path-connected.} in the subspace $\mathbb{N}$. Thus, $\exists$ a \textit{flat} $X_h$ (intersection of flats $h$ and $\mathbf{0} + \mathbb{N}$) which separates $r^{ind}_i$ and $r^{ind}_j$ in $\mathbb{N}$. Since, $r^{ind}_i$, $r^{ind}_j$ $\subseteq r$, thus $h$ separates $r^{ind}_i$ and $r^{ind}_j$ in $r$, which implies $r^{ind}_i$ and $r^{ind}_j$ are \tt{not} path-connected in $\mathbb{R}^d$. Contradiction! Thus, $\mathcal{B}$ is an injection.

Using \lemref{path connectivity}, we know any region $r\in \regionset$ has points contained in $\mathbb{N}$. The observation above implies that $\exists$ a \textit{unique} $\hat{r} \in \boldsymbol{\mathfrak{R}}(\mathcal{A}(\Hndd))$ such that $\hat{r} \subseteq r$. Thus, $\mathcal{B}$ is a surjection. We have shown that $\mathcal{B}$ is both an injection and a surjection, implying it is a bijection. This also implies that:
\begin{align*}
    \fr\parenb{\mathcal{A}(\Hnd)} = \fr\parenb{\mathcal{A}(\Hndd)}
\end{align*}
\end{proof}
\paragraph{Essentialization and Boolean lattices} The technique to reduce the counting problem of regions to the normal space, as used above, is studied more formally as \tt{essentialization} as discussed in \citet[chap: An introduction to hyperplane arrangement]{miller2007geometric}. One could potentially devise an alternate proof for \propref{bijection} using the technique but it would require introducing several other development on characteristics polynomials, inclusion lattices, m\"{o}bius functions, and inversions among others. On the other hand, our proof technique uses simpler geometric ideas to prove the result from the first principles. Similarly, we could also use \tt{boolean algebra} as discussed in \citet[chap: An introduction to hyperplane arrangement]{book,miller2007geometric} to show the result. Inclusion lattice for a $d'$-relaxed general position hyperplane arrangement could be shown to be isomorphic to a $d'$ \tt{truncated boolean algebra}, and thus one could arrive at a result similar to \propref{bijection}. 

\subsection{Proof of \lemref{key lemma}}\label{appendixsub: key lemma}
Using \propref{bijection} we have a constructively alternate way to ascertain $\rnd$. The previous discussion and results are useful in the sense that we can indeed find $\fr\parenb{\mathcal{A}(\Hndd)}$. As it turns out, $\cA\parenb{\Hndd}$ is in $d'$-relaxed general position arrangement. Since, counting the regions induced by $\cA\parenb{\Hndd}$ on the $d'$-dimensional subspace $\mathbb{N}$ arranged in $d'$-relaxed general position is same as counting the number of regions induced on $\mathbb{R}^{d'}$ by a size $n$ subset of $d'$-general position\footnote{We, interchangeably, use the term $d'$-\tt{general position} or \tt{general position} for $d'$-\tt{relaxed general position} arrangement in $\mathbb{R}^{d'}$} arranged hyperplanes, thus we can directly count $\boldsymbol{\mathfrak{R}}\parenb{\mathcal{A}(\Hndd)}$ using \lemref{d-general} and subsequent \corref{d-exact}. We show in the key \lemref{key lemma} that $\cA\parenb{\Hndd}$ is in $d'$-relaxed general position arrangement.\\

\begin{proof}[Proof of Key \lemref{key lemma}]
Let $1 \le k \le d'$. Consider an arbitrary size $k$ subset $\mathcal{S}_k^{ind} \subseteq \Hndd$ of hyperplanes ($d'-1$-dimensional flats in $\mathbb{N}$). We denote the size $k$ subset of corresponding hyperplanes in $\Rd$ by $\mathcal{S}_k \subseteq \Hnd$  (($d-1$)-dimensional flats). Since $\mathcal{A}(\Hnd)$ is in $d'$-\textit{relaxed general position} we notice that $\dim\paren{\bigcap_{h \in \mathcal{S}_k} h} = d - k$. Define the orthogonal subspace (complement) of $\mathbb{N}$
\begin{align*}
    \mathbb{N}^{\perp} = \condcurlybracket{\h{}\in \mathbb{R}^d}{{\h{}}\cdot v = 0\: \forall\: v \in \mathbb{N}}
\end{align*}
Using Theorem 16.5 \citep[as shown in][]{roman2007advanced} and noting that for any $h' \in \mathcal{S}_k^{ind}$ we can write $h' \equiv h \bigcap \parenb{\mathbf{0} + \mathbb{N}}$ for some $h \in \mathcal{S}_k$, we have:\\
\begin{align}
    \Biggparen{\bigcap_{h' \in \mathcal{S}_k^{ind}} h'} 
    = \paren{\bigcap_{h \in \mathcal{S}_k} h} \bigcap \parenb{\mathbf{0} + \mathbb{N}} \label{eqn:eqi01}
\end{align}
\noindent
Using the representation of \textit{flats}, we can write 
\begin{equation}
\bigcap_{h \in \mathcal{S}_k} h = \nu + W_{\cap}\:\: \textnormal{where}\:\: \nu \in \bigcap_{h \in \mathcal{S}_k} h\:\: \textnormal{and}\:\: W_{\cap} \triangleq \bigcap_{\h{}+ W \in \mathcal{S}_k}W  \label{eqn:eqf01}
\end{equation}
$\mathrm{WLOG}$ we enumerate the hyperplanes in $\mathcal{S}_k$ as $\curlybracket{h^{(1)}, h^{(2)},\dots,h^{(k)} }$. Now, we construct the matrix $\boldlam_k$ using the normal vectors of the hyperplanes in $\mathcal{S}_k$ \tt{i.e.} $\boldlam_{k}[i:] = \eta^{(i)}$ where $h^{(i)} \triangleq \eta^{(i)}\cdot \h{}+ b^{(i)} = 0$ $\forall\: i \in [k]$ ; to solve the system of equations for the intersection of $\mathcal{S}_k$ as follows:
\begin{align}
    \boldlam_k \h{}= -\parenb{b^{(1)},b^{(2)},\dots,b^{(k)}}^{\top} \label{eqn:eqk01}
\end{align}
Since $\bigcap_{h \in \mathcal{S}_k} h \neq \emptyset$, $\exists$ ${\h{}}_0 \in \mathbb{R}^d$ such that $\boldlam_k{\h{}}_0 = -\parenb{b^{(1)},b^{(2)},\dots,b^{(k)}}^{\top}$. But then any solution of $\boldlam_kz = 0$ implies $z-{\h{}}_0$ is a solution of \eqnref{eqn:eqk01}. We can succinctly write this as follows:
\begin{align}
    \boldlam_k\h{}= -\parenb{b^{(1)},b^{(2)},\dots,b^{(k)}}^{\top} \Longleftrightarrow \boldlam_k\h{}= \boldlam_k{\h{}}_0 \Longleftrightarrow \boldlam_k(\h{}-{\h{}}_0) = 0 \label{eqn:eqk02}
\end{align}
This implies that solving $\boldlam_kz = 0$ sufficiently solves \eqnref{eqn:eqk01}. We notice, by definition of $\mathbb{N}^{\perp}$ and construction of $\boldlam_k$,  $ \boldlam_k \perp \mathbb{N}^{\perp}$. Thus, $\mathbb{N}^{\perp}$ is a solution of $\boldlam_kz = 0$. But then, using \eqnref{eqn:eqk02} 
\begin{equation}
    -{\h{}}_0 + \mathbb{N}^{\perp} \subseteq \bigcap_{h \in \mathcal{S}_k} h \label{eqn:eqk03}
\end{equation}
At this point, we observe a small inclusion which would be helpful in claiming the dimension of $\mathcal{S}_k^{ind}$. We notice that $\parenb{-{\h{}}_0 + \mathbb{N}^{\perp}}$ and $\bigcap_{h \in \mathcal{S}_k} h$ are flats in $\Rd$ by definition and \eqnref{eqn:eqf01} respectively. 
Now, combining \eqnref{eqn:eqk03} and Theorem 16.1 \citep[as shown in][]{roman2007advanced}, we get that $\mathbb{N}^{\perp} \subseteq W_{\cap}$.\\
\newline
Finally, we would argue on the dimension of 
$\paren{\bigcap_{h' \in \mathcal{S}_k^{ind}} h'}$ as follows:
\begin{align}
\dim\Biggparen{\bigcap_{h' \in \mathcal{S}_k^{ind}} h'} 
    &= \dim\paren{\paren{\bigcap_{h \in \mathcal{S}_k} h} \bigcap \paren{\mathbf{0} + \mathbb{N}}} \label{eq:eqg01}\\ 
    &= \dim\paren{\bigcap_{h \in \mathcal{S}_k} h} + \dim\parenb{\mathbf{0} + \mathbb{N}} - \dim\parenb{W_{\cap} + \mathbb{N}}  \label{eq:eqg02}\\ 
    &= (d-k) + d' - d \label{eq:eqg03}\\
    &= d'-k  \nonumber
\end{align}
\eqnref{eq:eqg01} is the direct consequence of \eqnref{eqn:eqi01}. \eqnref{eq:eqg02} follows from Theorem 16.6 in \cite{roman2007advanced}.
Since $\mathbb{N}^{\perp} \subseteq W_{\cap}$ and $\mathbb{N}^{\perp}$ is \tt{orthogonal} to $\mathbb{N}$, thus $\dim\parenb{W_{\cap} + \mathbb{N}} = d$ (dimension of the space). Since, $\Hnd$ is in $d'$-relaxed general position and $k \le d'$, $\dim\bigparen{\bigcap_{h \in \mathcal{S}_k} h} = d-k$. These observations yield \eqnref{eq:eqg03}. Thus, for any arbitrary subset $\mathcal{S}_k^{ind}$ of size $1 \le k \le d'$, we have shown that $\dim\bigparen{\bigcap_{h' \in \mathcal{S}_k^{ind}} h'} = d'-k$.

Notice that if we select a subset of $\Hndd$ of size more than $d'$, then they don't intersect at any point since the corresponding subset of hyperplanes in $\Hnd$ has \tt{empty} intersection.\vspace{6pt}\\
\noindent
Thus, following \defref{defn relaxed general position}, we show that $\Hndd$ is in $d'$-relaxed general position. Hence, the lemma follows.
\end{proof}
\subsection{Proof of \thmref{thm:d'-general}}\label{appendixsub: proof of thm1}
We note that a subspace of dimension $k$ of $\Rd$ is \tt{isomorphic} to $\mathbb{R}^k$. Thus, $d'$-\tt{relaxed general position} hyperplane arrangement $\cA(\Hndd)$ in $\mathbb{N}$ can be \tt{uniquely} mapped to a $d'$-\tt{relaxed general position} hyperplane arrangement of $n$ hyperplanes in $\mathbb{R}^{d'}$. It implies that we can use \lemref{d-general} (discussed and proved in \appref{appendixsub: bounds on number of regions}, provides an exact form for the number of regions induced in $\Rd$ when the hyperplane arrangement is in general position) to ascertain $\fr(\mathcal{A}(\Hndd))$ since $\cA(\Hndd)$ satisfies all the required premises \tt{i.e.} $d'$-general position in $d'$ dimensional Euclidean space. Thus, we have
\begin{align*}
    \fr(\mathcal{A}(\Hndd)) = Q(n, d') = \sum_{i=0}^{d'} \binom{n}{i}
\end{align*}
Using \propref{bijection}, we finally show that:
\begin{align*}
    \fr\parenb{\mathcal{A}(\Hnd)} = \fr\parenb{\mathcal{A}(\Hndd)} = \sum_{i=0}^{d'} \binom{n}{i}
\end{align*}
This completes the proof of \thmref{thm:d'-general}.

\paragraph{Remark}
\label{appsubsec: discussion on extension}
One can study the arrangement of hyperplanes $\cA(\Hnd)$ using the characteristic polynomials as discussed in (An introduction to hyperplane arrangments), \cite{miller2007geometric}.  \citet{trove.nla.gov.au/work/10888269} connected the computation of the number of regions in an arrangment to the corresponding characteristic polynomials. But it can be extremely tricky to find exact (simple) forms for those polynomials even for rather straight-forward arragements. \citet{Fukuda1991BoundingTN} explicitly mentioned via citing the work of
\citet{Vergnas1980ConvexityIO} and \citet{trove.nla.gov.au/work/10888269} that computing the number of regions for arbitrary hyperplane arrangement is non-trivial as it depends on the underlying matroid structure. In our work, we are able to establish an exact form for a non-simple setting. The geometric ideas to understand the subspaces spanned by the normals (aka essentialization) corresponding to the hyperplanes can be further leveraged to establish exact forms or average teaching results for more general arragements than relaxed general position. One possible study could be to understand the induced regions in terms of faces for which intersection of hyperplanes on a given hyperplane could be studied. Our idea of path-connectivity could be a potential direction to find out simple forms for the characteristic polynomials corresponding to more relaxed arrangements.

\clearpage
\section{Faces Induced by Intersections of Halfspaces: Proof of \propref{faces_teach}}\label{appendix: faces teaching}
In this section, we provide the proof of \propref{faces_teach} for the number of faces induced by the hyperplane arrangement $\cA(\Hnd)$. 
\newline
\begin{proof}[Proof of \propref{faces_teach}]
To count the number of faces induced by the arrangement $\cA\parenb{\Hnd}$ on the hyperplanes, one way it can be ascertained is by counting the number of regions/faces induced on any hyperplane. If we fix any hyperplane $h^* \in \Hnd$ and look at the intersections of $h^*$ with $\Hnd\setminus \{h^*\}$, we can count the number of regions formed on $h^*$.

If $d' = 1$, then $\ff\parenb{\cA(\Hnd)} = n$ since all the hyperplanes are parallel to each other. Thus, we  assume that $d' > 1$ for further discussion.

Since $h^*$ can be interpreted as a \tt{flat}, we can write $h^* \equiv v^* + W^*$ for some vector $v^* \in \mathbb{R}^d$ and ($d-1$)-dimensional subspace $W^*$ of $\Rd$. By \defref{defn relaxed general position}, $\parenb{h^* \cap h}$ is a ($d-2$)-dimensional flat $\forall h \in \Hnd\setminus\{h^*\}$. Thus, we define by $\boldsymbol{\mathcal{H}'}_{n-1,d-1}  \triangleq  \condcurlybracket{\parenb{h^*\cap \ell}}{\ell \in \Hnd\setminus \{h^*\}}$ the induced set of $n-1$ flats (intersections) on $h^*$ (which is a $(d-1)$-dimensional flat). We note that for any $1 \le k \le d'-1$, if $\mathrm{T}_k \subset \boldsymbol{\mathcal{H}'}_{n-1,d-1}$ then 
\[\dim\paren{\bigcap_{\ell \in \mathrm{T}_k}\ell} = (d-1)-k \]
It holds because if $\dim\bigparen{\bigcap_{\ell \in \mathrm{T}_k}\ell} \neq (d-1)-k$ then $\dim\bigparen{\bigparen{\bigcap_{\ell \in \mathrm{T}_k}\ell} \bigcap h^*} \neq d-(k+1)$ since $\bigparen{\bigcap_{\ell \in \mathrm{T}_k}\ell} \subset h^*$. This violates $d'$-relaxed general position arrangement of $\Hnd$. Thus, $\boldsymbol{\mathcal{H}'}_{n-1,d-1}$ is in $(d'-1)$-relaxed general position arrangement. Since counting the number of regions induced on $h^*$ by $\boldsymbol{\mathcal{H}'}_{n-1,d-1}$ is the same as ascertaining $\fr\parenb{\cA(\Hndo)}$ \tt{i.e.} $(n-1)$ hyperplanes in $\mathbb{R}^{d-1}$ in $(d'-1)$-relaxed general position, using \thmref{thm:d'-general} we get: 
\[\fr\parenb{\cA(\boldsymbol{\mathcal{H}'}_{n-1,d-1})} = \fr\parenb{\cA(\Hndo)} = \sum_{i=0}^{d'-1} \binom{n-1}{i} \]
Since, there are $n$ hyperplanes thus the proposition follows,
\[\ff\parenb{\cA(\Hnd)}  = n\cdot \sum_{i=0}^{d'-1}\binom{n-1}{i}\] 
which completes the proof.
\end{proof}
\clearpage
\section{Teaching Complexity of Convex Polytopes: Proof of \thmref{theorem: Main Theorem}}\label{appendix: main theorem}
In this section, we provide the proof of the main \thmref{theorem: Main Theorem}. It is divided in three subsections: (i) worst-case of teaching complexity of convex polytopes of $\bigTheta{n}$ as part of \thmref{theorem: Main Theorem} in \secref{appendixsub: worst case},  (ii) bounds on $\rnd$ via proof of \corref{cor: bounds regions} in \secref{appendixsub:uplowregion} and (iii) proof of average-teaching complexity of Main \thmref{theorem: Main Theorem} in \secref{appendicsub: main theorem}.

\subsection{Worst-case Complexity for Teaching: $\bigTheta{n}$ }\label{appendixsub: worst case}
We would show the lower bound on the \tt{worst-case} of $\bigOmega{n}$ and notice that upper bound is trivial.
\newline
Consider $n$-dimensional hypersphere $\mathbb{S}$ in $\Rd$ and $\mathbb{S}_{pos}$ the restriction in the positive quadrant \tt{i.e.} all coordinates are positive. 
\vspace{6pt}\\
\noindent
To give an intuition of the worst-case scenario, we start with $\mathbb{R}^2$. Consider the unit circle $x^2+y^2 = 1$ restricted in the positive quadrant. We randomly drop $n$ points on the arc and draw tangents to them. Notice that \tt{no} three tangents can intersect at a point. Moreover, since all the tangents lie in a single quadrant, they can't be parallel. Thus, any two have a \tt{non-empty} intersection. It implies the $n$ hyperplanes thus constructed are in 2-\tt{relaxed general position}. Notice that the arc forms a convex connected set with all the hyperplanes sharing a point. Thus, arrangement of the tangents induces a region which has $n$ many sides or faces.
\vspace{6pt}\\
\noindent
We use the similar idea to construct $n$ hyperplanes in $\Rd$. Let us consider $\mathbb{S}_{pos}$ the restriction of unit hypersphere in $\Rd$. Now, drop $n$ points on the restriction in such a way that any $d$ are \tt{linearly independent}. Denote the $n$ points as $\curlybracket{x^{(1)}, x^{(2)}, \dots, x^{(n)}}$.
Now, consider the matrix $\boldlam$ defined by $x^{(i)}$ as row for each $i \in \bracket{n}$. Thus, for $k\in[d]$, any $k$ rows are \tt{linearly independent}.
Consider the hyperplanes defined by the $n$ points. Notice that the \tt{bias} is same for all the hyperplanes. Denote the hyperplanes by $\curlybracket{h^{(1)}, h^{(2)}, \dots, h^{(n)}}$. It is easy to see that we can equivalently write $h^{(i)} \equiv x^{(n)}\cdot y + 1$ for variable $y \in \Rd$ $\forall i \in \bracket{n}$. Let us define for $k\in\bracket{d}$ $\mathbb{I}_{[k]}$ $\triangleq$ $\{i_1,i_2,...,i_k\}$ as $k$ indices for rows. Denote by $\boldlam_{\mathbb{I}_{[k]}}$ = $\boldlam_{[\mathbb{I}_{[k]}\times d]}$ (rows of $\mathbf{\boldlam}$ corresponding to $\mathbb{I}_{[k]}$)
If we consider the linear system equation
\begin{align}
    \boldlam_{\mathbb{I}_{[k]}}\cdot y = \mathbf{1}^k \label{eqn:eqworst}
\end{align}
Notice that $\rank{\boldlam_{\mathbb{I}_{[k]}}} = k$ because row rank is $k$. Thus, \eqnref{eqn:eqworst} has a solution, call it $y_0$.\\
Using rank-nullity (\thmref{theorem:rank nullity}), we realize that $\dim\parenb{\condcurlybracket{y}{\boldlam_{\mathbb{I}_{[k]}}\cdot (y-y_0) = \mathbf{0}}}$ is $k$. Define a matrix $\boldlam_h$ with each row as $(x^{(i)},1) \forall i \in[n]$.
Now, if rewrite \eqnref{eqn:eqworst} as :
\begin{align}
    \boldlam_{\mathbb{I}_{[k]}}\cdot y = \mathbf{1}^k \Leftrightarrow (\boldlam_h)_{\mathbb{I}_{[k]}}\cdot\binom{y}{1} = 0 \label{eqn:eqworst1}
\end{align}
\eqnref{eqn:eqworst1} implies that $\dim\parenb{\condcurlybracket{y}{\boldlam_{\mathbb{I}_{[k]}}\cdot (y-y_0) = \mathbf{0}}}$ =  $\dim\parenb{\condcurlybracket{y}{(\boldlam_h)_{\mathbb{I}_{[k]}}\cdot\binom{y}{1} = 0}}$ = $d-k$. But solving \eqnref{eqn:eqworst1} is same as finding an intersection point of the hyperplane corresponding to rows $\mathbb{I}_{[k]}$ in $\boldlam_h$. Thus, we show that for any $k\in[n]$ subset of hyperplanes in $\curlybracket{h^{(1)}, h^{(2)}, \dots, h^{(n)}}$, they intersect in a $(d-k)$-dimensional plane. Thus, these hyperplanes are in $d$-relaxed general position. Since, $\mathbb{S}_{pos}$ is contained in exactly one halfspace of every hyperplane touching it implies it is contained in one region induced by the hyperplanes arrangement. Since all the hyperplanes share one point in that region, thus we show that there is one region with $n$ faces for arbitrary $d$-dimensional Euclidean space. This implies, the \tt{worst-case} of teaching complexity of convex polytopes is $\bigTheta{n}$.

This completes the second part of \thmref{theorem: Main Theorem}.
\subsection{ Upper and Lower Bound on number of regions}\label{appendixsub:uplowregion}
In this subsection, we establish bounds on $\rnd$ as \corref{cor: bounds regions}.
\\
\begin{proof}[Proof of \corref{cor: bounds regions}]\\
We'll prove the corollary in two parts -- by establishing the upper and lower bounds on $\rnd$.\vspace{6pt}\\
\noindent
The proof for the upper bound is based on a geometric series argument and uses the definition of a binomial term. First note that, using \thmref{thm:d'-general}, we have:
 \begin{align*}
     \fr\parenb{\mathcal{A}(\Hnd)} = \sum_{i=0}^{d'} \binom{n}{i}
 \end{align*}
Now, we observe the following computation:
 \begin{align*}
     \frac{\sum_{i=0}^{d'} \binom{n}{i}}{\binom{n}{d'}} &=\sum^{0}_{i=d'}\frac{\binom{n}{i}}{\binom{n}{d'}}\\
     &= 1 + \frac{d'}{(n-d'+1)}+ \frac{d'(d'-1)}{(n-d'+1)(n-d'+2)}+\cdots+\frac{d'!}{(n-d'+1)\cdots(n-d'+d')}\\
     &\le 1 + \frac{d'}{(n-d'+1)}+ \frac{(d')^2}{(n-d'+1)^2}+\cdots +\frac{(d')^{d'}}{(n-d'+1)^{d'}}\\
     &\le \sum_{i = 0}^{\infty}\left(\frac{d'}{(n-d'+1)}\right)^i\\
     &= \frac{1}{1-\frac{d'}{n-d'+1}} = \frac{n-d'+1}{n-2d'+1}
\end{align*}
The last inequality establishes the upper bound in the corollary.
\vspace{6pt}\\
For the lower bound we note that:
\begin{equation}
   \rnd =  \sum_{i=0}^{d'} \binom{n}{i} \ge \binom{n-1}{d'} \nonumber 
\end{equation}
Hence, the corollary is proven.
\end{proof}
\subsection{ Proof of \thmref{theorem: Main Theorem}}\label{appendicsub: main theorem}
In the subsection \ref{appendixsub:uplowregion}, we proved the key corollary to show tight bounds on $\rnd$. We use \corref{cor: bounds regions} to show the stated bounds on \ref{eqn:tag:a.1}-- upper bound in \lemref{theorem: upper Main Theorem} and lower bound in \lemref{theorem: Lower main theorem}.
We combine \lemref{theorem: upper Main Theorem} and \lemref{theorem: Lower main theorem} to prove the Main \thmref{theorem: Main Theorem}.

To simplify the notations, we use $Q(n,d)$ (discussed in details in \appref{appendixsub: bounds on number of regions}) to denote the number of regions induced by $n$ hyperplanes in $\Rd$ arranged in general position (\tt{cf} \defref{defn general position}). We note that, in the case of $d'$-\tt{relaxed general position} arrangement, $\rnd = Q(n, d')$ and $\ff\parenb{\cA(\Hnd)} = n\cdot Q(n-1, d'-1)$. This follows from the recursion on $Q(\cdot, \cdot)$ \tt{i.e.} $Q(n,d) = Q(n-1, d) + Q(n-1,d-1)$ (for $n > d$), as discussed in \lemref{d-general} and the subsequent exact form in \corref{d-exact} (in \appref{appendixsub: bounds on number of regions}).
We rewrite $\rnd$ and $\ff\parenb{\cA(\Hnd)}$ in terms of $Q(\cdot, \cdot)$ so that any bound on $Q(\cdot, \cdot)$ would help us in bounding $\ff\parenb{\cA(\Hnd)}/\rnd$. We leverage tight bounds (upper and lower) on the ratio $Q(n-1, d')/Q(n-1, d'-1)$ to achieve the results in the main theorem. We would formally state the two lemmas and provide their proofs before we complete the proof of the main theorem of the section.

\begin{lemma}[Upper bound]\label{theorem: upper Main Theorem}
 Assume $\Hnd$ is in $d'$-\tt{relaxed general position}. Assume $\rgn$ $\sim$ $\mathcal{U}$. Let the random variable $M_n$ denote the number of halfspace queries that are requested in the teaching \algoref{algo: teaching halfspace}, then 
\[\expctover{\cU}{M_n} = \cO(d')\]
\tt{i.e.} the average teaching complexity of convex polytopes is upper bounded by $\cO(d')$.
\end{lemma}
\begin{proof}
Since the \tt{target hypotheses} are sampled uniformly at random, each hypothesis is enclosed by $\ff(\cA(\Hnd))\big/\fr(\cA(\Hnd)$ hyperplanes on average.

We first provide an upper bound on the average teaching complexity and using similar technique show a lower bound.\\
Combining \thmref{thm:d'-general}, \lemref{d-general}, upper bound in \corref{cor: bounds regions}, and \propref{faces_teach}, we prove the lemma in two cases:\\
\newline
\noindent
\underline{\textbf{Case} 1:} $n > 2d'$ ($n$ is sufficiently large)\\
\begin{align}
 \frac{\ff\parenb{\cA(\Hnd)}}{\fr\parenb{\cA(\Hnd)}}  
 &= \frac{n\cdot Q(n-1, d'-1)}{Q(n,d')} \label{eqn:eq01}\\[2pt] 
 &= \frac{n\cdot Q(n-1,d'-1)}{Q(n-1,d') + Q(n-1,d'-1)} \label{eqn:eq02}\\[2pt] 
 &= n \cdot \Biggparen{1\bigg/\biggparen{\frac{Q(n-1,d')}{Q(n-1,d'-1)} + 1}} \nonumber \\[2pt] 
 &\le n \cdot \Biggparen{1\bigg/\biggparen{\frac{n-1}{2d'} + 1}} \label{eqn:eq03}\\[2pt] 
 &=  2d'\cdot \Biggparen{1 \bigg/ \biggparen{1+\frac{2d'-1}{n}}}\nonumber\\
 &\le 2d'  \label{eqn:eq04}.\: 
\end{align}

\eqnref{eqn:eq01} follows using \thmref{thm:d'-general} and \propref{faces_teach}, \eqnref{eqn:eq02} is based on the recursion mentioned in \lemref{d-general}, \eqnref{eqn:eq03} is bounded using \lemref{ratio bound} and in \eqnref{eqn:eq04}, we observe that $0 < \frac{2d'-1}{n}$.\\
\newline
\noindent
\underline{\textbf{Case} 2:}\: $n \le 2d'$ $\implies$ $n = \mathcal{O}\parenb{d'}$. This trivially gives $\mathcal{O}(d')$ as each \tt{target hypothesis} is enclosed by at the most $n$ hyperplanes.

Thus, in the two cases we have shown that the average teaching complexity of the algoithm is upper bounded by $\mathcal{O}(d')$.
\end{proof}

\begin{lemma}[Lower bound]\label{theorem: Lower main theorem}
Assume $\Hnd$ is in $d'$-\tt{relaxed general position}, and $r$ $\sim$ $\mathcal{U}$. Let the random variable $M_n$ denote the number of halfspace queries that are requested in the teaching \algoref{algo: teaching halfspace}, then 
\begin{equation*}
    \expctover{\cU}{M_n} = \bigOmega{d'}
\end{equation*}
\tt{i.e.} the average teaching complexity of convex polytopes is lower bounded by $\bigOmega{d'}$. 
\end{lemma}
 \begin{proof} Following similar steps as \lemref{theorem: upper Main Theorem}; for sufficiently large $n > d$ we get:
\begin{align}
    \frac{\ff\parenb{\cA(\Hnd)}}{\fr\parenb{\cA(\Hnd)}}  = \frac{n\cdot Q(n-1, d'-1)}{Q(n,d')}
 &= \frac{n\cdot Q(n-1,d'-1)}{Q(n-1,d') + Q(n-1,d'-1)} \label{eqn:eqlow01}\\[2pt] 
 &= n \cdot \Biggparen{1\bigg/\biggparen{\frac{Q(n-1,d')}{Q(n-1,d'-1)} + 1}} \nonumber \\[2pt] 
  &= n \cdot \Biggparen{1\bigg/\biggparen{\frac{\binom{n-1}{d'}}{Q(n-1,d'-1)} + 2}} \label{eqn:eqlow02} \\[2pt] 
 &\ge n \cdot \Biggparen{1\bigg/\biggparen{\frac{n-1}{d'} + 2}} \label{eqn:eqlow03} \\[2pt] 
 &= d' \cdot \Biggparen{1\bigg/\biggparen{\frac{n-1}{n} + \frac{2d'}{n}}} \nonumber \\[2pt] 
 &\ge \frac{d'}{1+2} \label{eqn:eqlow04}
\end{align}
\eqnref{eqn:eqlow01} follows using \thmref{thm:d'-general}, \lemref{d-general}, and \propref{faces_teach}. \eqnref{eqn:eqlow02} is a direct consequence of \corref{d-exact}. By carefully noting the lower bound in \corref{cor: bounds regions}, we get the bound in \eqnref{eqn:eqlow03}. We observe that $\frac{n-1}{n} + \frac{2d'}{n} < 1 + 2$. Thus for sufficiently large $n > d$, we show that the average teaching complexity of intersection of halfspaces is lower bounded by $\bigOmega{d'}$.
\end{proof}
\begin{proof}[Proof of \thmref{theorem: Main Theorem}] Using \lemref{theorem: upper Main Theorem} and \lemref{theorem: Lower main theorem}, it is straightforward that  $\expctover{\cU}{M_n} = \bigTheta{d'}$.
\end{proof}

\clearpage
 \section{Learning Complexity of Convex Polytopes: Proof of \thmref{pairwise teaching bound}}\label{appendix: active learning}

In this section, we would discuss the problem of active learning of convex polytopes induced by the hyperplanes arrangement in $\Rd$. We would provide some relevant results on the counting of the number of regions induced by the arrangement of $n$ hyperplanes in $\Rd$ in \tt{general position} (\defref{defn general position}). We would provide a procedure (shown in \algoref{algo: query selection}) which actively and sequentially learns a uniformly randomly sampled region. We show that the average query(sample) complexity for the algorithm is $\bigTheta{d'\log n}$. We would provide the proof of \thmref{pairwise teaching bound} when the hyperplane arrangement is in \tt{general position} (\defref{defn general position}) and then show the extension to the case of $d'$-\tt{relaxed general position} arrangement.

First we would start with some illustration of the \defref{defn general position} and see how it is an special case of \defref{defn relaxed general position}.
To illustrate and understand the definition, we can take a look at euclidean spaces $\mathbb{R}^2$ and $\mathbb{R}^3$. For $\mathbb{R}^2$, consider three lines denoted by $l_1, l_2$ \text{and} $l_3$ (hyperplanes). Note, $k$ can take two values. For $k = 1$, the given line $l_i$ intersects in a line which is vacuously true. For any two lines, they need to intersect in a point. For the three lines, they have an empty intersection. For $\mathbb{R}^3$, consider four planes denoted by $P_1, P_2, P_3$ and $P_4$. We can understand the definition from \tabref{general-position}.

\begin{table}[h!]
  \caption{General position of planes in $\mathbb{R}^3$}
  \label{general-position}
  \centering
  \begin{tabular}{lll}
    \toprule
    $k$     & Intersection      \\
    \midrule
    1 & A plane, $\mathbb{R}^2$     \\
    2 & A line, $\mathbb{R}$        \\
    3 & A point   \\
    4 & Null  \\
    \bottomrule
  \end{tabular}
\end{table}

We notice that \defref{defn general position} is a \tt{special} case of \defref{defn relaxed general position}. If we fix, say $k = 2$ and assume that for intersections of planes upto $k$ follow \tabref{general-position} but if any subset of hyperplanes of size more than $k$, they intersect only in null \tt{i.e.} if we pick three planes then they don't intersect in a common point. This would rightly give an example of an arrangement in $d'$-\tt{relaxed general position} for $d' = 2$. We illustrate this arrangement in \figref{fig:example.2general}. If $k = 1$, then that would give 1-\tt{relaxed general position} as illustrated in \figref{fig:example.1general} which accounts for case when hyperplanes are parallel to each other. In the case of $k = 3$, we get $3$-\tt{relaxed general position} (\figref{fig:example.3general}) which is also the case of \tt{general position} (\defref{defn general position}) arrangement. Relaxed general position is a natural extension to general position. It takes into account arrangements which can’t be structurally explained by general position setting in higher dimension as discussed above. From a learning point of view, data is usually embedded sparsely in spaces with much higher dimension than the information they contain. There has been extended study on learning the sparse representation using component analyses. Interestingly, relaxed notion of general position captures the essence of arrangements where hyperplanes could be sparsely embedded in high dimensional space but are in general position in a much smaller subspace. We interchangeably use $d'$-\tt{general position} or \tt{general position} when $d' = d$ if the hyperplane arrangement is in $d'$-\tt{relaxed general position}.\vspace{6pt}\\
\noindent
We are interested in the notion of \tt{general position} of hyperplanes for a variety of reasons. First, we show an existing duality (see \secref{sec.applied}) between a problem instance of finding the number of $\phi$-separable dichotomies (primal space) \citep{cover1965geometrical} to a problem instance of teaching intersection of halfspaces (dual space). This duality would be achieved when the points in primal space and hyperplanes in
dual space are in general position of points (see \defref{defn general position}, \defref{defn relaxed general position}) and general position of hyperplanes (see \defref{general position vectors}, \defref{phi general}) respectively. Second, \citet[chap: An Introduction to Hyperplane Arrangements]{miller2007geometric} mentions an exact form for the number of regions induced by the  \tt{general position} arrangement of hyperplanes $\Hnd$. This key result would be used in our significant contributions (see \secref{sec:avgtd}): \thmref{thm:d'-general} and \propref{faces_teach}, where we would try to reduce from  the case of $d'$-\tt{relaxed general position} to a case of \tt{general position}.\vspace{6pt}\\
\noindent
To prove \thmref{pairwise teaching bound}, we would show some relevant results in the following subsection:
\subsection{Bounds on Number of Regions Induced by General Position Arrangement}\label{appendixsub: bounds on number of regions}
Consider a set of $n$ hyperplanes in $\mathbb{R}^d$, denoted by $\Hnd$, and the underlying arrangement $\mathcal{A}(\Hnd)$ is in general position (\defref{defn general position}). Denote by $Q(n,d)$ the number of regions induced by $\mathcal{A}(\Hnd)$.
Although \citet{miller2007geometric} provides an exact form for $Q(n,d)$, we would provide a recursion similar to \cite{jamieson2011active} with a proof for continuity and flow of ideas.
\begin{lemma}[Regions induced by general-position hyperplane arrangement]\label{d-general}
Let $Q(n,d)$ denote the number of $d$-cells or regions induced by the general position hyperplane arrangement. $Q(n,d)$ satisfies the recursion:
\begin{equation}
    Q(n,d) = Q(n-1, d) + Q(n-1, d-1) \label{eqn:eqnstan}
\end{equation}
where $Q(1,d) = 2$ and $Q(n,0) = 1$.
\end{lemma}
\begin{proof}
The proof is based on a recursive argument on how hyperplanes are added to the $d$-dimensional space. Consider an arbitrary ordering on the hyperplanes. Denote the last hyperplane added by $h^{(n)}$. We observe that the number of new regions induced by 
$h^{(n)}$ to $\cA\parenb{\Hnd\setminus \{h^{(n)}\}}$ is equal to the number of regions/faces induced on $h^{(n)}$ by the intersections of $\parenb{\Hnd\setminus \{h^{(n)}\}}$ on it. Since, the hyperplanes are in general position, thus all the other $(n-1)$ hyperplanes intersect $h^{(n)}$ on $(d-2)$-plane. Thus, we have $(n-1)$ of $(d-2)$- dimensional hyperplanes\footnote{Proof follows similar steps as in \propref{faces_teach}.} arranged on a $(d-1)$-plane.
Denote this induced set of hyperplanes by $\boldsymbol{\hat{\mathcal{H}}}_{n-1,d-1}$, which can be defined as $\boldsymbol{\hat{\mathcal{H}}}_{n-1,d-1}  \triangleq  \condcurlybracket{\parenb{h^{(n)}\cap \ell}}{\ell \in \parenb{\Hnd\setminus \{h^{(n)}\}}}$ the induced set of $n-1$ flats (intersections) on $h^{(n)}$. We note that for any $1 \le k \le d-1$, if $\mathrm{T}_k \subset \boldsymbol{\hat{\mathcal{H}}}_{n-1,d-1}$ then 
\[\dim\paren{\bigcap_{\ell \in \mathrm{T}_k}\ell} = (d-1)-k \]
It holds because if $\dim\bigparen{\bigcap_{\ell \in \mathrm{T}_k}\ell} \neq (d-1)-k$ then $\dim\bigparen{\bigparen{\bigcap_{\ell \in \mathrm{T}_k}\ell} \bigcap h^{(n)}} \neq d-(k+1)$ since $\bigparen{\bigcap_{\ell \in \mathrm{T}_k}\ell} \subset h^{(n)}$. This violates the general position arrangement of $\Hnd$. Thus, $\boldsymbol{\hat{\mathcal{H}}}_{n-1,d-1}$ is in general position arrangement. But by definition, number of faces induced on $h^{(n)}$ by $\boldsymbol{\hat{\mathcal{H}}}_{n-1,d-1}$ is $Q(n-1,d-1)$.

Hence, the total number of regions in the $d$-dimensional space is $Q(n-1, d) + Q(n-1, d-1)$. Thus, the lemma follows.
\end{proof}
$Q(\cdot)$ as defined above has the following exact form:

\begin{corollary}[An introduction to hyperplane arrangement \cite{miller2007geometric}] \label{d-exact} The recusion in \lemref{d-general} has the form: 

 \begin{equation*}
    Q(n,d) = \sum_{i=0}^{d} \binom{n}{i}
 \end{equation*}
 for $n > d$. If $n \le d$, then $Q(n,d) = 2^n$.
\end{corollary}
We prove a simple corollary which claims an asymptotic bound on $Q(\cdot)$ that would be used in a number of results:
\begin{corollary}\label{d-exact lower bound} For sufficiently large $n > d$, there exist positive real number $k_1$ such that:
\begin{equation*}
     k_1\frac{n^d}{d!} < Q(n,d)
 \end{equation*}
\end{corollary}
\begin{proof}
Using \corref{d-exact}, we can write:
\begin{align*}
 Q(n,d)  &= \sum_{i=0}^{d} \binom{n}{i}\\[2pt] 
 &= \sum_{i=0}^{d} \Theta \left(\frac{n^i}{i!}\right) \text{(for sufficiently large $n$ each term is bounded by above and below)}\\[2pt] 
 &> k_1 \frac{n^d}{d!} \text{$\:$(by definition, $\exists$ $k_1 > 0$, $N_0$ such that $\forall$ $n > N_0$ condition holds )}\\[2pt]
\end{align*}
Specifically, we can show that for $n \ge d^2$, the condition holds. This is true because there exists a constant $c$ such that $c.\prod_{i=0}^{d-1}(n-i) > n^d$.
\end{proof}

\subsection{Average-case Analysis of Active Learning Complexity}\label{appendixsub: average-case analysis} 
In subsection \secref{sec:avgtd}, we introduced the problem of teaching convex polytopes via halfspace queries for a set of hyperplanes $\Hnd$ in $\Rd$ arranged in $d'$-relaxed general position. In \thmref{theorem: Main Theorem}, we showed that the teaching complexity for the arrangement is $\bigTheta{d'}$. Now, we would discuss the problem of active learning of convex polytopes induced by $\cA\parenb{\Hnd}$, via halfspace queries. Using motivations from \cite{jamieson2011active} in which they explore the problem of ranking, we provide \algoref{algo: query selection} to actively learn the enclosing region for a randomly sampled \tt{target region} via adaptive and sequential selection of halfspaces queries for a hyperplane. We analyze the problem in the framework of the average-case analysis as motivated in \cite{traub2003information} and section 1.1 of \cite{jamieson2011active}. We achieve $\bigTheta{d'\cdot \log n}$ average label complexity for active learning through our \algoref{algo: query selection}. The lower bound is straight forward using \corref{d-exact lower bound}. We need at least $|\log_2 \parenb{\regionset}|$ bits of information to specify (enumerate) all the possible target concepts \tt{i.e.}  $\log_2 \parenb{Q(n,d')} = \bigOmega{d'\cdot \log n}$ many for sufficiently large $n$. As discussed in \cite{jamieson2011active}, we note that
the overall computational complexity of the algorithm is $\cO(n\: \textbf{poly}(d)\: \textbf{poly}(\log n))$ because in total the number of queries requested are at max\footnote{In the case of $d'$-relaxed general position, the number of queries requested is $\cO(d'\log n)$.} $\cO(d \log n)$ and the complexity of each test is polynomial in the number of
queries requested because each one is a linear constraint. In fact, we could also show that our \algoref{algo: query selection} is attribute efficient \citep{attribute}. As defined, 
we could think of finding the exact labelling function (cf. \secref{sec:formulation}) as learning the boolean function \citep{attribute}.
If $d = \textbf{poly}(n)$ (or $d$ is small compared to $n$) then algorithm runs in $\mathcal{O}(\textbf{poly}(n)\cdot \textbf{poly}(\log n)) = \textbf{poly}(n)$, and hence is attribute efficient.\vspace{6pt}\\
\noindent
Our key observation is that the sequential algorithm doesn't ask for labels for \tt{non-trivial} number of hyperplanes since they are \tt{unambiguous} or \tt{uninformative} $\mathrm{wrt}$ to the target region. Our adaptive algorithm filters out such queries irrespective of the ordering in which the hyperplanes are queried for the enclosing region. In the following subsection, we formally provide the characterization of \tt{ambiguous} hyperplane queries which is based on our \defref{ambiguous hyperplane}.
\subsection{Characterization of an Ambiguous Query of a Hyperplane }\label{appendixsub: characterization}

In \defref{ambiguous hyperplane}, we gave the characterization for an ambiguous hyerplane $\mathrm{wrt}$ to a subset $\mathcal{S} \subset \Hnd$. \citet{jamieson2011active} gave similar characterization but for bisecting hyperplanes corresponding to pairwise queries of embedded objects. With our characterization we are able to show similar results which we use to give a bound on the query complexity.

\begin{algorithm}
\caption{Query Selection Algorithm}
\label{algo: query selection}
\setcounter{AlgoLine}{0}
\nl {\bf Input}: $n$ hyperplanes in $\mathbb{R}^d$\\
\Begin{
\nl {\bf Initialize}: hyperplanes $\Hnd$ = $\curlybracket{h^{(1)}, h^{(2)}, \dots, h^{(n)}}$  in uniformly random order\\
\For {i $\in$ $[n]$}{
\If{$h^{(i)}$ is ambiguous}{
\nl request $h^{(i)}$’s label from reference
}
\Else{
\nl impute $h^{(i)}$’s label from previously labeled queries.
}
}
}
\nl {\bf Output}: target region(region)\;
\end{algorithm}
As mentioned in \cite{jamieson2011active}, we call the arrangement of the set of $n$ hyperplanes in $\Rd$ as an $n$-\tt{partition} and a region induced by the arrangement as a $d$-\tt{cell}. Now consider the basic sequential procedure of \algoref{algo: query selection}. $\mathrm{WLOG}$, assume that the algorithm samples the $k$ hyperplanes  in the order $\curlybracket{h^{(1)},\cdots, h^{(k)}}$. It is not very difficult to see that the target region $r$ is contained within a $d$-cell, $C_{k}$ (defined by the labels of the queried hyperplanes from  $h^{(1)}$ through $h^{(k)}$. Assume that $h^{(k+1)}$ is sampled in the next iteration. Querying $h^{(k+1)}$ for labels is informative (\tt{i.e.}, ambiguous) \tt{iff} it intersects this $d$-cell $C_{k}$. We realize that this observation is significant because if $k$ is sufficiently larger than $d$, then the probability that the next sampled hyperplane intersects $C_{k}$ is very small; in fact the probability is on the order of $1/k$ (proved in \lemref{probability ambiguity hyperplane}). In the next subsection, we provide the proof of \lemref{probability ambiguity hyperplane} which ascertains a bound on the proabability that a sampled hyperplane is ambiguous for query.

\subsection{Probability of Ambiguity: Proof of \lemref{probability ambiguity hyperplane}} 
In this subsection, we would show that on a random ordering of hyperplanes, it is highly likely that a hyperplane query is unambiguous. This is the essential component of the query selection algorithm. We would start by stating an important result which would allow us to argue the probability with which a randomly sampled hyperplane is ambiguous. We denote a target hypothesis(region) by $r$.
\begin{lemma}\label{equally probable}
 Assume  $r \sim \mathcal{U}$. Consider the subset $\mathcal{S} \subset \Hnd$ with $|\mathcal{S}| = k$ that is randomly
selected from $\Hnd$ such that all $\binom{n}{k}$
subsets are equally probable. If $\boldsymbol{\mathfrak{R}}\parenb{\mathcal{A}(\mathcal{S})}$ denotes the set of regions induced by the arrangement of $\mathcal{S}$, then every $r \in \boldsymbol{\mathfrak{R}}\parenb{\mathcal{A}(\mathcal{S})}$ is equally probable (where $Q(k,d) = |\boldsymbol{\mathfrak{R}}\parenb{\mathcal{A}(\mathcal{S})}|$).
\end{lemma}
\begin{proof}
This lemma follows immediately using \citet[Lemma 3]{jamieson2011active}. Any uniformly random selection of $k$-tuple of hyperplanes induces $k$-partition of the $d$-dimensional space. Each $k$-partition contains some $d$-cells of $n$-partition induced by the arrangement of all the hyperplanes. Since the $k$-tuple has been uniform randomly selected and each $d$-cell of the $n$-partition is equally probable, thus there are $Q(n,d)/Q(k,d)$ $d$-cells of the $n$-partition in any $d$-cell of the $k$-partition. As each $d$-cell of the $n$-partition is equally probable which implies, probability mass in each $d$-cell of $k$-partition is $Q(n,d)/Q(k,d) \times 1/Q(n,d)$ = $1/Q(k,d)$. Hence, the lemma follows.
\end{proof}
We would state an easy inequality that we would use in the subsequent lemmas.
\begin{lemma}\label{ratio bound}
For $k > 2d$, the following inequality holds:
\begin{equation*}
    \frac{Q(k,d-1)}{Q(k,d)} \le \frac{d}{k/2}
\end{equation*}
\end{lemma}
\begin{proof}
First note that,
\begin{equation}
     d+\frac{(k-d+1)(k-2d+3)}{(k-d+2)} \ge \frac{k}{2} \label{eqn:eqratio}
\end{equation}
Using the following simplification, \eqnref{eqn:eqratio} holds.
\begin{align*}
    & 2d(k-d+2)+2(k-d+1)(k-2d+3)-k(k-d+2)\\
    &= (2d-k)(k-d+2) + 2(k-d+1)(k-2d+3)\\
    &= (2d-k)(k-d+1) + (2d-k) + 2(k-d+1)(k-2d+3)\\
    &= (k-d+1)\big[2(k-2d+3)+(2d-k)\big] + (2d-k)\\
    &= (k-d+1)(k-2d+6)-(k-2d) \ge 0
\end{align*}
Now, we would the result in the following computation: 
\begin{align}
 \frac{Q(k,d-1)}{Q(k,d)} &= 1\bigg/ \Biggparen{1+\frac{\binom{k}{d}}{Q(k,d-1)}}\:  \label{eqn:eqr01}\\
&\le 1\bigg/ \Biggparen{1+\frac{\binom{k}{d}}{\binom{k}{d-1}\frac{k-d+2}{k-2d+3}}}\:  \label{eqn:eqr02}\\
&= 1\bigg/ \Biggparen{1+\frac{(k-d+1)(k-2d+3)}{d(k-d+2)}} \nonumber\\
&= d\bigg/ \Biggparen{d+\frac{(k-d+1)(k-2d+3)}{(k-d+2)}} \nonumber\\
&\le \frac{d}{k/2} \label{eqn:eqr03}
\end{align}
Using \lemref{d-general} and \corref{d-exact}, we have $Q(k,d) = Q(k,d-1) + \binom{k}{d}$, which gives \eqnref{eqn:eqr01}. \eqnref{eqn:eqr02} is the straight forward consequence of \corref{cor: bounds regions} \tt{i.e.} $Q(k,d-1) \le \binom{k}{d-1}\frac{k-d+2}{k-2d+3}$. Finally, we use \eqnref{eqn:eqratio} to get \eqnref{eqn:eqr03}.
\end{proof}
\noindent
Now, we would talk about the probability of ambiguity of any randomly selected hyperplane. If we assume that $k$ hyperplanes have been selected uniformly at random, they induce a $k$-partition. We can ascertain the probability of the event of $(k+1)$th sampled hyperplane to be ambiguous conditioned on the labels queried/imputed of the first $k$ hyperplanes. We state the result in the \lemref{probability ambiguity hyperplane}. 
\newline
\begin{proof}[Proof of \lemref{probability ambiguity hyperplane}]
The first $k$ sampled hyperplanes induce a $k$-partition. The target region $r$ belongs to one of the $d$-cells, say $C_k$ in the $k$-partition. According to the characterization, hyperplane query for $h^{(k+1)}$ is ambiguous if it intersects $C_k$.  Let $P(k, d)$ denote the number of $d$-cells in the $k$-partition that are intersected by the hyperplane $h^{(k+1)}$. Using \lemref{equally probable}, we know that each of the $d$-cell in the $k$-partition is equally probable. Thus, probability of $q_{h^{(k+1)}}$ being ambiguous is same as the probability of each $d$-cell that $h^{(k+1)}$ intersects times the number of $d$-cells it intersects in the $k$-partition. Thus we have:
\begin{align*}
    P_A(k, d, \mathcal{U}) = \frac{P(k,d)}{Q(k,d)} &= \frac{Q(k,d-1)}{Q(k,d)} 
                         \stackrel{\lemref{ratio bound}}{\le} \frac{d}{k/2}
\end{align*}
Thus, for $a = 2$, we have achieved a bound on the probability of the event of a hyperplane query being ambiguous. 
\end{proof}
\subsection{Proof of \thmref{pairwise teaching bound}}
We denote by $M_n$ the number of queries asked for by the algorithm. But this is same as the number of queries being requested by the Query Selection Algorithm. Thus, we have $M_n = \sum_{i=1}^n {\bf 1}\{q_{h^{(i)}} \text{is requested}\}$.\vspace{6pt}\\
\noindent
We would provide the proof of the bound for the average-case complexity for active learning of convex polytopes in the main theorem of the section \thmref{pairwise teaching bound}.
\newline
\begin{proof}[Proof of \thmref{pairwise teaching bound}]
Let us denote the event of requesting the query for hyperplane $h^{(k)}$ for each $k$ by $B_k$. Note that each $B_k = {\bf 1}\curlybracket{q_{h^{(k)}} \text{is requested}}$ is a bernoulli distribution with parameter $P_A(k,d,\mathcal{U})$. Since, the bounds of $P_A(k,d,\mathcal{U})$ makes sense when $k > 2d$ so we assume that for $k \le 2d$, all the queries are ambiguous.
\begin{align*}
    \mathbb{E}_{\mathcal{U}}[M_n] 
    &= \sum_{i=1}^n\mathbb{E}_{\mathcal{U}}[B_i]\\
    &\le \sum_{i=1}^{2d}\mathbb{E}_{\mathcal{U}}[B_i] + \sum_{i= 2d+1}^n\mathbb{E}_{\mathcal{U}}[B_i]\\
    &\le 2d + \sum_{i= 2d+1}^n \frac{2d}{i}\\
    &\le 2d + 2d\log_2\left(\frac{n}{2d+1}\right)\\
    &= 2d\log_2\left(\frac{2n}{2d+1}\right) \le 2d\log_2(n)
\end{align*}
This completes the proof.
\end{proof}
\noindent
Thus, for a set of hyperplanes $\Hnd$ arranged in general position, we provide an algorithm with $\mathcal{O}(d \cdot \log n)$ average query complexity for active learning of an enclosing region for target region. 
\paragraph{Generalization to $d'$-relaxed general position} We note that with similar arguments we can achieve the bound of $\mathcal{O}(d'\cdot \log n)$ if the set of hyperplanes are arranged in $d'$-\tt{relaxed general position}. It is not very difficult to see that \thmref{thm:d'-general} and \propref{faces_teach} would yield similar results as \lemref{equally probable} and \lemref{probability ambiguity hyperplane} and then a result similar to \thmref{pairwise teaching bound} follows. We note that in the case of $d'$-\tt{relaxed general position} arrangement, the number of regions induced in $\Rd$ by $n$ hyperplanes is $Q(n,d')$. Similarly, the number of faces induced on a hyperplane turns out to be $Q(n-1, d'-1)$ (intersection of $n$ hyperplanes). \lemref{equally probable} and \lemref{probability ambiguity hyperplane} can be extended for the relaxed case by straight-forward replacement of $Q(\cdot,d')$ and $Q(\cdot,d'-1)$ for number of regions and faces accordingly. 

Earlier we argued on the lower bound which turns out to be $\bigOmega{d' \log n}$ (see \appref{appendixsub: average-case analysis}). With the upper bound of $\mathcal{O}(d'\cdot \log n)$ on the label complexity, thus we achieve the strong bound of $\Theta(d'\cdot \log n)$ for active learning of convex polytopes as shown in \tabref{tab:sample-complexity}.

For the \tt{worst-case} complexity of active learning of convex polytopes, we notice that it has to be $\bigTheta{n}$ since the lower bound holds because of the lower bound of $\bigOmega{n}$ for worst-case teaching complexity as shown in \appref{appendixsub: worst case}. It implies that there exists a worst-case construction of a target regions such that no matter how the ordering of the hyperplanes are initialized, every sampled hyperplane in any iteration of \algoref{algo: query selection} would be ambiguous requiring all the halfspace queries to be made to determine the target region. Since $n$ queries are sufficient thus the worst-case sample complexity of active learning of convex polytopes is $\bigTheta{n}$.

This completes the proof of the main theorem of the section.

\clearpage
 \section{Dual Map for $\phi$-Separable Dichotomy: Proof of \thmref{dual map}}\label{appendix: Dual map}
In this appendix, we provide the proof of our main result for the construction of dual map \tt{i.e.} \thmref{dual map}. Using the properties of the dual map and bounds on the average teaching complexity for convex polytopes \tt{i.e.} \thmref{theorem: Main Theorem}, we provide the proof of \corref{cor: teach phi surfaces} which establishes similar bound on the average teaching complexity of $\phi$-separable dichotomies. We first state and prove the necessary lemmas and results in order to prove \thmref{dual map}. Before that, we mention a fundamental result from linear algebra \citep[also mentioned in][Theorem 2.8]{roman2007advanced} which would be used in a number of lemmas across appendices.

\begin{theorem}[Rank-Nullity Theorem]\label{theorem:rank nullity}
Let $V$ and $W$ be vector spaces over a field $F$, and let $T$: $V \rightarrow W$ be a linear transformation. Assuming the dimension of $V$ is finite, then 
\begin{equation}
    \dim(V) = \dim(\textnormal{Ker}(T)) + \dim(\textnormal{Im}(T)) 
\end{equation}
where $\dim(\textnormal{Ker}(T))$ is nullity of $T$ and $\dim(\textnormal{Im}(T))$ is the rank of $T$.
\end{theorem}
\subsection{Relevant Lemmas for Proof of \thmref{dual map}}
First, we would prove a straight-forward result for homogeneous linear separability which forms the basis for the equivalence relation we obtained in \secref{sec.applied}.
\begin{lemma}\label{lemma:symmetry}
If $w$ is the normal vector for the homogeneous linear separator of $\curlybracket{\xndp,\xndn}$ then, $-w$ is the normal vector for the homogeneous linear separator of $\curlybracket{\xndn,\xnd}$.
\end{lemma}
\begin{proof} If $w$ is the normal vector for a homogeneous linear separator of $\curlybracket{\xndp,\xndn}$, then,
\begin{align*}
    w\cdot x > 0 \Leftrightarrow  (-w)\cdot x < 0\:  \: \text{if}\: x  \in \xndp\\
     w\cdot x < 0 \Leftrightarrow  (-w)\cdot x > 0\:  \: \text{if}\: x  \in \xndp
\end{align*}
Thus, $-w$ is the the normal vector for a homogeneous linear separator of $\curlybracket{\xndn,\xndp}$
\end{proof}
\noindent
To study the arrangement of dual hyperplanes, we define the matrices $\boldlam_{[(n-1)\times d]}$ and $\bracket{\boldlam_h}_{[(n-1)\times (d-1)]}$ such that $\forall$ $i \in \bracket{n-1}$ $\boldlam\bracket{i,:} = x^{(i)}$ and $\boldlam_h\bracket{i,:} = x_{[d-1]}^{(i)}$ where $x_{[d-1]}$ is first $d-1$ components of $x$. Using the $d'$-relaxed general position arrangement of $\xnd$ and nullity of $x^{(n)}$ as a dimension, in \lemref{rank dual} we show that $\rank{\boldlam_h}$ = $d'-1$ and any ($d'-1$) rows of $\boldlam_h$ are linearly independent .
\begin{lemma}\label{rank dual}
For the matrices constructed above, $\rank{\boldlam_h}$ = $d'-1$, and any ($d'-1$) rows of $\boldlam_h$ are linearly independent.
\end{lemma}
\begin{proof}
First part of the lemma is straight-forward since, by definition any $d'$ vectors in $\xnd$ are linearly independent which means $d'$ columns of $\boldlam$ are linearly independent, implying $(d'-1)$ columns of $\boldlam_h$ are linearly independent.\vspace{6pt}\\
\noindent
For the second part, for an indexed set $\mathbb{I}_{[d'-1]}$ $\triangleq$ $\curlybracket{i_1,i_2,\cdots,i_{d'-1}}$ consider the $(d'-1)$ rows
$\curlybracket{x_{[d-1]}^{(i_1)},x_{[d-1]}^{(i_2)},\cdots,x_{[d-1]}^{(i_{d'-1})}}$ of $\boldlam_h$ which are linearly dependent. Thus, $\exists$ scalars $\alpha_j$'s (not all zeros) such that:
\begin{align}
    \sum_{j=1}^{d'-1} \alpha_j\cdot x_{[d-1]}^{(i_j)} = 0 &\implies \sum_{j=1}^{d'-1} \alpha_j\cdot \bigparen{x_{[d-1]}^{(i_j)}, x^{(i_j)}_d} - \left(\sum_{j=1}^{d'-1}\alpha_j\cdot x^{(i_j)}_d\right)\cdot x^{(n)} = 0 \label{eqn:eqnd01}\\
    & \implies \sum_{j=1}^{d'-1}\alpha_j\cdot x^{(i_j)} - \left(\sum_{j=1}^{d'-1}\alpha_j\cdot x^{(i_j)}_d\right)\cdot x^{(n)} = 0 \label{eqn:eqnd02}
\end{align}
In \eqnref{eqn:eqnd01} we use that $x^{(n)} = \mathbi{e}_d$.
\eqnref{eqn:eqnd02} implies that we have $d'$ vectors of $\xnd$ linearly dependent. Contradiction! Thus, for any indexed set $\mathbb{I}_{[d'-1]}$, the corresponding submatrix of dimension $[d'-1 \times d'-1]$ of $\boldlam_h$, is full rank. Hence, the second part of the lemma is proven.
\end{proof}
Now, we would give the proof of the key lemma of duality which shows that the mapped hyperplanes follow the criterion of ($d'-1$)-relaxed general position. For the sake of clarity and flow, we would restart with the construction of sets. Let us define $\mathbb{I}_{[k]}$ $\triangleq$ $\{i_1,i_2,...,i_k\}$ as $k$ indices for rows. Denote by $\boldlam_{\mathbb{I}_{[k]}}$ = $\boldlam_{[\mathbb{I}_{[k]}\times d]}$ (rows of $\mathbf{\boldlam}$ corresponding to $\mathbb{I}_{[k]}$) and by $(\boldlam_h)_{\mathbb{I}_{[k]}}$ = $(\boldlam_h)_{[\mathbb{I}_{[k]}\times d-1]}$ (rows of $\boldlam_h$ corresponding to $\mathbb{I}_{[k]}$). As in \secref{sec.applied}, we redefine $\xnd$ $\triangleq$ $\curlybracket{x^{(1)}, x^{(2)}, \dots, x^{(n)}}$.

\begin{lemma}[Key lemma of duality]\label{key lemma of duality}
If $\mathcal{S}^k_{\cap}$ = $\condcurlybracket{x \in \mathbb{R}^{d-1}}{\boldlam_{\mathbb{I}_{[k]}}\binom{x}{1} = \mathbf{0}^k}$ for $1 \le k \le (d'-1)$, then $\dim(\mathcal{S}^k_{\cap}) = (d-1)-k$. Moreover, no $d'$ rows of $\boldlam$ intersects in dual space \tt{i.e.} $\condcurlybracket{x \in \mathbb{R}^{d-1}}{\boldlam_{\mathbb{I}_{[d']}}\binom{x}{1} = \mathbf{0}^{d'}} = \emptyset$.
\end{lemma}
\begin{proof}[Proof of \lemref{key lemma of duality} of Duality]
Define by $\mathbi{b} \triangleq \parenb{x^{(i_1)}_d, x^{(i_2)}_d,\cdots,x^{(i_k)}_d}$. Notice that, 
\begin{align}
\mbox{\large\(
    \boldlam_{\mathbb{I}_{[k]}}\binom{x}{1} = 0 \Longleftrightarrow  (\boldlam_h)_{\mathbb{I}_{[k]}}x = -\mathbi{b}^{\top} \label{eqn:eqnd03}
\)}
\end{align}
If $k = d'-1$ then $(\boldlam_h)_{\mathbb{I}_{[d'-1]}}$ is $d'-1$ rank invertible matrix implying \eqnref{eqn:eqnd03} has a \textit{unique} solution.
\newline
Note that using \lemref{rank dual}, $(\boldlam_h)_{\mathbb{I}_{[k]}}$ has rank $k$ for $k < d'$. This implies that there is some $x_0 \in \mathbb{R}^{d-1}$ such that $(\boldlam_h)_{\mathbb{I}_{[k]}}x_0 = -\mathbi{b}^{\top}$. Thus, we rewrite \eqnref{eqn:eqnd03} as
\begin{align*}
\mbox{\large\(
    \boldlam_{\mathbb{I}_{[k]}}\binom{x}{1} = 0 \Longleftrightarrow (\boldlam_h)_{\mathbb{I}_{[k]}}x = (\boldlam_h)_{\mathbb{I}_{[k]}}x_0 \Longleftrightarrow (\boldlam_h)_{\mathbb{I}_{[k]}}(x-x_0) = 0
    \)}
\end{align*}
But using \thmref{theorem:rank nullity}, \textbf{Ker}$\parenb{(\boldlam_h)_{\mathbb{I}_{[k]}}}$ = $(d-1)-k$. This implies that $\dim\bigparen{\condcurlybracket{x \in \mathbb{R}^{d-1}}{(\boldlam_h)_{\mathbb{I}_{[k]}}(x-x_0) = 0}}$ = $(d-1)-k$. Thus, $\dim(\mathcal{S}^k_{\cap}) = (d-1)-k$. \\
\newline
Notice that if $\boldlam_{\mathbb{I}_{[d']}}\binom{x}{1} = \mathbf{0}^{d'}$ has a solution then we can define $\binom{x}{1}$ as a homogeneous linear separator and the points of $\xnd$ corresponding to $\boldlam_{\mathbb{I}_{[d']}}$ lie on a $(d-1)$-dimensional halfspace (subspace) defined by $\binom{x}{1}$. Note, $\dt{}^{(n)}$ doesn't lie on that subspace. On the other hand, because of $d'$-relaxed general position arrangement of $\xnd$, rows of $\boldlam_{\mathbb{I}_{[d']}}$ are linearly independent and lie on the subspace. It implies $\textbf{rows}(\boldlam_{\mathbb{I}_{[d']}})\cup \{\dt{}^{(n)}\}$ are linearly independent. Contradiction. Thus, $\condcurlybracket{x \in \mathbb{R}^{d-1}}{\boldlam_{\mathbb{I}_{[d']}}\binom{x}{1} = \mathbf{0}^{d'}} = \emptyset$.
Hence, the lemma follows.
\end{proof}
With \lemref{key lemma of duality} and \eqnref{eqn:construction}, we can formally prove our main theorem of the section on the dual map which says the dual set of hyperplanes are in $(d'-1)$-\tt{relaxed general position} and each equivalence class of dichotomies $\mathfrak{E}\parenb{\xnd}$ maps \tt{uniquely} to all the concepts (hypotheses) in $\boldsymbol{\mathfrak{R}}\parenb{\mathcal{A}(\Hndo)}$.

\subsection{Proof of \thmref{dual map} and \corref{cor: teach phi surfaces}}
In this subsection, we provide the proof of the results of interest. Following the notations in \secref{sec.applied}, we use slightly different notations in the proofs for the sake of clarity.
For a dichotomy class $\bracket{\mathbi{v}} \in \mathfrak{E}\parenb{\xnd}$, we denote the dual point to a separator $w_{\bracket{\mathbi{v}}}$ of the representative dichotomy by $z_{w_{\bracket{\mathbi{v}}}}$ and region corresponding to $z_{w_{\bracket{\mathbi{v}}}}$ as\footnote{In section \secref{sec.applied}, we denote the dual point of the separator  $w_{\bracket{\mathbi{v}}}$ to $\bracket{\mathbi{v}}$ as $\h{}_{\bracket{\mathbi{v}}}$ and region containing $\h{}_{\bracket{\mathbi{v}}}$ as $r_{\h{}_{\bracket{\mathbi{v}}}}$.} $ r_{\h{}_{w_{\bracket{\mathbi{v}}}}} \in \boldsymbol{\mathfrak{R}}\parenb{\mathcal{A}(\boldsymbol{\bar{\mathcal{H}}})}$ such that $z_{w_{\bracket{\mathbi{v}}}} 
\in r_{\h{}_{w_{\bracket{\mathbi{v}}}}}$ \tt{i.e.} $r_{\h{}_{w_{\bracket{\mathbi{v}}}}} = \varphi_{\mathrm{dual}}(\bracket{\mathbi{v}})$.\\

\begin{proof}[Proof of \thmref{dual map}]
By the definition of \ref{tag: d.m}, we get $\Hndo = \Upsilon_{\mathrm{dual}}\parenb{\xnd}$. We constructed the matrices $\boldlam_{[n-1\times d]}$ and $\bracket{\boldlam_h}_{[n-1\times d-1]}$ to study the arrangement of dual hyperplanes. In the Key \lemref{key lemma of duality} of Duality, we proved that $\forall$ $1 \le k \le d'-1$, any size $k$ subset of $\Hndo$ intersects in a flat of dimension $(d-1-k)$ and no $d'$ dual hyperplanes intersect at a point. Thus, we show that $\Hndo$ is in $(d'-1)$-relaxed general position arrangement which proves the first part of the theorem.

First, we notice that $\varphi_{\mathrm{dual}}$ is well-defined since \eqnref{eqn:construction} is a sign preserving construction. 
To prove the bijection of $\varphi_{\mathrm{dual}}$, we first show that it is an \tt{injection}. We assume that $\#\mathfrak{E}\parenb{\xnd} > 1$ since the other case can be handled trivially.
Denote by $\bracket{\mathbi{u}}, \bracket{\mathbi{v}}$ two different equivalence classes of $\mathfrak{E}\parenb{\xnd}$. Let $w_{\bracket{\mathbi{u}}}$ and $w_{\bracket{\mathbi{v}}}$ be two corresponding linear separators respectively. Since $\bracket{\mathbi{u}} \neq \bracket{\mathbi{v}}$, $\exists$ at \tt{least} one point $x' \neq x^{(n)} \in \xnd$ which is classified/labeled \tt{differently}. Consider the dual hyperplane $h_{x'} = \Upsilon_{\mathrm{dual}}(x')$, and the dual points $z_{w_{\bracket{\mathbi{u}}}}$ and $z_{w_{\bracket{\mathbi{v}}}}$ of $w_{\bracket{\mathbi{u}}}$ and $w_{\bracket{\mathbi{v}}}$ respectively using the construction shown in \eqnref{eqn:construction}. Since $w_{\bracket{\mathbi{u}}}$ and $w_{\bracket{\mathbi{v}}}$ classify $x'$ differently, $z_{w_{\bracket{\mathbi{u}}}}$ and $z_{w_{\bracket{\mathbi{v}}}}$ belongs to two different regions of $h_{x'}$, implying $r_{z_{w_{\bracket{\mathbi{u}}}}} \not\equiv r_{z_{w_{\bracket{\mathbi{v}}}}}$ where $r_{z_{w_{\bracket{\mathbi{u}}}}} = \varphi_{\mathrm{dual}}(\bracket{\mathbi{u}})$ and $r_{z_{w_{\bracket{\mathbi{v}}}}} = \varphi_{\mathrm{dual}}(\bracket{\mathbi{v}})$. Thus, $\varphi_{\mathrm{dual}}$ is an injection. 
Consider a region $r \in \boldsymbol{\mathfrak{R}}\parenb{\mathcal{A}(\boldsymbol{\bar{\mathcal{H}}})}$. Pick a point $z_0 \in r$. Now, define $w_z \triangleq (z_0^{\top},1)$. Since $z_0 \in r$, $w_z$ is a homogeneous linear separator of a dichotomy in the primal space corresponding to $r$ where dichotomy is defined by signs using \eqnref{eqn:construction}. Note that it is a valid dichotomy since $0\cdot z_0 + 1 > 0$ implying $(z_0^{\top},1)$ labels $x^{(n)}$ positively. We represent the dichotomy using the class $\bracket{\mathbi{u}}$. Since, $z_0$ is arbitrary, thus $\varphi_{\mathrm{dual}}^{-1}(r)= \bracket{\mathbi{u}}$ implying surjection of $\varphi_{\mathrm{dual}}$. Hence, we show $\varphi_{\mathrm{dual}}$ is a bijection.
\end{proof}
\noindent
The properties of the dual map is key in showing the bound on the teaching complexity of $\phi$-separable dichotomies. We note that the dual map retains the arrangement of the general position of points (\defref{general position vectors}) to relaxed general position of hyperplanes in the dual space (\defref{defn relaxed general position}). Thus, our bound on the average teaching complexity of convex polytopes in \thmref{theorem: Main Theorem} applies in the case of average teaching complexity of $\phi$-separable dichotomies which we show in \corref{cor: teach phi surfaces}. We present the proof of the corollary here. 
\newline
\begin{proof}[Proof of \corref{cor: teach phi surfaces}] For the set $\xnd$, we consider the set of $\phi$-induced points $\phi(\xnd)$ =  $\{\phi(x^{(1)}),\phi(x^{(2)}),\dots,\phi(x^{(n)})\}$ in the $\phi$ induced primal space $\mathbb{R}^{d_{\phi}}$. For the $\phi$-separable dichotomies of $\xnd$, we denote the quotient set of equivalence classes of dichotomies as $\boldsymbol{\mathfrak{E}}_{\xnd}^{\phi}$. Since $\xnd$ are in $d'_\phi$-relaxed $\phi$-general position  for a fixed $d'_\phi \in [d_\phi]$ 
, we can apply the dual map $\bracket{\Upsilon_{\mathrm{dual}} ,\varphi_{\mathrm{dual}}}$ on the pair $[\phi(\xnd), \boldsymbol{\mathfrak{E}}_{\xnd}^{\phi}]$. We denote the set of 
$d'_\phi-1$-relaxed general position
dual hyperplanes by $\boldsymbol{\mathcal{H}}_{n-1,d'_{\phi}-1} \triangleq \Upsilon_{\mathrm{dual}}(\phi(\xnd))$, and the set of dual regions as $\boldsymbol{\mathfrak{R}}\parenb{\mathcal{A}(\boldsymbol{\mathcal{H}}_{n-1,d'_{\phi}-1})} \triangleq \varphi_{\mathrm{dual}}(\boldsymbol{\mathfrak{E}}_{\xnd}^{\phi})$. Using the definition of the teaching set for $\phi$-separable dichotomies and bijection of $\varphi_{\mathrm{dual}}(\cdot)$ (using \thmref{dual map}), we can write:
\begin{equation}
\expctover{\rgn_{[\mathbi{u}]} \sim \mathcal{U}}{M_n} =
\expctover{r\sim\cU}{|\mathcal{TS}(\boldsymbol{\mathcal{H}}_{n-1,d'_{\phi}-1},r)|} \label{eqn:phibound}
\end{equation}
where $\rgn_{[\mathbi{u}]}$ is a random class in $\boldsymbol{\mathfrak{E}}_{\xnd}^{\phi}$ and $r$ is a uniformly random region in $\boldsymbol{\mathfrak{R}}\parenb{\mathcal{A}(\boldsymbol{\mathcal{H}}_{n-1,d'_{\phi}-1})}$. But, using \thmref{theorem: Main Theorem}, we know that $\mathrm{rhs}$ in \eqnref{eqn:phibound} is bounded by $\cO(d_{\phi}')$. Thus, we show that the average teaching complexity of $\phi$-separable dichotomies is $\cO(d_{\phi}')$. This proves the corollary.
\end{proof}
\clearpage
 \section{Equivalence of Teaching Set and Extreme Points: Proof of \thmref{main extreme points}}\label{appendix: extreme points}
In this section, we would talk about the connection of teaching set in the dual space and extreme points in primal space as mentioned in \citet{cover1965geometrical}. In order to complete the proof of the main result \thmref{main extreme points} we would prove two lemmas: \lemref{equivalence first side} and \lemref{equivalence second side}.

In \secref{subsection:connection}, we discussed the characterization of \tt{ambiguous} points in the primal space. Formally, we state the lemma mentioned in \citet{cover1965geometrical} to characterize ambiguous points.

\begin{lemma}[Lemma 1, \citet{cover1965geometrical}]\label{ambiguous cover} Let $X^+$ and $X^-$ be subsets of  $\mathbb{R}^d$, and let $y$ be a point other than the origin in $\mathbb{R}^d$. Then the dichotomies $\{X^+\cup \{y\}, X^- \}$ and $\{X^+, X^-\cup \{y\} \}$ are both
homogeneously linearly separable if and only if $\{ X^+, X^- \}$ is homogeneously linearly separable by a ($d-1$)-dimensional subspace containing $y$. 
\end{lemma}
Using this lemma we can argue on the equivalence of the ambiguous points in the primal space and ambiguous hyperplanes in the dual space.  Let $P^+$ and $P^-$ be subsets of $X^+$ and $X^-$ respectively, whose classes/labels are ascertained (known). Denote by $H^+$ and $H^-$ (for $P^+$ and $P^-$) the corresponding subsets of dual hyperplanes in the dual space. Assume that $y$ is a new point in the primal space. Due to the nature of the dual map which uses the information of the vector $x^{(n)}$, we assume that the label for $x^{(n)}$ is known and $x^{(n)} \in P^+$. In the asymptotic analysis of our algorithms, this much information can be trivially included. We state this as a key assumption as mentioned in \assref{assumption:positivelabel}.\vspace{6pt}\\
\noindent
In section \secref{sec.applied}, we constructed a teaching set for a dichotomy via dual map.  
With the virtue of the \ref{tag: d.m}, we show the equivalence of extreme points in the primal space and teaching set in the dual space. In other words, extreme points are exactly the inverse of the teaching set in the dual space under $\bracket{\Upsilon_{\mathrm{dual}} ,\varphi_{\mathrm{dual}}}$.
In the next two lemmas we show that for the points $\mathcal{P}^+ \cup \mathcal{P}^-$ mapped to $\Upsilon_{\mathrm{dual}}\parenb{\mathcal{P}^+ \cup \mathcal{P}^-}$, $y$ is ambiguous $\mathrm{wrt}$ $\mathcal{P}^+ \cup \mathcal{P}^-$ \tt{iff} $\Upsilon_{\mathrm{dual}}\parenb{y}$ is ambiguous $\mathrm{wrt}$ to the region $\varphi_{\mathrm{dual}}\parenb{\bracket{\{\mathcal{P}^+,\mathcal{P}^-\}}}$. The key insights in establishing the connection is in using \eqnref{eqn:construction} and noting how \lemref{ambiguous cover} is essentially same as the characterization in \defref{ambiguous hyperplane}.

\begin{lemma}\label{equivalence first side}
If $y$ is ambiguous with respect to the partial dichotomy $\curlybracket{\mathcal{P}^+,\mathcal{P}^-}$, then $h_y\coloneqq \Upsilon_{\mathrm{dual}}\parenb{y}$ (dual hyperplane) is ambiguous with respect to $\varphi_{\mathrm{dual}}\parenb{\bracket{\{\mathcal{P}^+,\mathcal{P}^-\}}}$ \tt{i.e.} the region induced by the hyperplane arrangement of $\mathcal{H}_{\mathcal{P}^+ \cup \mathcal{P}^-}$.
\end{lemma}
\begin{proof}
Denote the region representing the partial dichotomy in the dual space by $r_{\mathrm{partial}}$. To show that, $h_y$ is ambiguous, we need to show that $h_y$ intersects $r_{\mathrm{partial}}$. Using \lemref{ambiguous cover}, we know that $y$ is ambiguous with respect to $\{ P^+, P^- \}$ \tt{iff} there exists
homogeneous linear separator $w_{y}$ for $\{ P^+, P^- \}$ passing through $y$. Notice that $w_{y}$ has a dual image (as a point) since $(w_y)_d$ > 0 as $w_y\cdot x^{(n)} > 0$. Say $z_{w_y}$ is the dual point then using \eqnref{eqn:construction}  $z_{w_y} \in r_{\mathrm{partial}}$ and since $w_y \cdot y$ = 0, it implies that hyperplane $h_y$ contains $z_{w_y}$. Hence, $h_y$ intersects $r_{\mathrm{partial}}$. Thus, lemma follows.
\end{proof}
Now, we would show that the pre-image (of dual map) of an ambiguous hyperplane with respect to a region in a hyperplane arrangement is 
an \textit{extreme point} for the corresponding dichotomy. Assume that the dual hyperplane of the point $y$ (in primal) is $h_y$ and it is ambiguous \tt{i.e.} it intersects the region corresponding to the partial dichotomy $\{ P^+, P^- \}$ in the dual space.

\begin{lemma}\label{equivalence second side}
If a hyperplane $h_y$ is ambiguous in the dual space, then $y \coloneqq \Upsilon_{\mathrm{dual}}^{-1}\parenb{h_y}$ is ambiguous in the primal space, where inverse of $\Upsilon_{\mathrm{dual}}$ is taken over the restriction $H^+ \cup H^-$.
\end{lemma}
\begin{proof}
To show that $y$ is ambiguous, we need to show that there is a homogeneous linear separator, say $w_y$ which separates the partial dichotomy $\{ P^+, P^- \}$ and passes through $y$. Similar to \lemref{equivalence first side}, define the region representing the partial dichotomy in the dual space by $r_{\mathrm{partial}}$. Since, $h_y$ intersects $r_{\mathrm{partial}}$, we know that there exists a point $z_0 \in r_{\mathrm{partial}}$ which lies on the hyperplane $h_y$. As shown in the construction in \eqnref{eqn:construction}, $h_y \equiv y_{[d-1]}\cdot z + y_d = 0$ for $z \in \mathbb{R}^{d-1}$. Now, define $w_y \triangleq (z_0^{\top},1)$. Note that, $y_{[d-1]}\cdot z_0 + y_d = 0$, thus implies $w_y\cdot y = 0$. Also, $w_y$ is a homogeneous linear separator of the partial dichotomy in the primal space since $z \in r_{partial.}$. Hence, we have shown that there exists a homogeneous linear hyperplane passing through $y$ and separating the partial dichotomy. Thus, $y$ is ambiguous. Hence, the lemma follows. 
\end{proof}
\noindent
Given that we have established the equivalence of \tt{ambiguous} points in the primal space and \tt{ambiguous} hyperplanes in the dual space, we can show the equivalence of extreme points and teaching set. We provide the proof of \thmref{main extreme points} here.\\
\begin{proof}[Proof of \thmref{main extreme points}]
$\mathrm{WLOG}$ we assume that $x^{(n)} \in \xndp$ as stated in \assref{assumption:positivelabel}. We denote the $\phi$-separable dichotomy class $\big[{\{\xndp,\xndn\}}\big]$ by $\bracket{\mathbi{u}}$. 
First, we show ($\Rightarrow$) \tt{i.e.} if condition.
Consider the mapped concept (dual region) $r_{z_{\bracket{\mathbi{u}}}} = \varphi\parenb{\bracket{\mathbi{u}}}$. Using \eqnref{eqn:construction} it is easy to see, if $T_s$ is the teaching set for $r_{z_{\bracket{\mathbi{u}}}}$, then using \lemref{equivalence second side}, $\Upsilon_{\mathrm{dual}}^{-1}\parenb{T_s}$ is \tt{ambiguous} $\mathrm{wrt}$ $\bracket{\mathbi{u}}$ following the characterization mentioned in \lemref{ambiguous cover}. This implies that $\Upsilon_{\mathrm{dual}}^{-1}\parenb{T_s} \subseteq E$. Now, using \lemref{equivalence first side}, since $E$ is ambiguous $\mathrm{wrt}$ $\bracket{\mathbi{u}}$ in the primal space, $\varphi_{\mathrm{dual}}\parenb{E}$ is ambiguous $\mathrm{wrt}$ $r_{z_{\bracket{\mathbi{v}}}}$ in the dual space. This implies $\Upsilon_{\mathrm{dual}}\parenb{E} \subseteq T_s$. Using the two sides of the containment, we have $\Upsilon_{\mathrm{dual}}\parenb{E} \equiv T_s$. This implies that $\Upsilon_{\mathrm{dual}}\parenb{E}$ is the teaching set for $\varphi_{\mathrm{dual}}\parenb{\big[{\{\xndp,\xndn\}}\big]}$.\vspace{6pt}\\
\noindent
Now, we show ($\Leftarrow$) \tt{i.e.} only if condition. Since $\Upsilon_{\mathrm{dual}}\parenb{E}$ is the teaching set for $\varphi_{\mathrm{dual}}\parenb{\big[{\{\xndp,\xndn\}}\big]}$, this implies $E$ is ambiguous in the primal space using \lemref{equivalence second side}, implying a subset of extremal points. We need to ascertain that $E$ is sufficiently a set of extremal points. Now, if $y' \notin E$ is ambiguous in the primal space, then $\Upsilon_{\mathrm{dual}}\parenb{y'}$ is \tt{ambiguous} in the dual space using \lemref{equivalence first side}. Thus, $\Upsilon_{\mathrm{dual}}\parenb{y'} \in \Upsilon_{\mathrm{dual}}\parenb{E}$ using the characterization of teaching set as stated in \defref{ambiguous hyperplane}. Hence, $E$ is sufficient. Thus, $E$ is a minimal set of extremal points.

Thus, we have proven the theorem. We show that the teaching set in the dual space is optimally recoverable as extreme points in the primal space.
\end{proof}


\clearpage
\section{Additional Use-case: Teaching Linear Ranking via Pairwise Comparisons}
\label{appendix: teaching ranking}
In this section, we would talk about the problem of teaching a randomly selected ranking of $n$ objects embedded in a $d$-dimensional space. Consider a set $\Theta$ of $n$ objects embedded in $\Rd$ (in general position). We define a ranking on the objects as an ordering $\sigma: \bracket{n} \rightarrow \bracket{n}$ of the form:
\begin{align*}
    \sigma\parenb{\Theta}:= \theta_{\sigma(1)} \prec \theta_{\sigma(2)}\cdots \prec \theta_{\sigma(n-1)} \prec \theta_{\sigma(n)}
\end{align*}
where $\theta_i \prec \theta_j$ implies $\theta_i$ precedes $\theta_j$ in ranking. The problem of interest is to construct a random ranking using \tt{pairwise comparisons} of the form:
\begin{align*}
    q_{i,j}:= \curlybracket{\theta_i \prec \theta_j}
\end{align*}
The response or label of $q_{i,j}$ is binary and denoted as $y_{i,j} := \mathbf{1} \curlybracket{q_{i,j}}$ where $\mathbf{1}$ is the indicator function; ties are not allowed. This is a well-studied problem in the literature and in the general setting it requires $\Theta(n \log n)$ bits of information to specify a ranking. But by imposing certain constraints on the embedding of the objects into the $d$-dimensional Euclidean space, \citet{jamieson2011active} shows we can get rid of the $n$ factor in the active query complexity.

\begin{figure*}[h]
\centering
		\includegraphics[width=.4\linewidth]{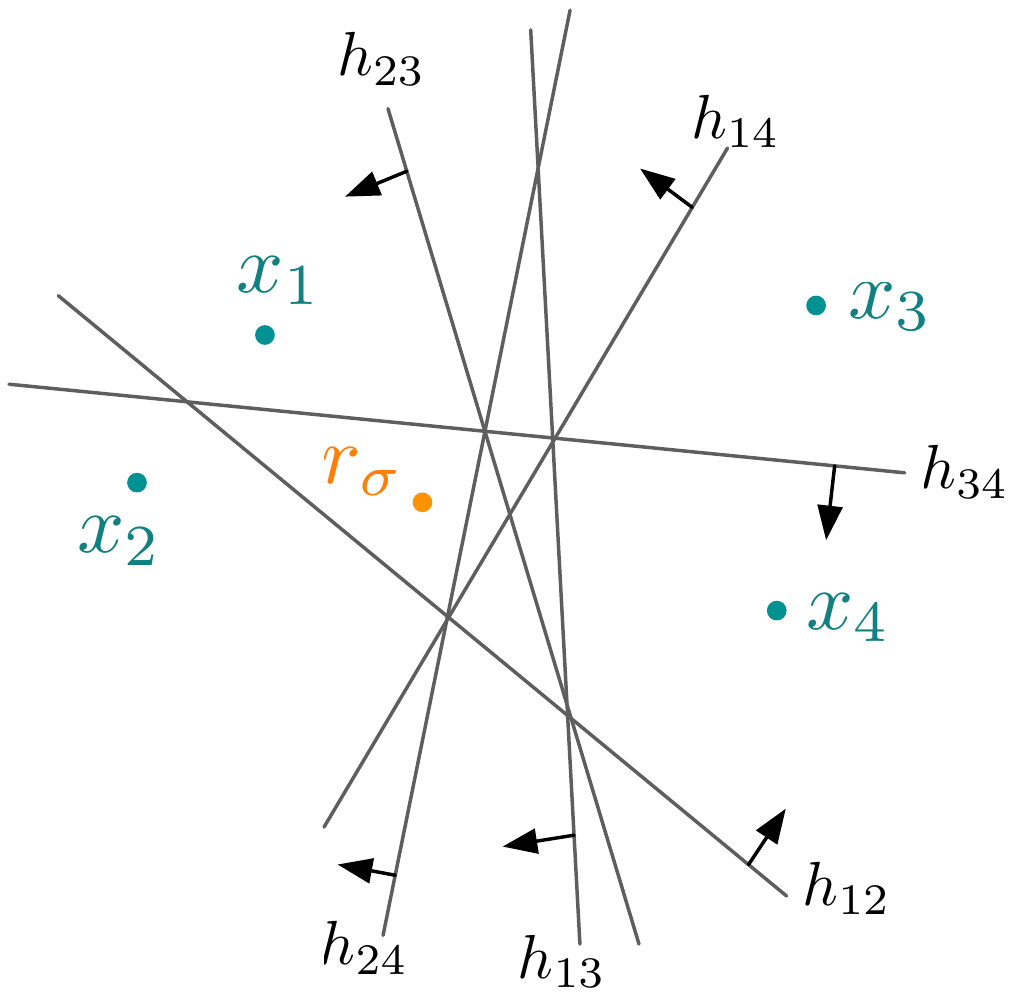}
  \caption{Teaching ranking via pairwise comparisons.}
	\label{fig:example.ranking}
	\vspace{-3mm}
\end{figure*}

We assume that for any ranking $\sigma$, there is a \tt{reference point} $r_{\sigma}$ such that if $\sigma$ ranks $\theta_i \prec \theta_j$ , then $||\theta_i - r_{\sigma}|| < ||\theta_j - r_{\sigma}||$. We refer to such assumption as {\bf E1}---This leads to an interpretation of a query ``is $\theta_i$ closer to $r_{\sigma}$ than $\theta_j$'', as identifying which side of the bisecting hyperplane (as shown in \defref{bisecting hyperplane}) of $\theta_i$ and $\theta_j$ does $r_{\sigma}$ lies in (as shown in \figref{fig:example.ranking}). 
Before we discuss our teaching results and connections to the prior work of \citet{jamieson2011active}, we mention our key assumption \citep[also mentioned in][]{jamieson2011active} over the space of rankings as follows:
\begin{assumption}[\textbf{E1 embedding}]
The set of $n$ objects are embedded in $\Rd$ (in general position) and we will also use $\theta_1,\theta_2,\cdots,\theta_n$ to refer to their (known) locations in $\Rd$. Every ranking $\sigma$ can be specified by a reference point $r_{\sigma}$ $\in \Rd$, as follows. The Euclidean distances between the reference and objects are consistent with the ranking in the following sense: if the $\sigma$ ranks $\theta_i \prec \theta_j$ , then $||\theta_i - r_{\sigma}|| < ||\theta_j - r_{\sigma}||$. Let $\Sigma_{n,d}$ denote the set of all possible rankings of the $n$ objects that satisfy this embedding condition.
\end{assumption}
We assume that every pairwise comparison is consistent with the ranking to be learned. That is, if the reference ranks $\theta_i \prec \theta_j$, then $\theta_i$ must precede $\theta_j$ in the (full) ranking. We define the notion of bisecting hyperplane corresponding to objects $\theta_i$ and $
\theta_j$ as follows:
\begin{definition}[Bisecting hyperplane]\label{bisecting hyperplane}
A hyperplane $h_{i,j}$ in $\Rd$ is a bisecting hyperplane to objects $\theta_i$ and $\theta_j$ if both are equidistant from $h_{i,j}$ and $h_{i,j}\cdot (\theta_i - \theta_j) = 0$.
\end{definition}
Thus, $n$ objects lead to $\binom{n}{2}$ hyperplanes (one query for each pair of objects) in $\Rd$:\\
\[\underbrace{n\: \textnormal{embedded objects in\:} \Rd}_{\textnormal{Rankings of\:} \Theta}
    \stackrel{\boldsymbol{\textbf{E1}}}{\:\leadsto\:} \underbrace{ {\tiny \binom{n}{2}} \textnormal{ hyperplanes in\:} \Rd}_{\textnormal{Convex polytopes : reference points}}\]
    
Each convex polytope corresponds to a reference point, thereby to a ranking of objects.
\paragraph{Geometric interpretation of E1} We summarize the geometric interpretation of the key assumption which follows similar motivations as given in \citet[section 3]{jamieson2011active}. If we consider two objects $\theta_i$ and $\theta_j$ in $\Rd$, querying for $y_{i,j}$ corresponding to $q_{i,j}$ is equivalent to ascertaining to which halfspace of the orthogonal bisecting hyperplane of $\theta_i$ and $\theta_j$, $r_{\sigma}$ belongs to. The set of all possible pairwise comparison queries can be represented as $\binom{n}{2}$ distinct halfspaces in $\Rd$. The intersections of these halfspaces partition $\Rd$ into a number of cells termed as $d$-cells, and each one corresponds to a unique ranking of $\Theta$. Arbitrary rankings are not possible due to the embedding assumption \textbf{E1}. Similar to \cite{jamieson2011active},  we represent the set of rankings possible under \textbf{E1} by $\Sigma_{n,d}$. The cardinality of $\Sigma_{n,d}$ is equal to the number of cells in the partition.\\
\newline
Now, we formulate the teaching problem of linear rankings under the mentioned assumptions here.
\paragraph{Teaching rankings as teaching convex polytopes} Denote the $\binom{n}{2}$ hyperplanes induced by pairwise-comparison of $n$ embedded objects by $\Hndp$.
Following our teaching framework in \secref{section: teaching framework}, we know that $\regionpair$ induced by $\cA(\Hndp)$ forms the underlying hypothesis class; with instances $\Hndp$ and corresponding labeling set $\{1,-1\}$. Thus, teaching a ranking $r_{\sigma}$ corresponds to providing the teaching set $\mathcal{TS}(\Hndp,r_{\sigma})$ to a learner.

Interestingly, we note that 
the hyperplanes induced by pairwise comparison of objects are no longer in general position. For example, in \figref{fig:example.ranking}, the three bisecting hyperplanes induced by any three points (in $\reals^2$) intersect at an 1-d subspace. When the embedded objects follow the assumption {\bf E1}(embedding) \citep{jamieson2011active}\footnote{We work in noise-free setting thus consistency is assumed similar to \cite{jamieson2011active}} show that the average query complexity for active ranking is $\mathcal{O}(d\log n)$. In contrast, we would show that the average teaching complexity of ranking via pairwise comparisons is $\mathcal{O}(d)$ via our \algoref{algo: teaching ranking}. 
\subsection{Algorithm for Teaching Rankings}
We present our basic algorithm for teaching a ranking via pairwise comparisons. We assume we are given a set of $n$ objects $\Theta$ embedded in $\Rd$ in general position and a uniformly random ranking $r_{\sigma} \in \Sigma_{n,d}$ over it.

\begin{algorithm}
\caption{Teaching Ranking via Pairwise Comparisons}
\label{algo: teaching ranking}
\setcounter{AlgoLine}{0}
\nl {\bf Input}: $n$ objects in $\mathbb{R}^d$, random ranking $r_{\sigma} \in \Sigma_{n,d}$ {\label{lst:line:pair01}}\\
\Begin{
\nl $\mathcal{TS}(\Hndp,r_{\sigma}) \leftarrow$ {\bf FindLabels}$\parenb{r_{\sigma}}$ \tcc{indentifies  $\mathcal{TS}(\Hndp,r_{\sigma})$ via linear programming} {\label{lst:line:dual02}}
\nl\For {$(h,l)$ $\in$ $\mathcal{TS}(\Hndp,r_{\sigma})$ {\label{lst:line:dual03}}}{
teacher provides halfspace queries $(h,l)$
}
}
\end{algorithm}

Note that to teach the ranking $r_{\sigma}$ teacher has to provide the labels in $\mathcal{TS}(\Hndp,r_{\sigma})$. Since, $\mathcal{TS}(\Hndp,r_{\sigma})$ corresponds to the labels of the query hyperplanes which form the bounding set for $r_{\sigma}$, thus the entire ranking can be \tt{inferred}. \algoref{algo: teaching ranking} is straight forward in which for the set of objects $\Theta$ and a random ranking $r_{\sigma}$  teacher identifies the pair of comparisons using the subroutine \textbf{FindLabels}($\cdot$) and iteratively provides the labels (or halfspace queries) $\mathrm{wrt}$ the reference $r_{\sigma}$.
As discussed for \algoref{algo: teaching halfspace}, the subroutine \textbf{FindLabels}($\cdot$) can obtain the enclosing region in $\cO(n^4)$ iteration by solving linear equations system corresponding to $\cO(n^2)$ constraints.

\subsection{Average Complexity of Teaching Linear Ranking Functions}
Before we delve into the relevant results of the subsection, we would motivate the notations.
\paragraph{Notations} Consider the set of $n$ objects $\Theta = \parenb{\theta_1,\theta_2,\cdots,\theta_n}$ embedded in $\Rd$ in general position. We denote by $h_{i,j}$ the bijecting hyperplane for the pairwise comparison $q_{i,j}$ for objects $\theta_i$ and $\theta_j$. We use $C(n,d)$ to denote the number of regions or equivalently $d$-cells induced by query hyperplanes corresponding to pairwise comparisons of the embedded objects. $F(n,d)$ denotes the number of faces induced on all the query hyperplanes by their intersections.\\
\newline
The ideas behind the bound share similar motivations as for \thmref{theorem: Main Theorem}. Since the rankings are selected uniform at random, if we ascertain the number of faces for any region on average we get the bound. Thus, first we mention a recursion on $C(n,d)$ stated in \cite{jamieson2011active}. Then, we provide the result for the total number of faces induced on all the bisecting hyperplanes. 

\begin{lemma}[Lemma 1, \cite{jamieson2011active}]\label{pairwise regions}
 Assume {\bf E1}. Let $C(n, d)$ denote the number of $d$-cells (regions) defined by the hyperplane arrangement of pairwise comparisons between these objects (\tt{i.e.} $C(n, d) = |\Sigma_{n,d}|)$. $C(n, d)$ satisfies the recursion:
 \begin{equation*}
     C(n,d) = C(n-1, d) + (n-1)C(n-1, d-1)
 \end{equation*}
\end{lemma}

\begin{lemma}\label{pairwise faces}
 Assume {\bf E1}. Let $F(n, d)$ denote the number of faces induced by the hyperplane arrangement of pairwise comparisons between these objects. $F(n, d)$ satisfies the recursion:
 \begin{equation*}
     F(n,d) = \binom{n}{2}\cdot C(n-1, d-1)
 \end{equation*}
\end{lemma}
\begin{proof}
If we consider any object say $\theta_k$, then the pairwise comparison induced hyperplane $h_{k,i}$ for a fixed $i \neq k$ is \textit{uniquely} intersected by query hyperplanes induced by pairwise comparison of other objects since they are in general position. Thus, on the $(d-1)$-dimensional hyperplane $h_{k,i}$ there are $\binom{n-1}{2}$ intersections (flats of dimension $d-2$). Following the discussion for Lemma 1, \cite{jamieson2011active} we note that the number of regions or $(d-1)$-cells induced on the bisecting hyperplane $h_{k,i}$ for a query is \tt{exactly} $C(n-1,d-1)$. Since there are $\binom{n}{2}$ hyperplanes for all the pairwise queries, thus the lemma follows.
\end{proof}
\begin{corollary}[Corollary 1, \cite{jamieson2011active}]\label{cor:nowak}
 There exist positive real numbers $k_1$ and $k_2$ such that 
 \begin{equation*}
     k_1\frac{n^{2d}}{2^d d!} < C(n,d) < k_2\frac{n^{2d}}{2^d d!}
 \end{equation*}
for $n > d+1$. If $n \le d+1$, then $C(n,d) = n!$.
\end{corollary}

The following result shows that even under this special arrangement of hyperplanes, the average complexity for teaching such a ranking is $\bigTheta{d}$.
\begin{theorem}\label{theorem: teaching ranking}
Assume $\bf{E1}$ and $r_{\sigma}$ $\sim$ $\mathcal{U}$. There exists a teaching algorithm which requests $\bigTheta{d}$ pairwise comparisons on average for ranking \tt{i.e.} $\mathbb{E}_{\mathcal{U}}[M_n] = \bigTheta{d}$ where $M_n$ denotes a random variable for the number of pairwise
comparisons requested by an algorithm. In other words, the average teaching complexity of ranking via pairwise comparisons is $\bigTheta{d}$.
\end{theorem}
\subsection{Proof of \thmref{theorem: teaching ranking}}
We would prove the main result in two parts: (i) \lemref{lemma: pairwise lower bound} claims the upper bound on the average teaching complexity and (ii) \lemref{lemma: pairwise upper bound} claims the average teaching complexity. Thus, we show the proof of the main result by combining (i) and (ii). Similar to \secref{subsection: main theorem}, we analyze the following ratio to achieve the bounds:
\begin{align}
    \underbrace{\expctover{r_{\sigma}\sim\cU}{|\mathcal{TS}(\Hndp,r_{\sigma})|} = \frac{2\cdot\ff\parenb{\cA(\Hndp)}}{\fr\parenb{\cA(\Hndp)}} = 
    \frac{ \text{\lemref{pairwise faces}}}{\text{\lemref{pairwise regions}}} }_{\textbf{Average teaching complexity of ranking}} 
    \tag{A.7}\label{eqn:tag:a.7}
\end{align}
Key idea of the proofs is to control the rate in \ref{eqn:tag:a.7}.
Let us denote by  $M_n$ a random variable for the number of labels provided by the teacher for a uniformly random sampled ranking $r_{\sigma} \in \Sigma_{n,d}$. We say $\sigma$ $\sim \mathcal{U}$ for ease of notation. We would show that \algoref{algo: teaching ranking}, runs for at most $\mathcal{O}(d)$ in the following lemma \footnote{Note that \citet{Fukuda1991BoundingTN} established an $\bigO{d}$ average complexity for teaching convex polytopes under any hyperplane arrangement. Therefore one can apply \cite{Fukuda1991BoundingTN} to achieve the upper bound in \thmref{theorem: teaching ranking}. Here, we provide an alternative proof of the upper bound, which could be of separate interest.}.

\begin{lemma}\label{lemma: pairwise upper bound}
Assume $\bf{E1}$ and $\sigma$ $\sim$ $\mathcal{U}$. Let the random variable $M_n$ denote the number of pairwise
comparisons that are requested in the teaching \algoref{algo: teaching ranking}, then
\begin{equation*}
    \mathbb{E}_{\mathcal{U}}[M_n] \le c\cdot d
\end{equation*}
for some positive constant $c$.
\end{lemma}
\begin{proof}
For teaching, the labels of enclosing query hyperplanes of the reference point $r_{\sigma}$ induced by the objects, should be specified. Since the rankings are sampled uniformly at random, each ranking is enclosed by $F(n,d)\big/C(n,d)$ hyperplanes on average. We prove the theorem in two cases using the \corref{cor:nowak} and \lemref{pairwise faces}.\\
\newline
\underline{\textbf{Case} 1:} $n > d+1$ ($n$ is sufficiently large)\\
\begin{align*}
 \frac{F(n,d)}{C(n,d)}  &= \frac{\binom{n}{2}\cdot C(n-1, d-1)}{C(n,d)} \le \binom{n}{2} \cdot\Biggparen{k_2 \frac{(n-1)^{2(d-1)}}{2^{d-1}(d-1)!}}\cdot\frac{1}{k_1 \frac{n^{2d}}{2^dd!}} = \Biggparen{1 - \frac{1}{n}}^{2d-1}\frac{k_2}{k_1}d \le c\cdot d
\end{align*}
The second inequality follows from \corref{cor:nowak}.
\newline
\underline{\textbf{Case} 2:} $n \le d+1$\\
\begin{equation*}
    \frac{F(n,d)}{C(n,d)} = \frac{\binom{n}{2}\cdot (n-1)!}{n!} = \frac{n-1}{2} \le \frac{d}{2}
\end{equation*}

Thus, in the two cases we have shown that $\frac{F(n,d)}{C(n,d)}$ = $\mathcal{O}(d)$. This proves the lemma.
\end{proof}


We would show that \algoref{algo: teaching ranking}, runs for at least $\bigOmega{d}$ in the following lemma for sufficiently large $n$.
\begin{lemma}\label{lemma: pairwise lower bound}
Assume $\bf{E1}$ and $\sigma$ $\sim$ $\mathcal{U}$. Let the random variable $M_n$ denote the number of pairwise
comparisons that are requested in the teaching \algoref{algo: teaching ranking}, then for sufficiently large $n > d$:
\begin{equation*}
    \mathbb{E}_{\mathcal{U}}[M_n] \ge c\cdot d
\end{equation*}
for some positive constant $c$.
\end{lemma}
\begin{proof}
Following similar steps in upper bound provided in \lemref{lemma: pairwise upper bound}, but instead using opposite side of bounds in \corref{cor:nowak}, we get:\\

For $n > d+1$ ($n$ is sufficiently large)
\begin{align*}
 \frac{F(n,d)}{C(n,d)}  &= \frac{\binom{n}{2}\cdot C(n-1, d-1)}{C(n,d)} \ge \binom{n}{2} \cdot\Biggparen{k_1 \frac{(n-1)^{2(d-1)}}{2^{d-1}(d-1)!}}\cdot\frac{1}{k_2 \frac{n^{2d}}{2^dd!}} = \Biggparen{1 - \frac{1}{n}}^{2d-1}\frac{k_1}{k_2}d \ge c\cdot d
\end{align*}
The second inequality follows from \corref{cor:nowak}. In the last inequality we note that $\parenb{1 - \frac{1}{n}}^{2d-1}$ is bounded since $ \lim_{n \to \infty} \parenb{1 - \frac{1}{n}}^{n} = \frac{1}{e}$ and is increasing for large enough $n$.

Thus, we have shown that $\frac{F(n,d)}{C(n,d)}$ = $\bigOmega{d}$. This proves the lemma.
\end{proof}

\begin{proof}[Proof of \thmref{theorem: teaching ranking}]
In \lemref{lemma: pairwise upper bound} and \lemref{lemma: pairwise lower bound}, we showed the required bounds of $\cO(d)$ and $\bigOmega{d}$, and thus $\mathbb{E}_{\mathcal{U}}[M_n] = \bigTheta{d}$, which completes the proof.
\end{proof}

 }
 }
 {}
\end{document}